%% file: camera_ready.tex
\documentclass[10pt,twocolumn,letterpaper]{article}

\usepackage{iccv}
\makeatletter
\@namedef{ver@everyshi.sty}{}
\makeatother

\usepackage{times}
\usepackage{epsfig}
\usepackage{graphicx}
\usepackage{amsmath}
\usepackage{amssymb}
\usepackage{amsthm}
\usepackage{wrapfig}
\usepackage{algorithm}
\usepackage{algpseudocode}
\newtheorem{proposition}{Proposition}
\newtheorem{assumption}{Assumption}
\usepackage{booktabs}
\usepackage{subcaption} 
\usepackage{tikz}
\usetikzlibrary{positioning,shapes}
\usepackage{pgfplots} 
\usetikzlibrary{arrows.meta,decorations.markings}
\usetikzlibrary{shapes.geometric}

\pgfplotsset{compat=newest} 
\pgfplotsset{plot coordinates/math parser=false} 

\newcommand{\CR}[1]{#1}

\newcommand{\iprod}[2]{\langle #1, #2 \rangle}
\newcommand{\norm}[1]{\| #1 \|}
\newcommand{\C}[2]{\mathcal{C}(#1,#2)}

\newcommand{\dd}{\mathrm{d}}

\newenvironment{customlegend}[1][]{%
  \begingroup
  \csname pgfplots@init@cleared@structures\endcsname
  \pgfplotsset{#1}%
}{%
  \csname pgfplots@createlegend\endcsname
  \endgroup
}%
\def\addlegendimage{\csname pgfplots@addlegendimage\endcsname}

\usepackage[pagebackref=true,breaklinks=true,letterpaper=true,colorlinks,bookmarks=false]{hyperref}

\iccvfinalcopy 

\def\iccvPaperID{2246} 

\ificcvfinal\fi
\begin{document}

\title{Controlling Neural Networks via Energy Dissipation}

\author{
Michael Moeller\\
University of Siegen\\
{\tt\small michael.moeller@uni-siegen.de}\\
\and
Thomas M\"ollenhoff\\
TU Munich\\
{\tt\small thomas.moellenhoff@tum.de}\\
\and
Daniel Cremers\\
TU Munich\\
{\tt\small cremers@tum.de}
}

\maketitle

\begin{abstract}
  The last decade has shown a tremendous success in solving various computer vision problems with the help of deep learning techniques. Lately, many works have demonstrated that learning-based approaches with suitable network architectures even exhibit superior performance for the solution of (ill-posed) image reconstruction problems such as deblurring, super-resolution, or medical image reconstruction. The drawback of purely learning-based methods, however, is that they cannot provide provable guarantees for the trained network to follow a given data formation process during inference. In this work we propose \emph{energy dissipating networks} that iteratively compute a descent direction with respect to a given cost function or energy at the currently estimated reconstruction. Therefore, an adaptive step size rule such as a line-search, along with a suitable number of iterations can guarantee the reconstruction to follow a given data formation model encoded in the energy to arbitrary precision, and hence control the model's behavior even during test time.
  We prove that under standard assumptions, descent using the direction predicted by the network converges (linearly) to the global minimum of the energy. 
  We illustrate the effectiveness of the proposed approach in experiments on single image super resolution and computed tomography (CT) reconstruction, and further illustrate extensions to convex feasibility problems. 
\end{abstract}

\section{Introduction}
In the overwhelming number of imaging applications, the quantity of interest cannot be observed directly, but rather has to be inferred from measurements that contain implicit information about it. For instance, color images have to be restored from the raw data captured through a color filter array (demosaicking), suboptimal foci or camera movements cause blurs that ought to be removed to obtain visually pleasing images (deblurring), and non-invasive medical imaging techniques such as magnetic resonance imaging (MRI) or computed tomography (CT) can faithfully be modeled as sampling the image's Fourier transform and computing its Radon transform, respectively. Mathematically, the above problems can be phrased as \textit{linear inverse problems} in which one tries to recover the desired quantity $\hat{u}$ from measurements $f$ that arise from applying an application-dependent linear operator $A$ to the unknown and contain additive noise $\xi$:
\begin{align}
\label{eq:linearInvProb}
f = A\hat{u} + \xi.
\end{align} 
Unfortunately, most practically relevant inverse problems are \textit{ill-posed}, meaning that equation \eqref{eq:linearInvProb} either does not determine $\hat{u}$ uniquely even if $\xi = 0$, or tiny
amounts of noise $\xi$ can alter the naive prediction of $\hat{u}$ significantly. These phenomena have been well-investigated from a mathematical perspective with \textit{regularization methods} being
the tool to still obtain provably stable reconstructions. 

At the heart of all regularization methods is the idea not to determine a naive estimate like $u = A^\dagger f$ for the pseudo-inverse $A^\dagger$. Instead, one determines an estimate $u$ for which 
\begin{align}
\label{eq:fidelityConstraint}
 \|Au-f\|\leq \delta,
\end{align}
holds for a suitable norm $\| \cdot \|$ and some $\delta \in \mathbb{R}$ proportional to $\| \xi \|$. The actual $u$ is found by some iterative procedure that is stopped as soon as \eqref{eq:fidelityConstraint} is met, or by explicitly enforcing a desired regularity via a penalty term or constraint in an energy minimization method.

\begin{figure*}[htb]
\setlength{\tabcolsep}{0pt}
\begin{tabular}{cccc}
\begin{tikzpicture}
\node(a){\includegraphics[width = 0.237\textwidth]{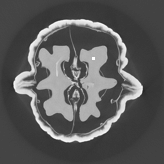}};
\node at(a.center)[draw, green, line width=2pt, ellipse, minimum width=22pt, minimum height=22pt,rotate=0,xshift=8pt, yshift=17pt]{};
\end{tikzpicture}   &
\begin{tikzpicture}
\node(a){\includegraphics[width = 0.237\textwidth]{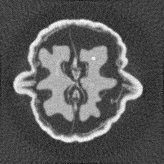}};
\node at(a.center)[draw, green, line width=2pt, ellipse, minimum width=22pt, minimum height=22pt,rotate=0,xshift=8pt, yshift=17pt]{};
\end{tikzpicture}   &
\begin{tikzpicture}
\node(a){\includegraphics[width = 0.237\textwidth]{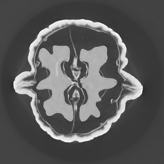}};
\node at(a.center)[draw, green, line width=2pt, ellipse, minimum width=22pt, minimum height=22pt,rotate=0,xshift=8pt, yshift=17pt]{};
\end{tikzpicture}   &
\begin{tikzpicture}
\node(a){\includegraphics[width = 0.237\textwidth]{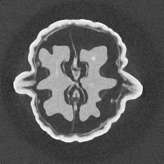}};
\node at(a.center)[draw, green, line width=2pt, ellipse, minimum width=22pt, minimum height=22pt,rotate=0,xshift=8pt, yshift=17pt]{};
\end{tikzpicture}   \\
{\small ground truth} & {\small gradient descent, PSNR 27.5} & {\small learned, PSNR 40.2} & {\small energy dissipating, PSNR 33.0}
\end{tabular}
\caption{Illustrating the danger of learning safety-critical reconstruction algorithms: When overfitting a network to reconstruct a walnut image without a simulated pathology and using it to predict the reconstruction of a walnut with such an artifact (left image, white blob in the green circle), the reconstruction (middle right) is extremely clean, but the most important aspect is entirely lost, despite being contained in the data as illustrated in the plain gradient descent reconstruction (middle left). Our energy dissipating networks (right image) are able to benefit from the power of learning-based techniques while allowing to provably guarantee data fidelity constraints such as \eqref{eq:fidelityConstraint}.}
\label{fig:teaser}
\end{figure*}
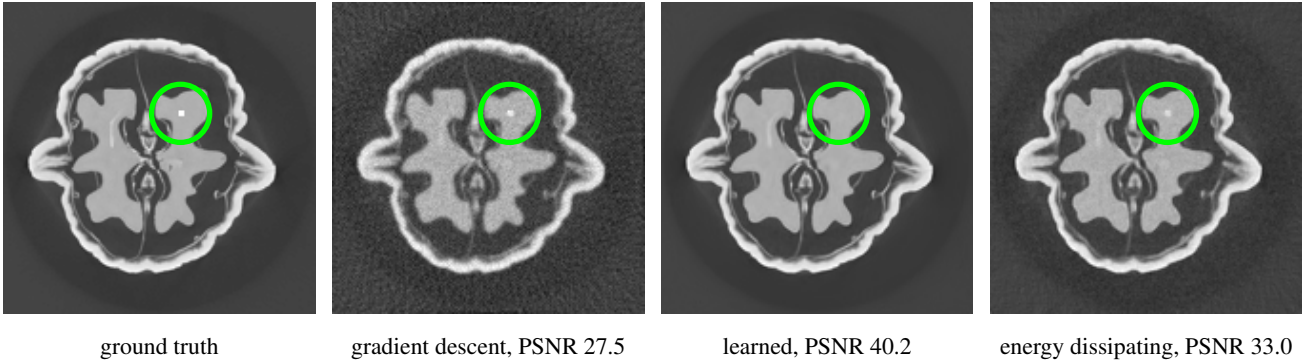

Motivated by the success of deep learning in classification and semantic segmentation tasks \cite{krizhevsky2012imagenet,long2015fully}, researchers have proposed to tackle a wide variety of different linear inverse problems with deep learning techniques, e.g., deblurring \cite{Xu14, Tao_2018_CVPR}, super-resolution \cite{Dong16}, demosaicking \cite{Gharbi16}, MRI- \cite{Yang18}, or CT-reconstruction \cite{Kan16}. All the aforementioned results are based on acquiring or simulating exemplary pairs $(\hat{u}_i, f_i)$ of ground-truth images $\hat{u}_i$ and data $f_i$ and training a network $\mathcal{G}$ which, for suitable parameters $\theta$, accomplishes a mapping $\mathcal{G}(f_i;\theta) \approx \hat{u}_i$. While such methods show truly remarkable reconstruction quality in practical applications, there is no mechanism to provably guarantee that the solutions $\mathcal{G}(f;\theta)$ predicted by the network actually explain the measured data in accordance with the (known) data formation process \eqref{eq:linearInvProb}, i.e., no a-priori proximity bound in the form of \eqref{eq:fidelityConstraint} can be guaranteed. The latter can pose a significant risk in trusting the networks prediction, particularly if little, or biased training data is provided. We illustrated such a risk in a toy example in Fig.~\ref{fig:teaser}: A CT scan is simulated on the ground truth image shown on the left, which is a slight modification of the walnut image from \cite{walnutData} with an additional small blob in the upper right. The simplest reconstruction, unregularized gradient descent, is very noisy, but the anomaly (our blob), is clearly visible. \CR{A network that has been trained on reconstructing the walnut without the blob is noise-free and has the highest peak-signal-to-noise-ratio (PSNR). However, it completely removed the important pathology (middle right).} This is particularly disconcerting because the measured data clearly contains information about the blob as seen in the gradient descent image, and a purely learning-based approach may just decide to disrespect the data formation process \eqref{eq:fidelityConstraint} during inference. 

Our proposed approach (illustrated in Fig. \ref{fig:teaser} on the right) is a learning-based iterative algorithm that can \textit{guarantee} the solution $u$ to meet the constraint \eqref{eq:fidelityConstraint} for any predefined (feasible) value of $\delta$. Despite also training our network on reconstructing the walnut without the anomaly only, it is able to reconstruct the blob, allowing its use \textit{even in safety critical applications like medical imaging}. 

The key idea is to train a neural network $\mathcal{G}$ for predicting a descent direction $d^k$ for a given model driven (differentiable) energy $E$ such as $E(u) = \frac{1}{2}\|Au-f\|^2$. The network takes the current iterate $u^k$, the data $f$, and the gradient $\nabla E(u^k)$ as inputs and predicts a direction:
\begin{align}
\label{eq:networkDirection}
d^k = \mathcal{G}(u^k,f,\nabla E(u^k);\theta).
\end{align}
The output of the network is constrained in such a way that for a fixed parameter $\zeta > 0$ the following condition
provably holds for arbitrary points $u^k$:
\begin{align}
\label{eq:descentDirectionProperty}
\langle d^k , \nabla E(u^k) \rangle \geq \zeta \|\nabla E(u^k)\|.
\end{align}
This property allows us to employ a descent algorithm 
\begin{align}
\label{eq:proposedScheme}
u^{k+1} = u^k - \tau_k ~ d^k,
\end{align}
in which a suitable (adaptive) step size rule for $\tau_k$, such as a line-search, can guarantee the convergence to a minimizer of $E$ under weak assumptions. Therefore, the proposed scheme \eqref{eq:proposedScheme} can provably enforce constraints like \eqref{eq:fidelityConstraint} on unseen test data, while also benefiting from training data and the inductive bias of modern deep architectures.

\section{Related Work} 
\subsection{Model-based solutions of inverse problems}
\label{sec:regularizationMethods}
Two of the most successful strategies for solving inverse problem in imaging are variational methods as well as (related) iterative regularization techniques. The former phrase the solution as a minimization problem of the form 
\begin{align}
\label{eq:variationalMethod}
\hat{u} = \arg \min_u H(u; f) + \alpha R(u), 
\end{align}
for a data fidelity term $H$, regularization term $R$ and a trade-off parameter $\alpha$ that has to be chosen depending on the amount of noise in the data. As an example, a common choice for the data term is $H(u; f) = \frac{1}{2}\|Au-f\|^2$ and $R(u) = |\nabla u|$ is often the celebrated total variation (TV) regularizer \cite{rof}.  Variants include minimizing the regularizer subject to a constraint on the fidelity, also know as Morozow regularization, or minimizing the fidelity term subject to a constraint on the regularizer, also known as Lavrentiev regularization. Other popular regularizers include extensions of the TV, for example  introducing higher-order derivatives \cite{tgv,chambolle}, sparsity priors for certain representations, such as wavelets or dictionaries \cite{mairal2014sparse}, or application-driven penalties \cite{he2011single}. 

Closely related iterative approaches construct sequences $\{ u_k \}$ that decrease the fidelity term monotonically along a suitable path (for instance steered by an appropriate regularizer as in \cite{BregmanIteration}), and become a regularization method via stopping criteria such as the \textit{discrepancy principle}, see \cite{benning_burger_2018}.  

\subsection{Data-driven and hybrid methods}
While the model-based reconstruction techniques of Sec.~\ref{sec:regularizationMethods} admit a thorough mathematical understanding of their behavior with well defined regularization properties, see \cite{benning_burger_2018} for an overview, the solution quality on common benchmarks can often be enhanced significantly by turning to data-driven approaches. 

Most frequently, convolutional neural networks (CNNs) \cite{fukushima1980neocognitron,lecun1989backpropagation} are used to solve such image reconstruction problems, e.g. for deblurring \cite{Xu14}, single image super resolution \cite{osendorfer2014image,Dong16}, or CT reconstruction \cite{Unser17}. Another line of works pioneered in \cite{gregor2010learning} is to take algorithms used to solve the model-based minimization problem \eqref{eq:variationalMethod}, unroll them, and declare certain parts to be learnable \cite{schmidt2014shrinkage, Zheng15, Kobler17, learningRDE}. Although such architectures are well motivated and often yield excellent results with rather few learnable parameters, they do not come with any provable guarantees.

Alternatively, researchers have considered learning the regularizer in \eqref{eq:variationalMethod} \cite{Roth2009, ksvd, Hawe13, Chen14}. This typically results in difficult bilevel optimization problems or requires additional assumptions such as sparsity. Moreover, the considered penalties are rather simple, such that they do not quite reach the performance of fully data-driven models. 

Other approaches have proposed algorithmic schemes that replace proximal operators of the regularizer by a CNN, but don't come with any provable guarantees \cite{Meinhardt17, Chang17,Zhang17a} or impose restrictive conditions on the network \cite{Romano2017}. 

In order to obtain guarantees, the recent approaches \cite{liu2018bridging, liu2019learning} propose convergent energy minimization methods that incorporate update steps based on deep networks. The fundamental difference to our approach is, that in these works the data-driven updates are considered an error, which is then controlled by an interleaving with sufficiently many standard optimization steps. In our approach, every update is solely based on the neural network, and each update provably reduces the energy. Furthermore, the models in \cite{liu2018bridging, liu2019learning} are trained on auxilliary tasks (such as image denoising), while we propose a training procedure that is more in line with the actual task of energy minimization. More general approaches on learning to optimize are based on recurrent neural networks \cite{andrychowicz2016learning} or reinforcement learning \cite{li2016learning}. However, it still remains challenging to provide a rigorous convergence analysis.

Some recent approaches have considered to regularize the reconstruction by the parametrization of a network, e.g., deep image priors in \cite{Ulyanov2018DeepIP}, or deep decoders \cite{Heckel2019DeepDC}. This requires to solve a large nonconvex optimization problem which is rather expensive and due to local optima may also prevent a provable satisfaction of \eqref{eq:fidelityConstraint}.

Finally, another line of works train generative models and represent the variable $u$ to be reconstructed as the transformation of a latent variable $z$. While these approaches yield strong results, see e.g., \cite{Bora17a}, the nonconvex optimization can not necessarily guarantee a constraint like \eqref{eq:fidelityConstraint}.

Current methods for imposing constraints directly in the network architecture \cite{optNet, frerix2019linear} are limited to low-dimensional linear inequality constraints or polyhedra with moderate numbers of vertices.
Therefore, directly incorporating (and differentiating through) projections onto complex constraint sets as \eqref{eq:fidelityConstraint} currently seems infeasible. Nevertheless, the proposed technique, which we present in the following section, can ensure difficult constraints such as \eqref{eq:fidelityConstraint} within a data-driven approach.

\section{Energy Dissipating Networks}
\subsection{Paths in the energy landscape}
\begin{figure*}[htb]
\centering
\begin{center}
\input{f1.tex}
\input{f2.tex}
\input{f3.tex}
\end{center}
\vspace{-0.5cm}
\caption{Two dimensional illustration of the proposed energy dissipation network for solving the (underdetermined) inverse problem of finding $u$ such that $u_1 + u_2 = f$: The left plot shows the path gradient descent takes when initialized with zero, along with the magnitude and direction of the gradient at various points in the 2D plane. The middle plot shows the same visualization determined with our energy dissipating networks and additional training data. The right plot shows the results of an energy dissipating network with a more aggressive enforcement of descent directions via scaling $\zeta$ in \eqref{eq:descentDirectionProperty} with an additional $\|\nabla E(u^k)\|$. As we can see, in both cases the network learned to create paths that lead towards the provided training data while provably guaranteeing non-increasing data-fidelity costs. The depicted networks even move solutions within the subspace $Au=f$ towards the clustered data points. }
\label{fig:2dIllustration}
\end{figure*}
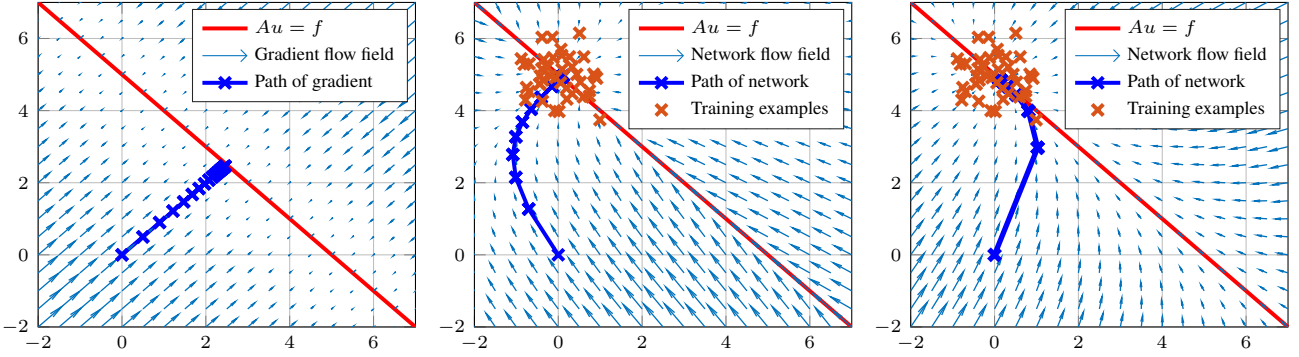

As presented in Eq.~\eqref{eq:proposedScheme},  we propose a simple technique that uses a data-driven neural network $\mathcal{G}$ to iteratively predict descent directions with respect to a given problem dependent cost function or \textit{energy} $E$, which we assume to be continuously differentiable. While the idea of learning the gradient descent update in inverse problems has been considered under the framework of unrolled algorithms \cite{Adler_2017}, our architecture is conceptually different: Instead of training an unrolled architecture in an end-to-end fashion, we merely train the network on separate iterations by minimizing an objective of the form
\begin{align}
\label{eq:trainingLoss}
    \mathbb{E}_{u^* \sim \mathcal{Y}} \mathbb{E}_{u \sim \mathcal{X}_\theta} \left[ \ell \left( u - \mathcal{G}(u,f,\nabla E(u); \theta), u^* \right) \right],
\end{align}
with respect to the model parameters $\theta$. In the above equation, $\ell$ is a loss function and $\mathcal{Y}$ the distribution of the "ground truth" data. A novelty in our approach is to consider an input distribution $\mathcal{X}_{\theta}$, which depends on the model parameters itself. We will make the above more precise later on, but $\mathcal{X}_\theta$ can be thought of as a distribution over iterates which are visited when running a descent method with directions provided by the model. While an objective of the form \eqref{eq:trainingLoss} seems daunting to optimize, we show that a simple "lagged" approach in which the parameters of $\mathcal{X}_\theta$ are updated sporadically and otherwise fixed works well in practice.
Another novelty of our approach is to use the Euclidean projection onto the set 
\begin{align}
\label{eq:constraintSet}
\C{\zeta}{g} = \left \{ d ~|~ \langle d, g \rangle \geq \zeta \| g \| \right \}
\end{align}
for $g$ being the gradient of the energy (or Lyapunov function) $E$ at the current iterate as a last layer of the network. By additionally utilizing a backtracking line-search detailed in Alg.~\ref{alg:ourScheme} to determine suitable step-sizes $\tau_k$, we can state the following convergence result:
\begin{proposition}
\label{prop:convergence}
Consider the iterates given by Alg.~\ref{alg:ourScheme} for an arbitrary starting point $u^0$, a continuously differentiable and coercive energy $E$, and a continuous (in the inputs) model $\mathcal{G}([u,f,\nabla E(u)];\theta)$ that satisfies  
\begin{align}
\label{eq:descentDirConstraint}
 \mathcal{G}(u, f, \nabla E(u); \theta) \in \C{\zeta}{\nabla E(u)} \qquad \forall u.
\end{align}
Then the while-loop of Alg.~\ref{alg:ourScheme} always terminates after finitely many iterations. The energy of the iterates is monotonically decreasing, i.e.
\begin{align}
\label{eq:energyDissipation}
    E(u^{k+1}) \leq E(u^k),
\end{align}
and the sequence of $\|\nabla E(u^k)\|$ converges to zero, i.e. 
\begin{align}
    \lim_{k\rightarrow \infty} \|\nabla E(u^k)\| = 0. 
\end{align}
Moreover, if $E$ is strictly convex, then the sequence of $u^k$ converges with 
\begin{align}
    \lim_{k\rightarrow \infty} u^k  = \arg \min_u E(u).
\end{align}
 \end{proposition}
\begin{proof}
The result is a conclusion of the descent direction, e.g. utilizing standard results like \cite[Thm. 3.2]{NoceWrig06}. 
\end{proof}

\begin{algorithm}[t!]
\caption{Learned energy dissipation with line-search}\label{alg:ourScheme}
\begin{algorithmic}[1]
\State \textbf{Inputs:} Starting point $u^0$, energy dissipating network $\mathcal{G}(\cdot,\cdot,\nabla E(\cdot); \theta)$, constants $c \in (0, 0.5), \rho \in (0,1)$
\While{not converged (e.g. based on $\|\nabla E(u^k)\|$)}
\State $d^k = \mathcal{G}(u^k, f, \nabla E(u^k); \theta)$
\State $\tau_{k} \leftarrow 1$ 
\State $u^{k+1} \leftarrow u^k - \tau_{k} d^k$
\While{ $E(u^{k+1})> E(u^k) - c \tau_k \langle d^k, \nabla E(u^k)\rangle$}
\State $\tau_k \leftarrow \rho \, \tau_k $
\State $u^{k+1} \leftarrow u^k - \tau_{k} d^k$
\EndWhile\label{innerwhileloop}
\EndWhile
\end{algorithmic}
\end{algorithm}

Due to property \eqref{eq:energyDissipation} we refer to our approach as \textit{energy dissipating networks}.  Our intuition is that such networks, which are trained on \eqref{eq:trainingLoss} but provably satisfy \eqref{eq:descentDirConstraint}, allow to predict paths in a given energy landscape that monotonically decrease the energy, but at the same time 
attract the iterates to locations which are more probable according to the training examples.
We have illustrated this concepts for $E(u)=\|Au-f\|^2$ and $A=[1 \ \ 1]$ in the two-dimensional toy example shown in Fig.~\ref{fig:2dIllustration}. While plain gradient descent merely drives the iterates towards the subspace of solutions in a direct way, the proposed energy dissipation networks have the freedom to exploit training data and learn a vector field that drives any initialization to solutions near the training data at $u = \begin{bmatrix} 0 & 5 \end{bmatrix}^\top$ for $f = 5$. 

Besides the possibility of reconstructing solutions that reflect properties of the training examples in underdetermined cases as illustrated in Fig.~\ref{fig:2dIllustration}, the intuition of taking a more favorable path extends to energies $E(u)=\|Au-f\|^2$ where $A$ is invertible, but severely ill-conditioned: As we will illustrate in the results (Sec.~\ref{sec:numerics}), a typical gradient descent iteration on energies $E(u)=\|Au-f^\delta\|^2$ for noisy observations $f^\delta$ typically improves the results up to a certain iteration before the ill-conditioned matrix $A$ leads to the introduction of heavy noise and a degradation of the results. Thus, early stopping, or discrepancy principles are used to stop the iteration once $\|Au-f^\delta\| \approx \|f-f^\delta\|$. Again, a favorable path hopefully allows to stop at a better overall solution. As discrepancy principles are frequently used in the analysis of (continuous / infinite dimensional) inverse problems, our framework possibly allows to derive a regularization method from \eqref{eq:proposedScheme} in the sense of \cite[Def. 4.7]{benning_burger_2018}.

\subsection{Ensuring global linear convergence}
While Prop.~\ref{prop:convergence} establishes convergence, in applications an upper bound on the number of iterations required to reach a certain optimality gap can be desirable. To establish such a bound, we will make the following assumptions, which are standard in literature:
\begin{assumption}
    The energy $E$ is $L$-Lipschitz differentiable, i.e., it satisfies 
    the inequality 
    \begin{equation}
      E(v) \leq E(u) + \iprod{\nabla E(u)}{v - u} + \frac{L}{2} \norm{u - v}^2,
      \label{eq:Lsmooth}
    \end{equation}
    and furthermore satisfies the Polyak-{\L}ojaciewicz inequality with modulus $\mu > 0$,
    \begin{equation}
    \frac{1}{2} \norm{\nabla E(u)}^2 \geq \mu (E(u) - E^*),
    \label{eq:PL}
    \end{equation}
    where $E^* = \min_u E(u)$ is the global minimum.
    \label{asm:one}
\end{assumption}
We remark that functions satisfying the Polyak-{\L}ojaciewicz inequality \eqref{eq:PL} include strongly convex functions, possibly composed with a linear operator that has a non-trivial kernel. Notably, convexity is not a required assumption, but rather invexity \cite{karimi2016linear,necoara2018linear}. All examples we consider in the numerical experiments fulfill Asm.~\ref{asm:one}.

To establish a linear convergence result, we constrain the output of the network to a slightly different constraint set as the one considered in \eqref{eq:constraintSet}. 
For $\zeta_1 \geq \zeta_2 > 0$ we define it as:
\begin{equation}
    \C{\zeta_1,\zeta_2}{g} = \left \{ d ~|~ \langle d, g \rangle \geq \zeta_1 \| g \|^2, \| d \| \leq \zeta_2 \| g \| \right \}.
    \label{eq:constraintSetStrict}
\end{equation}
To give an intuition about \eqref{eq:constraintSetStrict}, note that the two conditions imply that the angle $\theta$ 
between $d$ and $g$ is bounded by $\cos \theta \geq \zeta_1 / \zeta_2$.
Therefore, by choosing appropriate $\zeta_1$ and $\zeta_2$ one can control the angular deviation between the predicted direction of the network and the actual gradient direction. 
The linear convergence result is made precise in the following proposition.
\begin{proposition}
    Assume that the energy $E$ satisfies Asm.~\ref{asm:one} and that the update directions given by the network meet $d^k \in \C{\zeta_1,\zeta_2}{\nabla E(u^k)}$. Then \eqref{eq:proposedScheme} with constant step size $\tau_k \equiv \zeta_1 / ((\zeta_2)^2 L)$ converges linearly,
    \begin{equation}
        E(u^{k+1}) - E^* \leq \left( 1 -  \frac{\gamma^2 \mu}{L} \right)^k (E(u^0) - E^*),
    \end{equation}
    where $\gamma = \zeta_1 / \zeta_2$.
    \label{prop:linear}
\end{proposition}
\begin{proof}
  Combining \eqref{eq:Lsmooth} and the descent iteration \eqref{eq:proposedScheme} we have the following bound on the decrease of the energy:
  \begin{equation}
    \begin{aligned}
      &E(u^{k+1}) - E(u^k) \leq -\tau \iprod{\nabla E(u^k)}{d^k}\\
      &\qquad  + \frac{L \tau^2}{2} \norm{d^k}^2.
    \end{aligned}
  \end{equation}
  Using $d^k \in \C{\zeta_1,\zeta_2}{\nabla E(u^k)}$ we can further bound this
  \begin{equation}
    \begin{aligned}
      E(u^{k+1}) - E(u^k) &\leq \left(\frac{L \tau^2 (\zeta_2)^2}{2} -\tau \zeta_1\right) \norm{\nabla E(u^k)}^2 \\
&= -\frac{\gamma^2}{2L} \norm{\nabla E(u^k)}^2.
    \end{aligned}
  \end{equation}
  Finally, by \eqref{eq:PL} we have
  \begin{equation}
    E(u^{k+1}) - E(u^k) \leq -\frac{\gamma^2 \mu}{L}(E(u^k) - E^*),
  \end{equation}
  and rearranging and subtracting $E^*$ on both sides gives 
  \begin{equation}
    E(u^{k+1}) - E^* \leq \left( 1 -  \frac{\gamma^2 \mu}{L} \right) (E(u^k) - E^*),
  \end{equation}
  which yields the above result.
\end{proof}
Therefore, the required number of iterations to reach an $\varepsilon$-accurate solution is in the order of $\mathcal{O} \left( \gamma^{-2} (L / \mu) \log(1 / \varepsilon) \right)$. 
By giving the network more freedom to possibly deviate from the true gradient ($\theta \to \pm 90^{\circ}$, i.e., $\gamma \to 0$), more iterations are required in the worst-case.
As an example, the above analysis tells us that if we allow the directions predicted by the network to deviate by at most $45^{\circ}$ from the true gradient, then in the worst case 
we might require twice as many ($\cos 45^{\circ} = 1 / \sqrt{2}$) iterations to find an $\varepsilon$-accurate solution as standard gradient descent.
Nevertheless, we want to stress that there are also directions which dramatically improve the rate of convergence (for example the
Newton direction), which is not captured by this worst-case
analysis. As in practice the training data could coincide with
the model, it is to be expected that the learned direction will lead
to much faster initial convergence than the gradient direction. The
above analysis should therefore be only seen as a bound of what could
theoretically happen in case the network systematically picks the
worst possible (in the sense of energy decay) direction. Therefore, 
we use Alg.~\ref{alg:ourScheme} which performs a line-search 
instead of choosing the very conservative step-size from Prop.~\ref{prop:linear},
to benefit from the scenario when the direction predicted by the model is good.

\CR{We remark that the factor in the above iteration complexity could possibly be improved 
using accelerated variants of gradient descent \cite{nesterov2018lectures,necoara2018linear}.}

\section{Implementation} 
\subsection{Satisfying the descent constraints}
As discussed in the previous section, for the convergence results in Prop.~\ref{prop:convergence} and Prop.~\ref{prop:linear} to hold, we either have 
to provably satisfy the constraints \eqref{eq:descentDirConstraint} or \eqref{eq:constraintSetStrict}.
The constraint \eqref{eq:constraintSet} is a half-space, and can be satisfied by the projection: 
\begin{align}
    \label{eq:euclidProjection}
    z \mapsto z + \max \{ \zeta - \langle z, n \rangle,0 \} \cdot n, \quad n = g / \norm{g},
\end{align}
which merely consists of linear operations and a rectified linear unit (ReLU), so that it is readily implemented in any deep learning framework.
For simplicity, for the set \eqref{eq:constraintSetStrict} we propose to merely use a parametrization. 
\begin{proposition}
A surjective map onto \eqref{eq:constraintSetStrict} is given by
\begin{equation}
    z \mapsto \hat \eta g + \Pi_{B}(z - \eta g),
    \label{eq:surjective}
\end{equation}
where $\eta = \iprod{z}{g} / \norm{g}^2$, $\hat \eta = \Pi_{[\zeta_1, \zeta_2]}(\eta)$ and $\Pi_B$ is the projection onto $B = \{ d ~|~ \norm{d} \leq \sqrt{ (\zeta_2)^2 - \hat \eta^2 } \norm{g} \}$.
\end{proposition}
\begin{proof}
To see this, first note that since $\iprod{z - \eta g}{g} = 0$ also holds after the projection, the first constraint in \eqref{eq:constraintSetStrict}, 
\begin{equation}
    \iprod{\hat \eta g + \Pi_{B}(z - \eta g)}{g} \geq \zeta_1 \norm{g}^2,
\end{equation}
is satisfied since $\hat \eta \geq \zeta_1$. For the second condition, we have due to orthogonality:
\begin{align}
\norm{\hat \eta g + \Pi_{B}(z - \eta g)}^2 &= \hat{\eta}^2 \norm{g}^2 + \norm{\Pi_{B}(z - \eta g)}^2 \notag \\
&\leq (\zeta_2)^2 \norm{g}^2,
\end{align}
so that $\hat \eta g + \Pi_{B}(z - \eta g) \in \C{\zeta_1, \zeta_2}{g}$.
\end{proof}
 To avoid problems with the division by $\| g \|$, we approximate it using $\max \left \{ \| g \|, 10^{-6} \right \}$, but also note that Alg.~\ref{alg:ourScheme} stops once $\| g \|$ becomes small. 

\subsection{Lagged updating of the training distribution}
\label{sec:training}
For training, a first idea might be to incorporate a loss of the form \eqref{eq:trainingLoss} where $u \sim \mathcal{X}$ is sampled from the uniform distribution over the entire space of possible inputs. The latter is very high dimensional in most applications, and due to the curse of dimensionality it is unreasonable to aim at covering it exhaustively. Therefore we propose to introduce a dependency on the model parameters $u \sim \mathcal{X}_\theta$ and consider only inputs which are visited when running descent with the current model. To optimize such an objective, where the training distribution also depends on the model parameters, we propose the following iterative "lagged" strategy, which is an essential component of our approach.

The method is bootstrapped by fixing $\mathcal{X}_\theta$ to be the distribution of inputs obtained by running regular gradient descent on $E$ (up to a maximum number of iterations). After a certain number of epochs, we update the distribution to also contain iterates obtained by the algorithm \eqref{eq:proposedScheme} with the current parameters to generate new (network-based) paths. As the path of energy dissipating networks can differ from plain gradient descent significantly, such an iterative strategy is vital for the success of the training. 

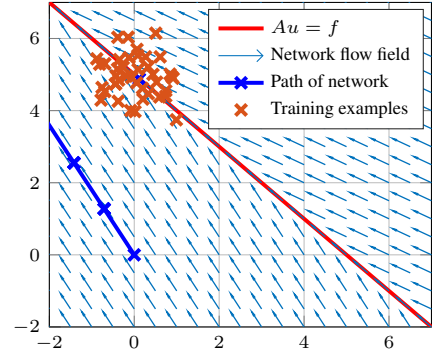
\begin{figure}[t!]
  \centering
    \begin{center}
      \input{f4.tex}
    \end{center}
    \vspace{-0.4cm}
    \caption{Motivation for the proposed ``lagged'' iterative scheme: Training only on iterates which are visited in gradient descent leads to a rough first guess. To arrive at a more refined result as shown on the right in Fig.~\ref{fig:2dIllustration}, we propose an iterative scheme in which the training distribution is updated based on the current parameters.}
    \label{fig:bad}
\end{figure}
\CR{Returning to our two dimensional toy example of Fig.~\ref{fig:2dIllustration} we show in Fig.~\ref{fig:bad}  
the network's descent vector field after only training on the gradient descent iterates (left plot in Fig. \ref{fig:2dIllustration}).} The results are reasonable (as the network recognized that the iterations have to be further to the top left), but of course unsatisfactory for practical purposes as the iterates are pushed beyond the set of exemplary points (marked by red crosses). After only two additional iterations of updating the training distribution via the models' current path, we obtain the results shown in Fig.~\ref{fig:2dIllustration} on the right. 

\begin{table*}[t!]
    \centering
    \caption{Quantitative evaluation on the task of single image super resolution ($4 \times$ upscaling). The proposed energy dissipating approach simulatenously improves upon the baseline in terms of mean PSNR, SSIM and reconstruction error $\norm{Au-f}^2$.}
    \begin{tabular}{rccc|ccc|ccc}
         \toprule
         & \multicolumn{3}{c}{Gradient Descent} & \multicolumn{3}{c}{Baseline Network} & \multicolumn{3}{c}{Energy Dissipating}\\
         & PSNR & SSIM & $\norm{Au-f}^2$
         & PSNR & SSIM & $\norm{Au-f}^2$
         & PSNR & SSIM & $\norm{Au-f}^2$\\
         \midrule 
         BSD100 
         & $25.09$ & $0.6440$ & $\mathbf{9.5 \cdot 10^{-3}}$
         & $26.67$ & $0.7237$ & $2.8 \cdot 10^{-1}$
         & $\mathbf{27.12}$ & $\mathbf{0.7297}$ & $\mathbf{9.8 \cdot 10^{-3}}$ \\
         Urban100 
         & $22.19$ & $0.6230$ & $\mathbf{6.8 \cdot 10^{-2}}$ 
         & $24.44$ & $0.7292$ & $1.9 \cdot 10^{-0}$
         & $\mathbf{24.83}$ & $\mathbf{0.7460}$ & $\mathbf{6.4 \cdot 10^{-2}}$ \\
         Set5 
         & $26.80$ & $0.7448$ & $\mathbf{1.3 \cdot 10^{-2}}$ 
         & $29.94$ & $0.8548$ & $2.5 \cdot 10^{-1}$
         & $\mathbf{31.16}$ & $\mathbf{0.8726}$ & $\mathbf{8.5 \cdot 10^{-3}}$ \\
         Set14 
         & $24.74$ & $0.6684$ & $\mathbf{1.6 \cdot 10^{-2}}$ 
         & $27.22$ & $0.7642$ & $4.3 \cdot 10^{-1}$
         & $\mathbf{27.74}$ & $\mathbf{0.7709}$ & $\mathbf{1.3 \cdot 10^{-2}}$ \\
         \bottomrule
    \end{tabular}
    \label{tab:super}
\end{table*}

\subsection{Choosing a suitable energy}
In many imaging applications suitable data fidelity terms of model-driven approaches are well-known, or can be derived from maximum a-posteriori probability estimates for suitable noise models. While only using a suitable data fidelity term as an energy or Lyapunov function $E$ allows to predict solutions that approximate the measured data to any desired accuracy, energy dissipating networks could further be safeguarded with a classical regularization term, making a classical regularized solution the lower bound on the performance of the method. We will demonstrate numerical results for both such approaches in Sec.~\ref{sec:numerics}. Moreover, although our main motivation stemmed from linear inverse problems in the form of \eqref{eq:linearInvProb}, the idea of energy dissipating networks extends far beyond this field of application.

By choosing the distance to any desired constraint set as a surrogate energy, one can provably constrain the predicted solutions to any desired constrained set (at least if the some measure of distance to the desired set can be stated in closed form). We will illustrate such an extension to convex feasibility problems in Sec.~\ref{sec:feasibility} and Sec.~\ref{sec:sudokus}.

\section{Applications}
\label{sec:numerics}
We implemented the experiments from Sec.~\ref{sec:sir} and Sec.~\ref{sec:sudokus} using the PyTorch framework. The CT reconstruction experiments in Sec.~\ref{sec:ct} 
are implemented in Matlab. All models were trained using the Adam optimizer \cite{adam}. \CR{In all experiments $\ell$ in Eq.~\eqref{eq:trainingLoss} was chosen as the square loss.}

\subsection{Single image super resolution}
\label{sec:sir}
As a first application, we train an energy dissipating network on the task of single image super resolution. The linear operator in the energy $E(u) = \frac{1}{2} \norm{Au-f}^2$ is chosen as a $4 \times$ downsampling, implemented via average pooling.

The architecture both for the energy dissipating and baseline approach is based on \cite{dncnn}, which 
consists of $20$ blocks of $3\times 3$ convolutional filters with 64 channels, ReLUs, and batch normalization layers, before a final $3\times 3$ convolutional layer reduces the output to a single channel. 
For the energy dissipating approach, the result is fed into the layer given by \eqref{eq:surjective} with $\zeta_1 = 25$, $\zeta_2 = 10000$. 
Following \cite{dncnn}, we initialize with the ``DnCNN-B'' model pretrained on the task of image denoising, which we found important for the baseline to achieve good performance but had no effect on our approach. 
Starting from that initialization, we train on the specific task of $4 \times$ super resolution on $52 \times 52$ patches from images in the BSD100 training set. 
\CR{For the energy dissipating network, the training data is updated every $100$ mini-batches according to Sec.~\ref{sec:training}, choosing uniformly $0$--$10$ descent steps to generate the iterates. As the data is generated in an online fashion, samples of a previous model or from the gradient descent initialization are discarded due to efficiency reasons.}  

\CR{During testing, we used a fixed number of $15$ iterations in the descent with our network direction. For the gradient descent baseline (corresponding to simple up--sampling) we used $75$ iterations to reach a similar reconstruction error. For both methods, running more iterations did not significantly effect the result. Note that in principle one can run both convergent descent methods until the forward model is satisfied to an arbitrary accuracy.} 
A quantitative evaluation on a standard super resolution benchmark is given in the above Table~\ref{tab:super}. Note that our goal was not necessarily to achieve high PSNR results, but rather demonstrate in a controlled experiment that our approach can be used to introduce provable guarantees into existing models without significantly worsening (here even improving) the quality.

\subsection{CT Reconstruction}
\label{sec:ct}
Next, we consider reconstructing CT images using the (sparse matrix) operator provided in \cite{walnutData}. The architecture 
is the same as for the super resolution experiment, but using $17$ blocks and the projection layer given by \eqref{eq:euclidProjection}. 
In the lack of highly accurate CT scans, we simulate training data using phantoms from \cite{brainwebWebsite, brainwebArticle} (despite being MRI phantoms), also include random crops of natural images
from BSD68 \cite{Roth2009}, and reconstruct in 2D (slices only). \CR{We use $E(u) = \frac{1}{2}\|Au-f\|^2 + \alpha TV_\epsilon(u)$ (with $TV_\epsilon$ being a Charbonnier total variation
  regularization)} as our surrogate energy and start with the first 10 gradient descent iterates as the initial distribution of the training data for the lagged updating scheme described in
Sec. \ref{sec:training}.

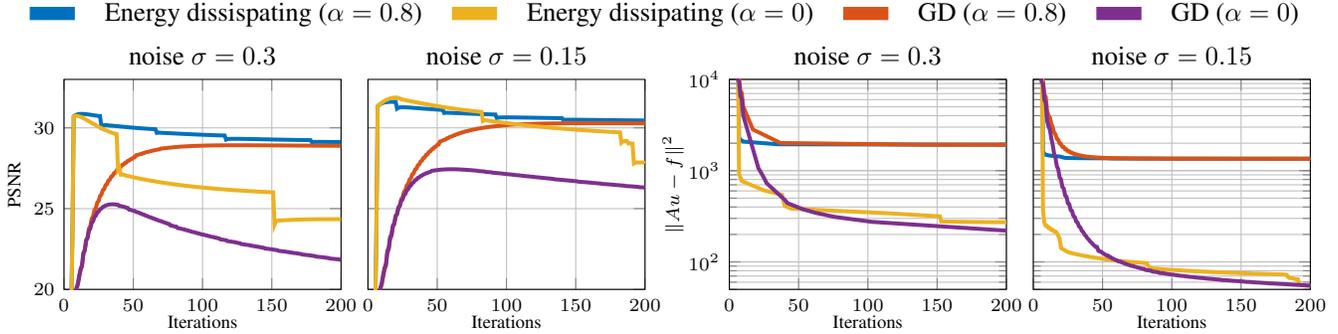
\begin{figure*}[t!]
    \centering
    \setlength{\tabcolsep}{-3pt}
  \begin{tikzpicture}
    \begin{customlegend}[legend columns=4,legend style={align=center,draw=none,column sep=1.5ex},
      legend entries={
        Energy dissispating ($\alpha = 0.8$),
        Energy dissipating ($\alpha = 0$),
        GD ($\alpha = 0.8$),
        GD ($\alpha = 0$),
      }]
      \definecolor{mycolor1}{rgb}{0.00000,0.44700,0.74100}%
      \definecolor{mycolor2}{rgb}{0.85000,0.32500,0.09800}%
      \definecolor{mycolor3}{rgb}{0.92900,0.69400,0.12500}%
      \definecolor{mycolor4}{rgb}{0.49400,0.18400,0.55600}%
      \addlegendimage{mark=none,solid,line legend,line width=5pt,color=mycolor1}
      \addlegendimage{mark=none,solid,line width=5pt,color=mycolor3}   
      \addlegendimage{mark=none,solid,line width=5pt,color=mycolor2}   
      \addlegendimage{mark=none,solid,line width=5pt,color=mycolor4}   
    \end{customlegend}
  \end{tikzpicture}
    \begin{tabular}{cccc}
    \input{f2_1.tex}&
    \input{f2_2.tex}&
    \input{f2_3.tex}&
    \input{f2_4.tex}
    \end{tabular}
    \vspace{-0.2cm}

    \caption{Comparing gradient descent (GD) with energy dissipating networks: the first two plots show the PSNR value over the number of iterations, and illustrate that the energy dissipating networks peak very quickly before (provably) converging to the same minimizer as their surrogate energies. As shown in the right plots, they quickly decrease the data discrepancy term and plateau in the regularized case, or keep decaying when using the data term as a surrogate energy only. Note that all curves of our approach were generated with the same energy dissipating networks, merely changing the surrogate energy during testing, which demonstrates an ability to generalize beyond the specific scenario used during training.}
    \label{fig:ct_results_plot}
\end{figure*}

We test the resulting model by simulating a CT-scan of a walnut image from \cite{walnutData}. We compare our scheme~\ref{alg:ourScheme} to plain gradient descent on the data term as well as to regularized reconstructions for different regularization parameters $\alpha$. Fig.~\ref{fig:ct_results_plot} shows the PSNR values as well as the decay of the data fidelity term over the number of iterations for two different noise levels. As we can see, after only 5-10 iterations, the energy dissipating networks have reached significantly higher PSNR values than their model-based counter parts, before converging to the same solution (at least in the regularized cases) as predicted by Prop.~\ref{prop:convergence}. \CR{The energy dissipating network with $\alpha=0$ reaches $\|Au-f\|\leq \sigma$ after 8 (for $\sigma = 0.3$) and 15 (for $\sigma = 0.15$) iterations, indicating that a discrepancy principle can be useful.} 

 Moreover, the data term decays quicker than for the respective gradient methods, indicating that the trained networks can even represent fast optimizers. All plots have been generated with a single network, hinting at an ability to generalize to different noise levels and regularization parameters. 

\subsection{Imposing convex constraints}
\label{sec:feasibility}
In this section we consider convex feasibility problems, in which the task is to find a point in 
the non-empty intersection of $N$ closed convex sets $\{ C_1, \hdots, C_N \}$.
A formulation as a differentiable energy minimization problem is given by:
\begin{equation}
    E(u) = \frac{1}{2N} \sum_{i=1}^N \dd_{C_i}^2(u),
    \label{eq:feas}
\end{equation}
where $\dd_C^2(u) = \min_{v \in C} ~ \norm{u - v}^2$ is the squared Euclidean distance of the point $u$ to the set $C$. The energy \eqref{eq:feas} is non-negative and reaches zero if and only if $u$ satisfies all constraints. By standard results in convex 
analysis \cite[Corollary~12.30]{BauschkeCombettes} it follows that $\nabla d_C^2 = 2 (\text{id} - \Pi_C)$ so that the gradient of \eqref{eq:feas} is $1$-Lipschitz and can be implemented if the projections onto the sets $C_i$ are available. Under mild regularity assumptions (which even hold for some nonconvex sets) the energy \eqref{eq:feas} satisfies the Polyak-{\L}ojaciewicz inequality, see \cite[Prop.~8.6]{lewis2007local}, and therefore also Asm.~\ref{asm:one}.

\subsection{Sudoku}
\label{sec:sudokus}
To demonstrate the previous application we tackle the problem of solving $9 \times 9$ Sudokus,
which has emerged as a benchmark for learning based methods \cite{optNet,palm2018recurrent}. Note that OptNet \cite{optNet} only considers $4 \times 4$ due to scalability issues as remarked in
\cite{palm2018recurrent}. We use a convex relaxation of the binary feasibility problem on $\{0, 1\}^{9 \times 9 \times 9}$ from \cite{artacho2014recent}. Specifically, we have $N=5$ in
\eqref{eq:feas}, where $C_1, \hdots C_3$ contain simplex constraints for each dimension of the tensor, $C_4$ constraints on the $3 \times 3$ fields in the Sudoku and $C_5$ encodes the givens. 

We consider the same architecture as for the previous experiments and train on a dataset of one million Sudokus \cite{park}. For the energy dissipating approach, we constrain the output to be in the constraint set \eqref{eq:constraintSetStrict} and use the scheme from Sec.~\ref{sec:training} with $15$ iterations to update the training data. 
\begin{table}[t]
\centering
    \caption{Evaluation on a test set of 50 easy Sudokus.}
        \begin{tabular}{cc|cc|cc}
         \toprule
         \multicolumn{2}{c}{GD (100 it.)} &
         \multicolumn{2}{c}{Baseline Net.} &
         \multicolumn{2}{c}{En. Diss. (100 it.)} \\
         Acc. & Solve &
         Acc. & Solve &
         Acc. & Solve \\
         \midrule
         $ 76.9 \%$ & $2.0 \%$ &   
         $ 82.6 \%$ & $4.0 \%$ &
         $ \mathbf{87.1 \%} $ & $\mathbf{52.0 \%} $ \\
         \bottomrule 
        \end{tabular}
\end{table}
In the above table we compare $100$ iterations of gradient descent (equivalent to averaged projections), a baseline which performs a forward pass through a (seperately trained) network and $100$
iterations of energy dissipation. We improve upon the baseline which indicates that it is possible to learn directions which reduce the energy of feasibility problems faster than the steepest descent direction. 

\section{Conclusion}
We constrained the outputs of deep networks to be descent directions for suitably chosen energy functions. \CR{Combined with a line-serach algorithm, this yields (linear) convergence to minimizers of the underlying energy during inference.} We further proposed an iterative training scheme to learn a model which provides descent directions that bias the optimization into the direction of
the training data. In experiments, we demonstrated that this approach can be used to control networks via model-based energies, and at the same time improve over standard end-to-end baselines.

{\small
\bibliographystyle{ieee_fullname}
\bibliography{refs}
}

\end{document}

%% file: f1.tex
%
%
\definecolor{mycolor1}{rgb}{0.00000,0.44700,0.74100}%
\begin{tikzpicture}

\begin{axis}[%
width=1.975in,
height=1.7in,
at={(1.011in,0.642in)},
scale only axis,
xmin=-2,
xmax=7,
ymin=-2,
ymax=7,
axis background/.style={fill=white},
xmajorgrids,
ymajorgrids,
tick label style={font=\scriptsize},
legend style={font=\scriptsize, legend cell align=left, align=left, draw=white!15!black}
]
\addplot [color=red, line width=1.5pt]
  table[row sep=crcr]{%
7	-2\\
-2	7\\
};
\addlegendentry{$Au=f$}

\addplot[-Straight Barb, color=mycolor1] 
table[row sep=crcr] {
x	y	u	v\\
-2	-2	0.424750526787253	0.424750526787253\\
};
\addlegendentry{Gradient flow field}

\addplot [color=blue, line width=1.5pt, draw=none, mark size=3.0pt, mark=x, mark options={solid, blue}]
  table[row sep=crcr]{%
0	0\\
};

\addplot [color=blue, line width=1.5pt, mark size=3.0pt, mark=x, mark options={solid, blue}, forget plot]
  table[row sep=crcr]{%
0.498344739545118	0.498344739545118\\
0	0\\
};
\addplot [color=blue, line width=1.5pt, mark size=3.0pt, mark=x, mark options={solid, blue}, forget plot]
  table[row sep=crcr]{%
0.897020531181213	0.897020531181213\\
0.498344739545118	0.498344739545118\\
};
\addplot [color=blue, line width=1.5pt, mark size=3.0pt, mark=x, mark options={solid, blue}, forget plot]
  table[row sep=crcr]{%
1.21596116449009	1.21596116449009\\
0.897020531181213	0.897020531181213\\
};
\addplot [color=blue, line width=1.5pt, mark size=3.0pt, mark=x, mark options={solid, blue}, forget plot]
  table[row sep=crcr]{%
1.47111367113719	1.47111367113719\\
1.21596116449009	1.21596116449009\\
};
\addplot [color=blue, line width=1.5pt, mark size=3.0pt, mark=x, mark options={solid, blue}, forget plot]
  table[row sep=crcr]{%
1.67523567645487	1.67523567645487\\
1.47111367113719	1.47111367113719\\
};
\addplot [color=blue, line width=1.5pt, mark size=3.0pt, mark=x, mark options={solid, blue}, forget plot]
  table[row sep=crcr]{%
1.83853328070901	1.83853328070901\\
1.67523567645487	1.67523567645487\\
};
\addplot [color=blue, line width=1.5pt, mark size=3.0pt, mark=x, mark options={solid, blue}, forget plot]
  table[row sep=crcr]{%
1.96917136411233	1.96917136411233\\
1.83853328070901	1.83853328070901\\
};
\addplot [color=blue, line width=1.5pt, mark size=3.0pt, mark=x, mark options={solid, blue}, forget plot]
  table[row sep=crcr]{%
2.07368183083498	2.07368183083498\\
1.96917136411233	1.96917136411233\\
};
\addplot [color=blue, line width=1.5pt, mark size=3.0pt, mark=x, mark options={solid, blue}, forget plot]
  table[row sep=crcr]{%
2.1572902042131	2.1572902042131\\
2.07368183083498	2.07368183083498\\
};
\addplot [color=blue, line width=1.5pt, mark size=3.0pt, mark=x, mark options={solid, blue}, forget plot]
  table[row sep=crcr]{%
2.2241769029156	2.2241769029156\\
2.1572902042131	2.1572902042131\\
};
\addplot [color=blue, line width=1.5pt, mark size=3.0pt, mark=x, mark options={solid, blue}, forget plot]
  table[row sep=crcr]{%
2.2776862618776	2.2776862618776\\
2.2241769029156	2.2241769029156\\
};
\addplot [color=blue, line width=1.5pt, mark size=3.0pt, mark=x, mark options={solid, blue}, forget plot]
  table[row sep=crcr]{%
2.3204937490472	2.3204937490472\\
2.2776862618776	2.2776862618776\\
};
\addplot [color=blue, line width=1.5pt, mark size=3.0pt, mark=x, mark options={solid, blue}, forget plot]
  table[row sep=crcr]{%
2.35473973878288	2.35473973878288\\
2.3204937490472	2.3204937490472\\
};
\addplot [color=blue, line width=1.5pt, mark size=3.0pt, mark=x, mark options={solid, blue}, forget plot]
  table[row sep=crcr]{%
2.38213653057142	2.38213653057142\\
2.35473973878288	2.35473973878288\\
};
\addplot [color=blue, line width=1.5pt, mark size=3.0pt, mark=x, mark options={solid, blue}, forget plot]
  table[row sep=crcr]{%
2.40405396400225	2.40405396400225\\
2.38213653057142	2.38213653057142\\
};
\addplot [color=blue, line width=1.5pt, mark size=3.0pt, mark=x, mark options={solid, blue}, forget plot]
  table[row sep=crcr]{%
2.42158791074692	2.42158791074692\\
2.40405396400225	2.40405396400225\\
};
\addplot [color=blue, line width=1.5pt, mark size=3.0pt, mark=x, mark options={solid, blue}, forget plot]
  table[row sep=crcr]{%
2.43561506814266	2.43561506814266\\
2.42158791074692	2.42158791074692\\
};
\addplot [color=blue, line width=1.5pt, mark size=3.0pt, mark=x, mark options={solid, blue}, forget plot]
  table[row sep=crcr]{%
2.44683679405924	2.44683679405924\\
2.43561506814266	2.43561506814266\\
};
\addplot [color=blue, line width=1.5pt, mark size=3.0pt, mark=x, mark options={solid, blue}, forget plot]
  table[row sep=crcr]{%
2.45581417479251	2.45581417479251\\
2.44683679405924	2.44683679405924\\
};
\addplot [color=blue, line width=1.5pt, mark size=3.0pt, mark=x, mark options={solid, blue}, forget plot]
  table[row sep=crcr]{%
2.46299607937913	2.46299607937913\\
2.45581417479251	2.45581417479251\\
};
\addlegendentry{Path of gradient}

\addplot[-Straight Barb, color=mycolor1, line width=0.5pt, point meta={sqrt((\thisrow{u})^2+(\thisrow{v})^2)}, point meta min=0, quiver={scale arrows=1.5,u=\thisrow{u}, v=\thisrow{v}, every arrow/.append style={-{Straight Barb[angle'=18.263, scale={(2/1000)*\pgfplotspointmetatransformed}]}}}]
 table[row sep=crcr] {%
x	y	u	v\\
-2	-2	0.424750526787253	0.424750526787253\\
-2	-1.5	0.401109795780004	0.401109795780004\\
-2	-1	0.377469064772756	0.377469064772756\\
-2	-0.5	0.353828333765508	0.353828333765508\\
-2	0	0.33018760275826	0.33018760275826\\
-2	0.5	0.306546871751011	0.306546871751011\\
-2	1	0.282906140743763	0.282906140743763\\
-2	1.5	0.259265409736515	0.259265409736515\\
-2	2	0.235624678729267	0.235624678729267\\
-2	2.5	0.211983947722018	0.211983947722018\\
-2	3	0.18834321671477	0.18834321671477\\
-2	3.5	0.164702485707522	0.164702485707522\\
-2	4	0.141061754700274	0.141061754700274\\
-2	4.5	0.117421023693025	0.117421023693025\\
-2	5	0.0937802926857772	0.0937802926857772\\
-2	5.5	0.0701395616785289	0.0701395616785289\\
-2	6	0.0464988306712807	0.0464988306712807\\
-2	6.5	0.0228580996640325	0.0228580996640325\\
-2	7	-0.000782631343215795	-0.000782631343215795\\
-1.5	-2	0.401109795780004	0.401109795780004\\
-1.5	-1.5	0.377469064772756	0.377469064772756\\
-1.5	-1	0.353828333765508	0.353828333765508\\
-1.5	-0.5	0.33018760275826	0.33018760275826\\
-1.5	0	0.306546871751011	0.306546871751011\\
-1.5	0.5	0.282906140743763	0.282906140743763\\
-1.5	1	0.259265409736515	0.259265409736515\\
-1.5	1.5	0.235624678729267	0.235624678729267\\
-1.5	2	0.211983947722018	0.211983947722018\\
-1.5	2.5	0.18834321671477	0.18834321671477\\
-1.5	3	0.164702485707522	0.164702485707522\\
-1.5	3.5	0.141061754700274	0.141061754700274\\
-1.5	4	0.117421023693025	0.117421023693025\\
-1.5	4.5	0.0937802926857772	0.0937802926857772\\
-1.5	5	0.0701395616785289	0.0701395616785289\\
-1.5	5.5	0.0464988306712807	0.0464988306712807\\
-1.5	6	0.0228580996640325	0.0228580996640325\\
-1.5	6.5	-0.000782631343215795	-0.000782631343215795\\
-1.5	7	-0.024423362350464	-0.024423362350464\\
-1	-2	0.377469064772756	0.377469064772756\\
-1	-1.5	0.353828333765508	0.353828333765508\\
-1	-1	0.33018760275826	0.33018760275826\\
-1	-0.5	0.306546871751011	0.306546871751011\\
-1	0	0.282906140743763	0.282906140743763\\
-1	0.5	0.259265409736515	0.259265409736515\\
-1	1	0.235624678729267	0.235624678729267\\
-1	1.5	0.211983947722018	0.211983947722018\\
-1	2	0.18834321671477	0.18834321671477\\
-1	2.5	0.164702485707522	0.164702485707522\\
-1	3	0.141061754700274	0.141061754700274\\
-1	3.5	0.117421023693025	0.117421023693025\\
-1	4	0.0937802926857772	0.0937802926857772\\
-1	4.5	0.0701395616785289	0.0701395616785289\\
-1	5	0.0464988306712807	0.0464988306712807\\
-1	5.5	0.0228580996640325	0.0228580996640325\\
-1	6	-0.000782631343215795	-0.000782631343215795\\
-1	6.5	-0.024423362350464	-0.024423362350464\\
-1	7	-0.0480640933577123	-0.0480640933577123\\
-0.5	-2	0.353828333765508	0.353828333765508\\
-0.5	-1.5	0.33018760275826	0.33018760275826\\
-0.5	-1	0.306546871751011	0.306546871751011\\
-0.5	-0.5	0.282906140743763	0.282906140743763\\
-0.5	0	0.259265409736515	0.259265409736515\\
-0.5	0.5	0.235624678729267	0.235624678729267\\
-0.5	1	0.211983947722018	0.211983947722018\\
-0.5	1.5	0.18834321671477	0.18834321671477\\
-0.5	2	0.164702485707522	0.164702485707522\\
-0.5	2.5	0.141061754700274	0.141061754700274\\
-0.5	3	0.117421023693025	0.117421023693025\\
-0.5	3.5	0.0937802926857772	0.0937802926857772\\
-0.5	4	0.0701395616785289	0.0701395616785289\\
-0.5	4.5	0.0464988306712807	0.0464988306712807\\
-0.5	5	0.0228580996640325	0.0228580996640325\\
-0.5	5.5	-0.000782631343215795	-0.000782631343215795\\
-0.5	6	-0.024423362350464	-0.024423362350464\\
-0.5	6.5	-0.0480640933577123	-0.0480640933577123\\
-0.5	7	-0.0717048243649605	-0.0717048243649605\\
0	-2	0.33018760275826	0.33018760275826\\
0	-1.5	0.306546871751011	0.306546871751011\\
0	-1	0.282906140743763	0.282906140743763\\
0	-0.5	0.259265409736515	0.259265409736515\\
0	0	0.235624678729267	0.235624678729267\\
0	0.5	0.211983947722018	0.211983947722018\\
0	1	0.18834321671477	0.18834321671477\\
0	1.5	0.164702485707522	0.164702485707522\\
0	2	0.141061754700274	0.141061754700274\\
0	2.5	0.117421023693025	0.117421023693025\\
0	3	0.0937802926857772	0.0937802926857772\\
0	3.5	0.0701395616785289	0.0701395616785289\\
0	4	0.0464988306712807	0.0464988306712807\\
0	4.5	0.0228580996640325	0.0228580996640325\\
0	5	-0.000782631343215795	-0.000782631343215795\\
0	5.5	-0.024423362350464	-0.024423362350464\\
0	6	-0.0480640933577123	-0.0480640933577123\\
0	6.5	-0.0717048243649605	-0.0717048243649605\\
0	7	-0.0953455553722088	-0.0953455553722088\\
0.5	-2	0.306546871751011	0.306546871751011\\
0.5	-1.5	0.282906140743763	0.282906140743763\\
0.5	-1	0.259265409736515	0.259265409736515\\
0.5	-0.5	0.235624678729267	0.235624678729267\\
0.5	0	0.211983947722018	0.211983947722018\\
0.5	0.5	0.18834321671477	0.18834321671477\\
0.5	1	0.164702485707522	0.164702485707522\\
0.5	1.5	0.141061754700274	0.141061754700274\\
0.5	2	0.117421023693025	0.117421023693025\\
0.5	2.5	0.0937802926857772	0.0937802926857772\\
0.5	3	0.0701395616785289	0.0701395616785289\\
0.5	3.5	0.0464988306712807	0.0464988306712807\\
0.5	4	0.0228580996640325	0.0228580996640325\\
0.5	4.5	-0.000782631343215795	-0.000782631343215795\\
0.5	5	-0.024423362350464	-0.024423362350464\\
0.5	5.5	-0.0480640933577123	-0.0480640933577123\\
0.5	6	-0.0717048243649605	-0.0717048243649605\\
0.5	6.5	-0.0953455553722088	-0.0953455553722088\\
0.5	7	-0.118986286379457	-0.118986286379457\\
1	-2	0.282906140743763	0.282906140743763\\
1	-1.5	0.259265409736515	0.259265409736515\\
1	-1	0.235624678729267	0.235624678729267\\
1	-0.5	0.211983947722018	0.211983947722018\\
1	0	0.18834321671477	0.18834321671477\\
1	0.5	0.164702485707522	0.164702485707522\\
1	1	0.141061754700274	0.141061754700274\\
1	1.5	0.117421023693025	0.117421023693025\\
1	2	0.0937802926857772	0.0937802926857772\\
1	2.5	0.0701395616785289	0.0701395616785289\\
1	3	0.0464988306712807	0.0464988306712807\\
1	3.5	0.0228580996640325	0.0228580996640325\\
1	4	-0.000782631343215795	-0.000782631343215795\\
1	4.5	-0.024423362350464	-0.024423362350464\\
1	5	-0.0480640933577123	-0.0480640933577123\\
1	5.5	-0.0717048243649605	-0.0717048243649605\\
1	6	-0.0953455553722088	-0.0953455553722088\\
1	6.5	-0.118986286379457	-0.118986286379457\\
1	7	-0.142627017386705	-0.142627017386705\\
1.5	-2	0.259265409736515	0.259265409736515\\
1.5	-1.5	0.235624678729267	0.235624678729267\\
1.5	-1	0.211983947722018	0.211983947722018\\
1.5	-0.5	0.18834321671477	0.18834321671477\\
1.5	0	0.164702485707522	0.164702485707522\\
1.5	0.5	0.141061754700274	0.141061754700274\\
1.5	1	0.117421023693025	0.117421023693025\\
1.5	1.5	0.0937802926857772	0.0937802926857772\\
1.5	2	0.0701395616785289	0.0701395616785289\\
1.5	2.5	0.0464988306712807	0.0464988306712807\\
1.5	3	0.0228580996640325	0.0228580996640325\\
1.5	3.5	-0.000782631343215795	-0.000782631343215795\\
1.5	4	-0.024423362350464	-0.024423362350464\\
1.5	4.5	-0.0480640933577123	-0.0480640933577123\\
1.5	5	-0.0717048243649605	-0.0717048243649605\\
1.5	5.5	-0.0953455553722088	-0.0953455553722088\\
1.5	6	-0.118986286379457	-0.118986286379457\\
1.5	6.5	-0.142627017386705	-0.142627017386705\\
1.5	7	-0.166267748393954	-0.166267748393954\\
2	-2	0.235624678729267	0.235624678729267\\
2	-1.5	0.211983947722018	0.211983947722018\\
2	-1	0.18834321671477	0.18834321671477\\
2	-0.5	0.164702485707522	0.164702485707522\\
2	0	0.141061754700274	0.141061754700274\\
2	0.5	0.117421023693025	0.117421023693025\\
2	1	0.0937802926857772	0.0937802926857772\\
2	1.5	0.0701395616785289	0.0701395616785289\\
2	2	0.0464988306712807	0.0464988306712807\\
2	2.5	0.0228580996640325	0.0228580996640325\\
2	3	-0.000782631343215795	-0.000782631343215795\\
2	3.5	-0.024423362350464	-0.024423362350464\\
2	4	-0.0480640933577123	-0.0480640933577123\\
2	4.5	-0.0717048243649605	-0.0717048243649605\\
2	5	-0.0953455553722088	-0.0953455553722088\\
2	5.5	-0.118986286379457	-0.118986286379457\\
2	6	-0.142627017386705	-0.142627017386705\\
2	6.5	-0.166267748393954	-0.166267748393954\\
2	7	-0.189908479401202	-0.189908479401202\\
2.5	-2	0.211983947722018	0.211983947722018\\
2.5	-1.5	0.18834321671477	0.18834321671477\\
2.5	-1	0.164702485707522	0.164702485707522\\
2.5	-0.5	0.141061754700274	0.141061754700274\\
2.5	0	0.117421023693025	0.117421023693025\\
2.5	0.5	0.0937802926857772	0.0937802926857772\\
2.5	1	0.0701395616785289	0.0701395616785289\\
2.5	1.5	0.0464988306712807	0.0464988306712807\\
2.5	2	0.0228580996640325	0.0228580996640325\\
2.5	2.5	-0.000782631343215795	-0.000782631343215795\\
2.5	3	-0.024423362350464	-0.024423362350464\\
2.5	3.5	-0.0480640933577123	-0.0480640933577123\\
2.5	4	-0.0717048243649605	-0.0717048243649605\\
2.5	4.5	-0.0953455553722088	-0.0953455553722088\\
2.5	5	-0.118986286379457	-0.118986286379457\\
2.5	5.5	-0.142627017386705	-0.142627017386705\\
2.5	6	-0.166267748393954	-0.166267748393954\\
2.5	6.5	-0.189908479401202	-0.189908479401202\\
2.5	7	-0.21354921040845	-0.21354921040845\\
3	-2	0.18834321671477	0.18834321671477\\
3	-1.5	0.164702485707522	0.164702485707522\\
3	-1	0.141061754700274	0.141061754700274\\
3	-0.5	0.117421023693025	0.117421023693025\\
3	0	0.0937802926857772	0.0937802926857772\\
3	0.5	0.0701395616785289	0.0701395616785289\\
3	1	0.0464988306712807	0.0464988306712807\\
3	1.5	0.0228580996640325	0.0228580996640325\\
3	2	-0.000782631343215795	-0.000782631343215795\\
3	2.5	-0.024423362350464	-0.024423362350464\\
3	3	-0.0480640933577123	-0.0480640933577123\\
3	3.5	-0.0717048243649605	-0.0717048243649605\\
3	4	-0.0953455553722088	-0.0953455553722088\\
3	4.5	-0.118986286379457	-0.118986286379457\\
3	5	-0.142627017386705	-0.142627017386705\\
3	5.5	-0.166267748393954	-0.166267748393954\\
3	6	-0.189908479401202	-0.189908479401202\\
3	6.5	-0.21354921040845	-0.21354921040845\\
3	7	-0.237189941415698	-0.237189941415698\\
3.5	-2	0.164702485707522	0.164702485707522\\
3.5	-1.5	0.141061754700274	0.141061754700274\\
3.5	-1	0.117421023693025	0.117421023693025\\
3.5	-0.5	0.0937802926857772	0.0937802926857772\\
3.5	0	0.0701395616785289	0.0701395616785289\\
3.5	0.5	0.0464988306712807	0.0464988306712807\\
3.5	1	0.0228580996640325	0.0228580996640325\\
3.5	1.5	-0.000782631343215795	-0.000782631343215795\\
3.5	2	-0.024423362350464	-0.024423362350464\\
3.5	2.5	-0.0480640933577123	-0.0480640933577123\\
3.5	3	-0.0717048243649605	-0.0717048243649605\\
3.5	3.5	-0.0953455553722088	-0.0953455553722088\\
3.5	4	-0.118986286379457	-0.118986286379457\\
3.5	4.5	-0.142627017386705	-0.142627017386705\\
3.5	5	-0.166267748393954	-0.166267748393954\\
3.5	5.5	-0.189908479401202	-0.189908479401202\\
3.5	6	-0.21354921040845	-0.21354921040845\\
3.5	6.5	-0.237189941415698	-0.237189941415698\\
3.5	7	-0.260830672422946	-0.260830672422946\\
4	-2	0.141061754700274	0.141061754700274\\
4	-1.5	0.117421023693025	0.117421023693025\\
4	-1	0.0937802926857772	0.0937802926857772\\
4	-0.5	0.0701395616785289	0.0701395616785289\\
4	0	0.0464988306712807	0.0464988306712807\\
4	0.5	0.0228580996640325	0.0228580996640325\\
4	1	-0.000782631343215795	-0.000782631343215795\\
4	1.5	-0.024423362350464	-0.024423362350464\\
4	2	-0.0480640933577123	-0.0480640933577123\\
4	2.5	-0.0717048243649605	-0.0717048243649605\\
4	3	-0.0953455553722088	-0.0953455553722088\\
4	3.5	-0.118986286379457	-0.118986286379457\\
4	4	-0.142627017386705	-0.142627017386705\\
4	4.5	-0.166267748393954	-0.166267748393954\\
4	5	-0.189908479401202	-0.189908479401202\\
4	5.5	-0.21354921040845	-0.21354921040845\\
4	6	-0.237189941415698	-0.237189941415698\\
4	6.5	-0.260830672422946	-0.260830672422946\\
4	7	-0.284471403430195	-0.284471403430195\\
4.5	-2	0.117421023693025	0.117421023693025\\
4.5	-1.5	0.0937802926857772	0.0937802926857772\\
4.5	-1	0.0701395616785289	0.0701395616785289\\
4.5	-0.5	0.0464988306712807	0.0464988306712807\\
4.5	0	0.0228580996640325	0.0228580996640325\\
4.5	0.5	-0.000782631343215795	-0.000782631343215795\\
4.5	1	-0.024423362350464	-0.024423362350464\\
4.5	1.5	-0.0480640933577123	-0.0480640933577123\\
4.5	2	-0.0717048243649605	-0.0717048243649605\\
4.5	2.5	-0.0953455553722088	-0.0953455553722088\\
4.5	3	-0.118986286379457	-0.118986286379457\\
4.5	3.5	-0.142627017386705	-0.142627017386705\\
4.5	4	-0.166267748393954	-0.166267748393954\\
4.5	4.5	-0.189908479401202	-0.189908479401202\\
4.5	5	-0.21354921040845	-0.21354921040845\\
4.5	5.5	-0.237189941415698	-0.237189941415698\\
4.5	6	-0.260830672422946	-0.260830672422946\\
4.5	6.5	-0.284471403430195	-0.284471403430195\\
4.5	7	-0.308112134437443	-0.308112134437443\\
5	-2	0.0937802926857772	0.0937802926857772\\
5	-1.5	0.0701395616785289	0.0701395616785289\\
5	-1	0.0464988306712807	0.0464988306712807\\
5	-0.5	0.0228580996640325	0.0228580996640325\\
5	0	-0.000782631343215795	-0.000782631343215795\\
5	0.5	-0.024423362350464	-0.024423362350464\\
5	1	-0.0480640933577123	-0.0480640933577123\\
5	1.5	-0.0717048243649605	-0.0717048243649605\\
5	2	-0.0953455553722088	-0.0953455553722088\\
5	2.5	-0.118986286379457	-0.118986286379457\\
5	3	-0.142627017386705	-0.142627017386705\\
5	3.5	-0.166267748393954	-0.166267748393954\\
5	4	-0.189908479401202	-0.189908479401202\\
5	4.5	-0.21354921040845	-0.21354921040845\\
5	5	-0.237189941415698	-0.237189941415698\\
5	5.5	-0.260830672422946	-0.260830672422946\\
5	6	-0.284471403430195	-0.284471403430195\\
5	6.5	-0.308112134437443	-0.308112134437443\\
5	7	-0.331752865444691	-0.331752865444691\\
5.5	-2	0.0701395616785289	0.0701395616785289\\
5.5	-1.5	0.0464988306712807	0.0464988306712807\\
5.5	-1	0.0228580996640325	0.0228580996640325\\
5.5	-0.5	-0.000782631343215795	-0.000782631343215795\\
5.5	0	-0.024423362350464	-0.024423362350464\\
5.5	0.5	-0.0480640933577123	-0.0480640933577123\\
5.5	1	-0.0717048243649605	-0.0717048243649605\\
5.5	1.5	-0.0953455553722088	-0.0953455553722088\\
5.5	2	-0.118986286379457	-0.118986286379457\\
5.5	2.5	-0.142627017386705	-0.142627017386705\\
5.5	3	-0.166267748393954	-0.166267748393954\\
5.5	3.5	-0.189908479401202	-0.189908479401202\\
5.5	4	-0.21354921040845	-0.21354921040845\\
5.5	4.5	-0.237189941415698	-0.237189941415698\\
5.5	5	-0.260830672422946	-0.260830672422946\\
5.5	5.5	-0.284471403430195	-0.284471403430195\\
5.5	6	-0.308112134437443	-0.308112134437443\\
5.5	6.5	-0.331752865444691	-0.331752865444691\\
5.5	7	-0.355393596451939	-0.355393596451939\\
6	-2	0.0464988306712807	0.0464988306712807\\
6	-1.5	0.0228580996640325	0.0228580996640325\\
6	-1	-0.000782631343215795	-0.000782631343215795\\
6	-0.5	-0.024423362350464	-0.024423362350464\\
6	0	-0.0480640933577123	-0.0480640933577123\\
6	0.5	-0.0717048243649605	-0.0717048243649605\\
6	1	-0.0953455553722088	-0.0953455553722088\\
6	1.5	-0.118986286379457	-0.118986286379457\\
6	2	-0.142627017386705	-0.142627017386705\\
6	2.5	-0.166267748393954	-0.166267748393954\\
6	3	-0.189908479401202	-0.189908479401202\\
6	3.5	-0.21354921040845	-0.21354921040845\\
6	4	-0.237189941415698	-0.237189941415698\\
6	4.5	-0.260830672422946	-0.260830672422946\\
6	5	-0.284471403430195	-0.284471403430195\\
6	5.5	-0.308112134437443	-0.308112134437443\\
6	6	-0.331752865444691	-0.331752865444691\\
6	6.5	-0.355393596451939	-0.355393596451939\\
6	7	-0.379034327459188	-0.379034327459188\\
6.5	-2	0.0228580996640325	0.0228580996640325\\
6.5	-1.5	-0.000782631343215795	-0.000782631343215795\\
6.5	-1	-0.024423362350464	-0.024423362350464\\
6.5	-0.5	-0.0480640933577123	-0.0480640933577123\\
6.5	0	-0.0717048243649605	-0.0717048243649605\\
6.5	0.5	-0.0953455553722088	-0.0953455553722088\\
6.5	1	-0.118986286379457	-0.118986286379457\\
6.5	1.5	-0.142627017386705	-0.142627017386705\\
6.5	2	-0.166267748393954	-0.166267748393954\\
6.5	2.5	-0.189908479401202	-0.189908479401202\\
6.5	3	-0.21354921040845	-0.21354921040845\\
6.5	3.5	-0.237189941415698	-0.237189941415698\\
6.5	4	-0.260830672422946	-0.260830672422946\\
6.5	4.5	-0.284471403430195	-0.284471403430195\\
6.5	5	-0.308112134437443	-0.308112134437443\\
6.5	5.5	-0.331752865444691	-0.331752865444691\\
6.5	6	-0.355393596451939	-0.355393596451939\\
6.5	6.5	-0.379034327459188	-0.379034327459188\\
6.5	7	-0.402675058466436	-0.402675058466436\\
7	-2	-0.000782631343215795	-0.000782631343215795\\
7	-1.5	-0.024423362350464	-0.024423362350464\\
7	-1	-0.0480640933577123	-0.0480640933577123\\
7	-0.5	-0.0717048243649605	-0.0717048243649605\\
7	0	-0.0953455553722088	-0.0953455553722088\\
7	0.5	-0.118986286379457	-0.118986286379457\\
7	1	-0.142627017386705	-0.142627017386705\\
7	1.5	-0.166267748393954	-0.166267748393954\\
7	2	-0.189908479401202	-0.189908479401202\\
7	2.5	-0.21354921040845	-0.21354921040845\\
7	3	-0.237189941415698	-0.237189941415698\\
7	3.5	-0.260830672422946	-0.260830672422946\\
7	4	-0.284471403430195	-0.284471403430195\\
7	4.5	-0.308112134437443	-0.308112134437443\\
7	5	-0.331752865444691	-0.331752865444691\\
7	5.5	-0.355393596451939	-0.355393596451939\\
7	6	-0.379034327459188	-0.379034327459188\\
7	6.5	-0.402675058466436	-0.402675058466436\\
7	7	-0.426315789473684	-0.426315789473684\\
};

\end{axis}
\end{tikzpicture}

%% file: f2.tex
%
%
\definecolor{mycolor1}{rgb}{0.00000,0.44700,0.74100}%
\definecolor{mycolor2}{rgb}{0.85000,0.32500,0.09800}%
\begin{tikzpicture}

\begin{axis}[%
width=1.975in,
height=1.7in,
at={(1.011in,0.642in)},
scale only axis,
xmin=-2,
xmax=7,
ymin=-2,
ymax=7,
axis background/.style={fill=white},
xmajorgrids,
ymajorgrids,
tick label style={font=\scriptsize},
legend style={font=\scriptsize, legend cell align=left, align=left, draw=white!15!black}
]
\addplot [color=red, line width=1.5pt]
  table[row sep=crcr]{%
7	-2\\
-2	7\\
};
\addlegendentry{$Au=f$}

\addplot[-Straight Barb, color=mycolor1] 
table[row sep=crcr] {
x	y	\\
-0.441152396757815	4.73838054114084\\
};
\addlegendentry{Network flow field}

\addplot [color=blue, line width=1.5pt, mark size=3.0pt, mark=x, mark options={solid, blue}]
  table[row sep=crcr]{%
0.0719954818487167	4.92671298980713\\
0.0701527744531631	4.92477989196777\\
};
\addlegendentry{Path of network}

\addplot [color=mycolor2, only marks, line width=1.5pt, draw=none, mark size=3.0pt, mark=x, mark options={solid, mycolor2}]
table[row sep=crcr] {
x y \\
-0.441152396757815	4.73838054114084\\
};
\addlegendentry{Training examples}

\addplot[-Straight Barb, color=mycolor1, line width=0.5pt, point meta={sqrt((\thisrow{u})^2+(\thisrow{v})^2)}, point meta min=0, quiver={scale arrows=1.5,u=\thisrow{u}, v=\thisrow{v}, every arrow/.append style={-{Straight Barb[angle'=18.263, scale={(2/1000)*\pgfplotspointmetatransformed}]}}}]
 table[row sep=crcr] {%
x	y	u	v\\
-2	-2	-0.165643190496182	0.2968493798571\\
-2	-1.5	-0.142184908302129	0.273391075543339\\
-2	-1	-0.118726603988369	0.249932749109871\\
-2	-0.5	-0.095268266495046	0.22647444479611\\
-2	0	-0.0718099511214314	0.20301614048235\\
-2	0.5	-0.0483516412777438	0.179557814048881\\
-2	1	-0.0248933424939102	0.15609950973512\\
-2	1.5	-0.00143502712029557	0.132641194361506\\
-2	2	0.022023271663538	0.109182901107599\\
-2	2.5	0.0454815980970067	0.0857245802040575\\
-2	3	0.0689398968808402	0.0622662593005158\\
-2	3.5	0.0923981956646738	0.0388079660466093\\
-2	4	0.115856505508361	0.0153496451430677\\
-2	4.5	0.139314798762268	-0.00810866470061995\\
-2	5	0.162773103076029	-0.0315670021939426\\
-2	5.5	0.186231374210227	-0.0550253562770464\\
-2	6	0.209689567925447	-0.0784837214200041\\
-2	6.5	0.233147385605631	-0.101942545546904\\
-2	7	0.190574733725913	-0.191431872413267\\
-1.5	-2	-0.188585353542053	0.319791542902971\\
-1.5	-1.5	-0.165127060288146	0.296333216469502\\
-1.5	-1	-0.14166876703424	0.272874956395158\\
-1.5	-0.5	-0.118210440600771	0.249416629961689\\
-1.5	0	-0.0947521141673023	0.22595830352822\\
-1.5	0.5	-0.0712938264433227	0.20249999921446\\
-1.5	1	-0.0478355165996351	0.179041672780991\\
-1.5	1.5	-0.0243772565252907	0.155583412706647\\
-1.5	2	-0.000918930091822012	0.132125086273178\\
-1.5	2.5	0.0225393908117196	0.108666759839709\\
-1.5	3	0.0459977006554072	0.0852084665858029\\
-1.5	3.5	0.0694560104990948	0.0617501346224072\\
-1.5	4	0.0929143092829284	0.0382918081889386\\
-1.5	4.5	0.116372580417127	0.014833509405105\\
-1.5	5	0.139830895790742	-0.00862486126777978\\
-1.5	5.5	0.163289100565816	-0.0320832485304456\\
-1.5	6	0.186746896126291	-0.0555420505376376\\
-1.5	6.5	0.144255766430631	-0.1449498054506\\
-1.5	7	0.102458642575378	-0.23366352687352\\
-1	-2	-0.211527582947048	0.342733772307966\\
-1	-1.5	-0.188069234393871	0.319275423754789\\
-1	-1	-0.164610952199819	0.295817141560737\\
-1	-0.5	-0.141152647886058	0.272358837246976\\
-1	0	-0.117694343572297	0.248900488693799\\
-1	0.5	-0.0942360392585368	0.225442184380039\\
-1	1	-0.0707777294148492	0.201983880066278\\
-1	1.5	-0.0473194251010886	0.178525597872226\\
-1	2	-0.023861109727474	0.155067249319049\\
-1	2.5	-0.00040280541371338	0.131608967124996\\
-1	3	0.0230554933701202	0.108150651751382\\
-1	3.5	0.0465137976838808	0.0846923197879862\\
-1	4	0.0699721185874224	0.0612339822946635\\
-1	4.5	0.0934303841916939	0.0377756448013408\\
-1	5	0.116888588966768	0.014317252008748\\
-1	5.5	0.140346406646951	-0.00914156658822501\\
-1	6	0.0979368101952036	-0.0984677384879319\\
-1	6.5	0.0560581641558922	-0.187263037394181\\
-1	7	0.0795159984258565	-0.210721850461227\\
-0.5	-2	-0.234469745992919	0.365675935353837\\
-0.5	-1.5	-0.211011441679158	0.34221758680066\\
-0.5	-1	-0.187553137365398	0.318759326726316\\
-0.5	-0.5	-0.164094833051637	0.295301022412555\\
-0.5	0	-0.140636528737876	0.271842695979086\\
-0.5	0.5	-0.11717823548397	0.248384413785034\\
-0.5	1	-0.0937198979906471	0.224926065231857\\
-0.5	1.5	-0.0702616047367405	0.201467760918097\\
-0.5	2	-0.0468032893631259	0.178009456604336\\
-0.5	2.5	-0.0233449684595843	0.154551130170867\\
-0.5	3	0.000113302674614207	0.131092825857107\\
-0.5	3.5	0.0235716125183018	0.107634499423638\\
-0.5	4	0.0470298836525003	0.0841761508704614\\
-0.5	4.5	0.0704880773677206	0.0607177636077956\\
-0.5	5	0.0939458950479038	0.0372589616006037\\
-0.5	5.5	0.0516178871393381	-0.0519857047048264\\
-0.5	6	0.00965764149699051	-0.140862514735279\\
-0.5	6.5	0.0331154702370278	-0.164321327802325\\
-0.5	7	0.0565737081916642	-0.187779731654772\\
0	-2	-0.257411931158498	0.388618142639124\\
0	-1.5	-0.233953604725029	0.365159794085947\\
0	-1	-0.210495300411268	0.341701489772187\\
0	-0.5	-0.187037018217216	0.318243163338718\\
0	0	-0.163578736023163	0.294784881144665\\
0	0.5	-0.140120409589695	0.271326576830905\\
0	1	-0.116662105275934	0.247868272517144\\
0	1.5	-0.0932038009621735	0.224409968203384\\
0	2	-0.0697454855885589	0.200951619650207\\
0	2.5	-0.0462871978645794	0.177493315336446\\
0	3	-0.0228289046106728	0.154035011022686\\
0	3.5	0.000629383113306718	0.130576662469509\\
0	4	0.024087582358454	0.107118258617062\\
0	4.5	0.0475454000386373	0.0836594510799432\\
0	5	0.00529893643383751	-0.00550364050712227\\
0	5.5	-0.036742842452422	-0.0944620252559395\\
0	6	-0.0132850192423118	-0.117920827263131\\
0	6.5	0.0101731910626896	-0.141379220055724\\
0	7	0.0336314843165961	-0.164837590728609\\
0.5	-2	-0.280354160563493	0.411560349924411\\
0.5	-1.5	-0.256895812010316	0.388102001371234\\
0.5	-1	-0.233437507696555	0.364643697057474\\
0.5	-0.5	-0.209979225502503	0.341185414863421\\
0.5	0	-0.186520899069034	0.317727044190536\\
0.5	0.5	-0.16306263899469	0.294268784116192\\
0.5	1	-0.139604290441513	0.270810435563015\\
0.5	1.5	-0.116146008247461	0.247352153368963\\
0.5	2	-0.0926876707541379	0.223893782696078\\
0.5	2.5	-0.0692293885600854	0.200435478382317\\
0.5	3	-0.0457711284857409	0.176977163008703\\
0.5	3.5	-0.0223128960610315	0.153518737036548\\
0.5	4	0.00114491608922472	0.13005994608921\\
0.5	4.5	-0.0410200032118091	0.0409784126307278\\
0.5	5	-0.0831433098120535	-0.048061557896308\\
0.5	5.5	-0.0596855087216513	-0.0715203488436459\\
0.5	6	-0.036227303946577	-0.0949787305763847\\
0.5	6.5	-0.0127689996328164	-0.118437101249269\\
0.5	7	0.0106873083772917	-0.141893436909013\\
1	-2	-0.303296323609364	0.434502468730866\\
1	-1.5	-0.279838019295603	0.411044164417105\\
1	-1	-0.256379714981842	0.387585904342761\\
1	-0.5	-0.232921410668082	0.364127555789584\\
1	0	-0.209463106354321	0.340669251475823\\
1	0.5	-0.186004779920853	0.317210925042355\\
1	1	-0.162546453487384	0.293752598608886\\
1	1.5	-0.13908819341304	0.270294338534542\\
1	2	-0.115629889099279	0.246836012101073\\
1	2.5	-0.0921716179650805	0.223377641428188\\
1	3	-0.0687134076600791	0.199919248635595\\
1	3.5	-0.0452555955098229	0.176460457688257\\
1	4	-0.0873389428574556	0.0874604574736874\\
1	4.5	-0.129543826941028	-0.00166102970747933\\
1	5	-0.106086009260845	-0.0251198372445983\\
1	5.5	-0.0826277934259166	-0.0485782355671181\\
1	6	-0.0591695333515721	-0.0720365675305138\\
1	6.5	-0.0357131811020479	-0.0954929474296731\\
1	7	-0.0122571717079986	-0.118948973413503\\
1.5	-2	-0.326238553014359	0.457444676016153\\
1.5	-1.5	-0.302780182341474	0.433986371702392\\
1.5	-1	-0.279321878027713	0.410528067388632\\
1.5	-0.5	-0.255863573713953	0.387069763074871\\
1.5	0	-0.232405269400192	0.363611414521694\\
1.5	0.5	-0.20894698720614	0.340153132327642\\
1.5	1	-0.185488693952233	0.316694805894173\\
1.5	1.5	-0.162030422818035	0.293236523700121\\
1.5	2	-0.13857210744442	0.269778153027236\\
1.5	2.5	-0.1151139137292	0.246319760234643\\
1.5	3	-0.0916560849891625	0.222860969287305\\
1.5	3.5	-0.133657893562956	0.133942529966282\\
1.5	4	-0.175944316420368	0.0447394708317143\\
1.5	4.5	-0.152486520859893	0.0212806688245223\\
1.5	5	-0.129028282905256	-0.00217774608777858\\
1.5	5.5	-0.105570011771058	-0.0256360918759918\\
1.5	6	-0.0821136982310225	-0.0490924192408444\\
1.5	6.5	-0.0586576501274841	-0.0725484894640909\\
1.5	7	-0.0352016020239457	-0.0960045541574103\\
2	-2	-0.345548527278405	0.476754694519615\\
2	-1.5	-0.324631666822285	0.455837856183203\\
2	-1	-0.302264085313	0.433470274673919\\
2	-0.5	-0.27880578099924	0.410011926120742\\
2	0	-0.255347498805187	0.386553666046397\\
2	0.5	-0.231889216611135	0.363095317493221\\
2	1	-0.208430868057958	0.339636968940044\\
2	1.5	-0.184972607983614	0.316178664626283\\
2	2	-0.161514381088831	0.292720227594274\\
2	2.5	-0.138056574468502	0.269261436646937\\
2	3	-0.17997683873853	0.180424607988804\\
2	3.5	-0.222344828019415	0.0911399769008349\\
2	4	-0.198887010339232	0.0676811638337889\\
2	4.5	-0.175428805564158	0.044222759981342\\
2	5	-0.151970512310251	0.0207644114281653\\
2	5.5	-0.128514176650508	-0.00269194082135892\\
2	6	-0.105058161726532	-0.0261479833949703\\
2	6.5	-0.0816021136229934	-0.0496040425583627\\
2	7	-0.058146071049382	-0.0730600851319741\\
2.5	-2	-0.355634826597905	0.486841015958823\\
2.5	-1.5	-0.336227923773136	0.467434113134055\\
2.5	-1	-0.316821065187784	0.448027254548702\\
2.5	-0.5	-0.296299240597921	0.427505385719423\\
2.5	0	-0.275382402261509	0.406588503143595\\
2.5	0.5	-0.254465497565973	0.385671620567767\\
2.5	1	-0.231373119582661	0.362579131985915\\
2.5	1.5	-0.207914903747733	0.339120739193322\\
2.5	2	-0.184457075007696	0.315661948245984\\
2.5	2.5	-0.226295761794395	0.226906619652201\\
2.5	3	-0.268745317498755	0.13754045532032\\
2.5	3.5	-0.245287499818572	0.114081653313128\\
2.5	4	-0.221829283983643	0.0906232549906086\\
2.5	4.5	-0.198371023909299	0.0671649174972859\\
2.5	5	-0.174914666129848	0.0437085541879077\\
2.5	5.5	-0.151458640146017	0.0202525060843693\\
2.5	6	-0.128002614162187	-0.00320354754909614\\
2.5	6.5	-0.104546543938941	-0.0266596011825615\\
2.5	7	-0.0810905179551102	-0.050115660345954\\
3	-2	-0.354686466234508	0.48589261135601\\
3	-1.5	-0.335301373441887	0.46650751856339\\
3	-1	-0.315916302768975	0.447122425770769\\
3	-0.5	-0.296531209976354	0.427737332978148\\
3	0	-0.277146117183733	0.408352195946111\\
3	0.5	-0.257761090750237	0.38896710315349\\
3	1	-0.23837606431674	0.369581899762329\\
3	1.5	-0.218991480277405	0.350196342455839\\
3	2	-0.264863286614384	0.265554306294202\\
3	2.5	-0.311067154008772	0.179862325009899\\
3	3	-0.289141103992088	0.157935268546499\\
3	3.5	-0.26699951954006	0.135793507136806\\
3	4	-0.243563047378155	0.112356946496069\\
3	4.5	-0.22012852175056	0.0889223987487659\\
3	5	-0.196694283679169	0.0654881385576673\\
3	5.5	-0.173260023488071	0.0420538562468607\\
3	6	-0.149825763296972	0.0186196071156162\\
3	6.5	-0.126395495713179	-0.0048106715280314\\
3	7	-0.103076611919325	-0.028129555321885\\
3.5	-2	-0.345600751909137	0.476806941270055\\
3.5	-1.5	-0.326215681236224	0.457421848477434\\
3.5	-1	-0.306830566323895	0.438036711445397\\
3.5	-0.5	-0.28744551777069	0.418651618652776\\
3.5	0	-0.268060447097778	0.39926648162074\\
3.5	0.5	-0.248675442783989	0.379881300349287\\
3.5	1	-0.229290858744654	0.360495720923088\\
3.5	1.5	-0.275155321338557	0.275861072743944\\
3.5	2	-0.321725513218139	0.190520651039704\\
3.5	2.5	-0.302340929178804	0.171135093733214\\
3.5	3	-0.282955924865015	0.151749901401907\\
3.5	3.5	-0.26357087631181	0.132364775429724\\
3.5	4	-0.244187752173207	0.112981629171413\\
3.5	4.5	-0.223721072015539	0.0925149490137455\\
3.5	5	-0.201796813695208	0.0705906796335604\\
3.5	5.5	-0.179872566434731	0.0486663991935213\\
3.5	6	-0.157996097743666	0.0267899415623095\\
3.5	6.5	-0.136187226880433	0.00498104857936937\\
3.5	7	-0.114378333897493	-0.0168278333437168\\
4	-2	-0.336515081823181	0.467721226944683\\
4	-1.5	-0.31712998903056	0.448336089912646\\
4	-1	-0.29774489623794	0.428950997120026\\
4	-0.5	-0.278359847684735	0.409565904327405\\
4	0	-0.258974843370947	0.390180678816536\\
4	0.5	-0.239590237211903	0.370795099390338\\
4	1	-0.28544735606273	0.286167794954269\\
4	1.5	-0.332024913805096	0.200820095866078\\
4	2	-0.312640307646053	0.181434472200463\\
4	2.5	-0.293255303332264	0.162049279869156\\
4	3	-0.273870276898768	0.142664153896974\\
4	3.5	-0.254487174879872	0.12328102975837\\
4	4	-0.235104360417182	0.103898193175972\\
4	4.5	-0.215721523834783	0.0845153510636461\\
4	5	-0.196338665132677	0.0651325144812475\\
4	5.5	-0.177047470500774	0.0458412977296366\\
4	6	-0.15778001031562	0.026573854134264\\
4	6.5	-0.136703843722046	0.00549766542098236\\
4	7	-0.11489496179896	-0.0163111999123227\\
4.5	-2	-0.32742936749781	0.458635512619312\\
4.5	-1.5	-0.308044274705189	0.439250419826691\\
4.5	-1	-0.288659226151984	0.419865282794654\\
4.5	-0.5	-0.269274221838196	0.400480101523201\\
4.5	0	-0.24988963779886	0.381094522097003\\
4.5	0.5	-0.295739435026319	0.296474517164594\\
4.5	1	-0.342324314392053	0.211119452213619\\
4.5	1.5	-0.32293970823301	0.191733872787421\\
4.5	2	-0.303554726038929	0.172348680456114\\
4.5	2.5	-0.284169655366017	0.152963554483931\\
4.5	3	-0.264786553347121	0.133580452465036\\
4.5	3.5	-0.245403738884431	0.114197593762929\\
4.5	4	-0.226020902302032	0.0948147571805303\\
4.5	4.5	-0.206658106175137	0.0754519555237078\\
4.5	5	-0.187390668109691	0.0561845174582622\\
4.5	5.5	-0.168123207924538	0.0369170572731086\\
4.5	6	-0.148855780918946	0.0176496026178819\\
4.5	6.5	-0.129588331793646	-0.00161783544756369\\
4.5	7	-0.110544955311002	-0.0206612146951714\\
5	-2	-0.318343697411854	0.44954979829394\\
5	-1.5	-0.298958604619233	0.430164661261903\\
5	-1	-0.279573622425153	0.41077947999045\\
5	-0.5	-0.260188994146401	0.391393900564252\\
5	0	-0.306031469750492	0.306781283614336\\
5	0.5	-0.352623670739594	0.221418852800576\\
5	1	-0.333239108819967	0.202033273374378\\
5	1.5	-0.313854104506179	0.182648081043071\\
5	2	-0.294469055952974	0.163262955070888\\
5	2.5	-0.275085931814371	0.143879830932285\\
5	3	-0.25570311735168	0.124496983290032\\
5	3.5	-0.236320280769282	0.105114124587925\\
5	4	-0.217001325903762	0.0857951586625523\\
5	4.5	-0.197733865718609	0.0665277095372527\\
5	5	-0.178466427653163	0.0472602604119531\\
5	5.5	-0.15919896746801	0.0279928002267994\\
5	6	-0.13993151834271	0.0087253511014998\\
5	6.5	-0.120664080277264	-0.0105421035537268\\
5	7	-0.107204757735434	-0.0240014205656304\\
5.5	-2	-0.309258027325899	0.440464039729152\\
5.5	-1.5	-0.28987302301211	0.421078858457699\\
5.5	-1	-0.270488416853067	0.401693279031501\\
5.5	-0.5	-0.316323504474665	0.317088005824661\\
5.5	0	-0.36292309344626	0.231718231267825\\
5.5	0.5	-0.343538487287216	0.212332651841627\\
5.5	1	-0.324153482973428	0.192947470570174\\
5.5	1.5	-0.304768456539931	0.173562355657845\\
5.5	2	-0.285385332401328	0.154179187279826\\
5.5	2.5	-0.266002540058345	0.134796372817135\\
5.5	3	-0.246619681356239	0.115413525174883\\
5.5	3.5	-0.227344545632388	0.0961383783911779\\
5.5	4	-0.208077085447234	0.0768709126760972\\
5.5	4.5	-0.188809625262081	0.0576034690807246\\
5.5	5	-0.169542187196635	0.038336014425498\\
5.5	5.5	-0.150274727011481	0.0190685542403443\\
5.5	6	-0.131007288946036	-0.000198900414882317\\
5.5	6.5	-0.111739828760882	-0.0194663467752184\\
5.5	7	-0.105357231239367	-0.025848958121551\\
6	-2	-0.300172401479359	0.431378236924948\\
6	-1.5	-0.280787817440024	0.41199265749875\\
6	-1	-0.326615539198838	0.327394728034986\\
6	-0.5	-0.373222471913509	0.242017609735074\\
6	0	-0.353837865754465	0.222632030308876\\
6	0.5	-0.334452905680093	0.203246893276839\\
6	1	-0.31506783500718	0.183861734125094\\
6	1.5	-0.295684755107993	0.164478609986491\\
6	2	-0.276301918525594	0.145095773404092\\
6	2.5	-0.256955181306751	0.125749036185249\\
6	3	-0.237687743241306	0.106481587059949\\
6	3.5	-0.218420283056152	0.0872141268747958\\
6	4	-0.199152844990706	0.0679466832794232\\
6	4.5	-0.179885373745699	0.0486792175643425\\
6	5	-0.160617946740107	0.0294117850288239\\
6	5.5	-0.141350486554953	0.0101443137838162\\
6	6	-0.122083048489508	-0.00912313534148338\\
6	6.5	-0.10816971893973	-0.0230364538314069\\
6	7	-0.104149760615747	-0.0270564010955364\\
6.5	-2	-0.291087195907273	0.422292080205416\\
6.5	-1.5	-0.336907573923011	0.337701450245312\\
6.5	-1	-0.383521850380758	0.25231703244174\\
6.5	-0.5	-0.364137310580839	0.23293147513525\\
6.5	0	-0.34475230626705	0.21354624962438\\
6.5	0.5	-0.325367213474429	0.194161112592344\\
6.5	1	-0.305984111455534	0.174778010573448\\
6.5	1.5	-0.286601296992844	0.155395151871342\\
6.5	2	-0.267298401035377	0.136092222734313\\
6.5	2.5	-0.248030940850223	0.116824773609013\\
6.5	3	-0.228763502784777	0.0975573355435673\\
6.5	3.5	-0.209496042599624	0.0782898808883407\\
6.5	4	-0.19022858241447	0.059022420703187\\
6.5	4.5	-0.170961155408879	0.0397549881676685\\
6.5	5	-0.151693684163871	0.0204875224525878\\
6.5	5.5	-0.132426246098425	0.00122006226743413\\
6.5	6	-0.113158785913272	-0.0180473813279385\\
6.5	6.5	-0.10632218138381	-0.0248839858574005\\
6.5	7	-0.104311223424736	-0.026894954876328\\
7	-2	-0.3471996528866	0.348008260934469\\
7	-1.5	-0.393821228848007	0.262616433028697\\
7	-1	-0.37443666692838	0.243230831482791\\
7	-0.5	-0.355051662614591	0.223845650211338\\
7	0	-0.335666636181095	0.204460535299009\\
7	0.5	-0.316283512042491	0.185077366920989\\
7	1	-0.29690903670974	0.165702891588238\\
7	1.5	-0.277641576524586	0.146435431403084\\
7	2	-0.258374138459141	0.12716797121793\\
7	2.5	-0.239106700393695	0.107900533152485\\
7	3	-0.219839240208541	0.0886330729673312\\
7	3.5	-0.200571780023388	0.0693656238420316\\
7	4	-0.181304341957942	0.0500981581269509\\
7	4.5	-0.162036903892496	0.0308307200615053\\
7	5	-0.142769421587635	0.0115632543464246\\
7	5.5	-0.123501983522189	-0.00770418924894799\\
7	6	-0.109134669084173	-0.0220714981570375\\
7	6.5	-0.104474643827889	-0.0267315261782846\\
7	7	-0.104472675173872	-0.0267334975972655\\
};

\addplot [color=blue, line width=1.5pt, draw=none, mark size=3.0pt, mark=x, mark options={solid, blue}]
  table[row sep=crcr]{%
-0.705226838588715	1.27091217041016\\
};

\addplot [color=blue, line width=1.5pt, mark size=3.0pt, mark=x, mark options={solid, blue}, forget plot]
  table[row sep=crcr]{%
-0.705226838588715	1.27091217041016\\
0	0\\
};
\addplot [color=blue, line width=1.5pt, draw=none, mark size=3.0pt, mark=x, mark options={solid, blue}, forget plot]
  table[row sep=crcr]{%
-1.01386404037476	2.14523482322693\\
};
\addplot [color=blue, line width=1.5pt, mark size=3.0pt, mark=x, mark options={solid, blue}, forget plot]
  table[row sep=crcr]{%
-1.01386404037476	2.14523482322693\\
-0.705226838588715	1.27091217041016\\
};
\addplot [color=blue, line width=1.5pt, draw=none, mark size=3.0pt, mark=x, mark options={solid, blue}, forget plot]
  table[row sep=crcr]{%
-1.08458864688873	2.78164482116699\\
};
\addplot [color=blue, line width=1.5pt, mark size=3.0pt, mark=x, mark options={solid, blue}, forget plot]
  table[row sep=crcr]{%
-1.08458864688873	2.78164482116699\\
-1.01386404037476	2.14523482322693\\
};
\addplot [color=blue, line width=1.5pt, draw=none, mark size=3.0pt, mark=x, mark options={solid, blue}, forget plot]
  table[row sep=crcr]{%
-1.01259076595306	3.27533221244812\\
};
\addplot [color=blue, line width=1.5pt, mark size=3.0pt, mark=x, mark options={solid, blue}, forget plot]
  table[row sep=crcr]{%
-1.01259076595306	3.27533221244812\\
-1.08458864688873	2.78164482116699\\
};
\addplot [color=blue, line width=1.5pt, draw=none, mark size=3.0pt, mark=x, mark options={solid, blue}, forget plot]
  table[row sep=crcr]{%
-0.854974150657654	3.68340086936951\\
};
\addplot [color=blue, line width=1.5pt, mark size=3.0pt, mark=x, mark options={solid, blue}, forget plot]
  table[row sep=crcr]{%
-0.854974150657654	3.68340086936951\\
-1.01259076595306	3.27533221244812\\
};
\addplot [color=blue, line width=1.5pt, draw=none, mark size=3.0pt, mark=x, mark options={solid, blue}, forget plot]
  table[row sep=crcr]{%
-0.645995318889618	4.04010725021362\\
};
\addplot [color=blue, line width=1.5pt, mark size=3.0pt, mark=x, mark options={solid, blue}, forget plot]
  table[row sep=crcr]{%
-0.645995318889618	4.04010725021362\\
-0.854974150657654	3.68340086936951\\
};
\addplot [color=blue, line width=1.5pt, draw=none, mark size=3.0pt, mark=x, mark options={solid, blue}, forget plot]
  table[row sep=crcr]{%
-0.406204640865326	4.36600160598755\\
};
\addplot [color=blue, line width=1.5pt, mark size=3.0pt, mark=x, mark options={solid, blue}, forget plot]
  table[row sep=crcr]{%
-0.406204640865326	4.36600160598755\\
-0.645995318889618	4.04010725021362\\
};
\addplot [color=blue, line width=1.5pt, draw=none, mark size=3.0pt, mark=x, mark options={solid, blue}, forget plot]
  table[row sep=crcr]{%
-0.147930383682251	4.67341136932373\\
};
\addplot [color=blue, line width=1.5pt, mark size=3.0pt, mark=x, mark options={solid, blue}, forget plot]
  table[row sep=crcr]{%
-0.147930383682251	4.67341136932373\\
-0.406204640865326	4.36600160598755\\
};
\addplot [color=blue, line width=1.5pt, draw=none, mark size=3.0pt, mark=x, mark options={solid, blue}, forget plot]
  table[row sep=crcr]{%
0.121429860591888	4.96973037719727\\
};
\addplot [color=blue, line width=1.5pt, mark size=3.0pt, mark=x, mark options={solid, blue}, forget plot]
  table[row sep=crcr]{%
0.121429860591888	4.96973037719727\\
-0.147930383682251	4.67341136932373\\
};
\addplot [color=blue, line width=1.5pt, draw=none, mark size=3.0pt, mark=x, mark options={solid, blue}, forget plot]
  table[row sep=crcr]{%
0.0490329042077065	4.90074729919434\\
};
\addplot [color=blue, line width=1.5pt, mark size=3.0pt, mark=x, mark options={solid, blue}, forget plot]
  table[row sep=crcr]{%
0.0490329042077065	4.90074729919434\\
0.121429860591888	4.96973037719727\\
};
\addplot [color=blue, line width=1.5pt, draw=none, mark size=3.0pt, mark=x, mark options={solid, blue}, forget plot]
  table[row sep=crcr]{%
0.0835437476634979	4.93687295913696\\
};
\addplot [color=blue, line width=1.5pt, mark size=3.0pt, mark=x, mark options={solid, blue}, forget plot]
  table[row sep=crcr]{%
0.0835437476634979	4.93687295913696\\
0.0490329042077065	4.90074729919434\\
};
\addplot [color=blue, line width=1.5pt, draw=none, mark size=3.0pt, mark=x, mark options={solid, blue}, forget plot]
  table[row sep=crcr]{%
0.0655843988060951	4.91966104507446\\
};
\addplot [color=blue, line width=1.5pt, mark size=3.0pt, mark=x, mark options={solid, blue}, forget plot]
  table[row sep=crcr]{%
0.0655843988060951	4.91966104507446\\
0.0835437476634979	4.93687295913696\\
};
\addplot [color=blue, line width=1.5pt, draw=none, mark size=3.0pt, mark=x, mark options={solid, blue}, forget plot]
  table[row sep=crcr]{%
0.0741171911358833	4.92856311798096\\
};
\addplot [color=blue, line width=1.5pt, mark size=3.0pt, mark=x, mark options={solid, blue}, forget plot]
  table[row sep=crcr]{%
0.0741171911358833	4.92856311798096\\
0.0655843988060951	4.91966104507446\\
};
\addplot [color=blue, line width=1.5pt, draw=none, mark size=3.0pt, mark=x, mark options={solid, blue}, forget plot]
  table[row sep=crcr]{%
0.0701527744531631	4.92477989196777\\
};
\addplot [color=blue, line width=1.5pt, mark size=3.0pt, mark=x, mark options={solid, blue}, forget plot]
  table[row sep=crcr]{%
0.0701527744531631	4.92477989196777\\
0.0741171911358833	4.92856311798096\\
};
\addplot [color=blue, line width=1.5pt, draw=none, mark size=3.0pt, mark=x, mark options={solid, blue}, forget plot]
  table[row sep=crcr]{%
0.0719954818487167	4.92671298980713\\
};
\addplot [color=blue, line width=1.5pt, mark size=3.0pt, mark=x, mark options={solid, blue}, forget plot]
  table[row sep=crcr]{%
0.0719954818487167	4.92671298980713\\
0.0701527744531631	4.92477989196777\\
};

\addplot [color=mycolor2, line width=1.5pt, draw=none, mark size=3.0pt, mark=x, mark options={solid, mycolor2}]
  table[row sep=crcr]{%
0.0564122048664746	4.7717430512871\\
0.526380372501768	4.7896522953997\\
-0.884209330862958	5.4367437241026\\
0.089995154469105	5.14141897969074\\
0.610385882491538	5.4908836008354\\
-0.693364070255412	5.30021720482627\\
-0.693776926268841	4.52996803371776\\
0.0651002292560559	5.69779026083464\\
0.330846431068311	4.47001001843488\\
0.44109520484135	5.00479066343026\\
-0.814509535557422	5.28625612011379\\
0.883248203822483	4.94332041567406\\
0.917186792741197	4.89788472159849\\
-0.758406129956736	4.641460056008\\
-0.387842285894774	6.0322800312358\\
-0.348794152393402	5.41503160351195\\
-0.121038124937394	5.46913408235362\\
-0.0844850967268441	4.00092599377353\\
-0.429381064585652	4.96229005696987\\
0.19475090437651	4.61853477341041\\
0.843954502042891	5.01149276027048\\
-0.142591857804827	6.03534692952404\\
-0.258785247909006	5.20486447765904\\
0.743591140369287	4.55209537197808\\
0.20561379726423	4.31929928019662\\
-0.367037787465725	4.84549875602467\\
0.506504042789148	6.14736859963147\\
-0.394133520307896	5.15435369724667\\
0.0175270807338539	3.97471424765517\\
-0.376939637719114	4.26011714002417\\
0.597866891228481	4.37830483248397\\
-0.179367604824169	4.95714955639856\\
-0.0620136943089292	4.91361081056674\\
0.0830829217019989	5.53158130068721\\
0.267087059091798	5.49179506451985\\
-0.78150157615695	4.29252956776867\\
0.742335271893561	4.39888424791663\\
0.329491271036343	5.21301056383554\\
0.992408019378515	3.74549029229844\\
-0.441152396757815	4.73838054114084\\
};

\end{axis}
\end{tikzpicture}

%% file: f3.tex
%
%
\definecolor{mycolor1}{rgb}{0.00000,0.44700,0.74100}%
\definecolor{mycolor2}{rgb}{0.85000,0.32500,0.09800}%
\begin{tikzpicture}

\begin{axis}[%
width=1.975in,
height=1.7in,
at={(1.011in,0.642in)},
scale only axis,
xmin=-2,
xmax=7,
ymin=-2,
ymax=7,
axis background/.style={fill=white},
xmajorgrids,
ymajorgrids,
tick label style={font=\scriptsize},
legend style={font=\scriptsize, legend cell align=left, align=left, draw=white!15!black}
]
\addplot [color=red, line width=1.5pt]
  table[row sep=crcr]{%
7	-2\\
-2	7\\
};
\addlegendentry{$Au=f$}

\addplot[-Straight Barb, color=mycolor1] 
table[row sep=crcr] {
x	y	u	v\\
-2	-2	0.424750526787253	0.424750526787253\\
};
\addlegendentry{Network flow field}

\addplot [color=blue, line width=1.5pt, mark size=3.0pt, mark=x, mark options={solid, blue}]
  table[row sep=crcr]{%
0.118062905967236	4.8792724609375\\
0.118074178695679	4.8792610168457\\
};
\addlegendentry{Path of network}

\addplot [color=mycolor2, only marks, line width=1.5pt, draw=none, mark size=3.0pt, mark=x, mark options={solid, mycolor2}]
table[row sep=crcr] {
x y \\
-0.441152396757815	4.73838054114084\\
};
\addlegendentry{Training examples}

\addplot[-Straight Barb, color=mycolor1, line width=0.5pt, point meta={sqrt((\thisrow{u})^2+(\thisrow{v})^2)}, point meta min=0, quiver={scale arrows=1.5,u=\thisrow{u}, v=\thisrow{v}, every arrow/.append style={-{Straight Barb[angle'=18.263, scale={(2/1000)*\pgfplotspointmetatransformed}]}}}]
 table[row sep=crcr] {%
x	y	u	v\\
-2	-2	0.301562295649196	0.522063680552409\\
-2	-1.5	0.283226762887044	0.494642277668316\\
-2	-1	0.267024460217303	0.465087579235897\\
-2	-0.5	0.256102158940745	0.43025281395438\\
-2	0	0.245179879482826	0.395418135947416\\
-2	0.5	0.234257621843545	0.360583414303175\\
-2	1	0.223335298748349	0.325748670840296\\
-2	1.5	0.212413041109067	0.290913971014694\\
-2	2	0.201490761651148	0.256079249370453\\
-2	2.5	0.19056846037459	0.221244527726213\\
-2	3	0.179646180916671	0.18640982790061\\
-2	3.5	0.16872391236807	0.15157510625637\\
-2	4	0.157801611091513	0.116740384612129\\
-2	4.5	0.146879342542912	0.0819056575132288\\
-2	5	0.135957063084993	0.0470709467783073\\
-2	5.5	0.125034761808435	0.0122362305887263\\
-2	6	0.114112493259835	-0.0225984910555144\\
-2	6.5	0.103190202892596	-0.0574332072450954\\
-2	7	0.0919430439092958	-0.0925928029600572\\
-1.5	-2	0.276779573434073	0.501089467121287\\
-1.5	-1.5	0.255580911670004	0.476531062327282\\
-1.5	-1	0.237296936350328	0.44905812381935\\
-1.5	-0.5	0.22216571063834	0.418432348429179\\
-1.5	0	0.211243409361782	0.383597583147661\\
-1.5	0.5	0.200321151722501	0.348762883322059\\
-1.5	1	0.189398828627305	0.31392813985918\\
-1.5	1.5	0.178476570988023	0.279093461852216\\
-1.5	2	0.167554280620785	0.244258718389337\\
-1.5	2.5	0.156631990253546	0.209424018563735\\
-1.5	3	0.145709710795627	0.174589286010175\\
-1.5	3.5	0.134787431337707	0.139754575275254\\
-1.5	4	0.123865140970468	0.104919842721694\\
-1.5	4.5	0.112942872421868	0.0700851319867723\\
-1.5	5	0.10202058205463	0.0352504157971912\\
-1.5	5.5	0.0910982971420505	0.000415699607610181\\
-1.5	6	0.0801760122294714	-0.0344190193093007\\
-1.5	6.5	0.0689835634818927	-0.0695238856972324\\
-1.5	7	0.0583314369496537	-0.104088454415793\\
-1	-2	0.251996807581674	0.480115166415612\\
-1	-1.5	0.230798167636242	0.455556805258883\\
-1	-1	0.209599571328088	0.430998444102154\\
-1	-0.5	0.191367087994974	0.403473926333108\\
-1	0	0.177306928331419	0.371777030347907\\
-1	0.5	0.166384659782818	0.336942352340943\\
-1	1	0.155462380324899	0.302107652515341\\
-1	1.5	0.14454008995766	0.267272909052462\\
-1	2	0.133617799590422	0.232438209226859\\
-1	2.5	0.122695520132502	0.197603487582619\\
-1	3	0.111773240674583	0.162768765938378\\
-1	3.5	0.100850950307344	0.127934055203457\\
-1	4	0.0899286708494245	0.093099333559216\\
-1	4.5	0.0790063968461646	0.0582646119149754\\
-1	5	0.0680841064789259	0.0234298957253943\\
-1	5.5	0.0571618106570277	-0.0114048177368569\\
-1	6	0.0460240994184683	-0.0464549793437268\\
-1	6.5	0.0353172462865291	-0.0810742528433488\\
-1	7	0.02439496137395	-0.11590896357827\\
-0.5	-2	0.227214063547912	0.45914095298449\\
-0.5	-1.5	0.206015467239758	0.434582591827761\\
-0.5	-1	0.184816827294326	0.410024187033755\\
-0.5	-0.5	0.163618198258214	0.385465825877027\\
-0.5	0	0.145437261458257	0.357889772484142\\
-0.5	0.5	0.132448189661774	0.325121843178465\\
-0.5	1	0.121525888385216	0.290287121534225\\
-0.5	1.5	0.110603608927297	0.255452399889984\\
-0.5	2	0.0996813294693775	0.220617678245743\\
-0.5	2.5	0.0887590445567985	0.185782956601503\\
-0.5	3	0.077836765098879	0.150948245866581\\
-0.5	3.5	0.0669144747316403	0.116113524222341\\
-0.5	4	0.0559921898190612	0.0812788080327595\\
-0.5	4.5	0.0450699158158014	0.0464440809338593\\
-0.5	5	0.0341476254485627	0.0116093688352729\\
-0.5	5.5	0.0230646135364056	-0.0233860511715829\\
-0.5	6	0.0123030488050801	-0.0580600567255647\\
-0.5	6.5	0.0013807679834957	-0.0928947729151457\\
-0.5	7	-0.00954151965641317	-0.127729472740748\\
0	-2	0.202431319514151	0.438166695916091\\
0	-1.5	0.181232712296677	0.413608334759362\\
0	-1	0.160034083260565	0.389049929965357\\
0	-0.5	0.138835443315134	0.364491568808628\\
0	0	0.11763683609766	0.339933164014622\\
0	0.5	0.0995074349215411	0.312305596816538\\
0	1	0.0875894073548533	0.27846656873447\\
0	1.5	0.0766671333515934	0.243631868908868\\
0	2	0.0657448429843547	0.208797147264627\\
0	2.5	0.0548225635264352	0.173962425620387\\
0	3	0.043900267704537	0.139127703976146\\
0	3.5	0.0329779991559367	0.104293004150544\\
0	4	0.0220557142433576	0.069458282506303\\
0	4.5	0.0111334388764328	0.0346235608620624\\
0	5	0.00010515356397595	-0.000317150272736895\\
0	5.5	-0.0107111418580445	-0.0350458660624401\\
0	6	-0.0216334172249693	-0.0698805822520212\\
0	6.5	-0.0325556994102185	-0.104715282077623\\
0	7	-0.0434779897774572	-0.139550003721864\\
0.5	-2	0.177648586389709	0.417192438847692\\
0.5	-1.5	0.156449935534958	0.392634034053687\\
0.5	-1	0.135251339226803	0.368075672896958\\
0.5	-0.5	0.114052710190691	0.343517311740229\\
0.5	0	0.0928540811545793	0.318958906946224\\
0.5	0.5	0.0716554521184671	0.294400545789495\\
0.5	1	0.0536529590524476	0.266646059571992\\
0.5	1.5	0.0427306686852089	0.231811337927752\\
0.5	2	0.0318083728633107	0.196976616283511\\
0.5	2.5	0.0208860879507316	0.16214190554859\\
0.5	3	0.00996381394747173	0.127307194813668\\
0.5	3.5	-0.000958476419766935	0.0924724731694276\\
0.5	4	-0.0118807558776864	0.0576377569798465\\
0.5	4.5	-0.022854315954108	0.0227517628990932\\
0.5	5	-0.0337253311575041	-0.0120316713083209\\
0.5	5.5	-0.0446476160700832	-0.046866386134237\\
0.5	6	-0.0555698900733431	-0.0817011132331373\\
0.5	6.5	-0.0664921749859222	-0.11653581305874\\
0.5	7	-0.077415027183099	-0.151369967418382\\
1	-2	0.152865820537309	0.396218181779294\\
1	-1.5	0.131667213319835	0.371659820622565\\
1	-1	0.110468584283723	0.347101415828559\\
1	-0.5	0.08926996615693	0.32254305467183\\
1	0	0.0680713371208179	0.297984649877825\\
1	0.5	0.0468727299033441	0.273426310539734\\
1	1	0.0256741117765512	0.248867927564367\\
1	1.5	0.00879419856416491	0.219990828765274\\
1	2	-0.00212809725773335	0.185156107121033\\
1	2.5	-0.013050387624972	0.150321385476793\\
1	3	-0.0239726670828915	0.115486663832552\\
1	3.5	-0.0348949574501302	0.0806519530976307\\
1	4	-0.0458137895631866	0.045820678798253\\
1	4.5	-0.0567395218206287	0.0109825179911388\\
1	5	-0.0676617958238886	-0.0238522009257721\\
1	5.5	-0.0785840916457868	-0.0586869062060339\\
1	6	-0.0895063601943871	-0.0935216333049342\\
1	6.5	-0.100429217846224	-0.128355787664577\\
1	7	-0.111352446414912	-0.163189560198049\\
1.5	-2	0.128083098322186	0.375243924710895\\
1.5	-1.5	0.106884458376754	0.350685563554166\\
1.5	-1	0.0856858402499615	0.32612715876016\\
1.5	-0.5	0.0644872221231686	0.301568797603432\\
1.5	0	0.0432885821777373	0.277010414628064\\
1.5	0.5	0.0220899749602635	0.252452031652697\\
1.5	1	0.000891345924151419	0.227893670495968\\
1.5	1.5	-0.0203072885666203	0.203335287520601\\
1.5	2	-0.036064572833437	0.173335576139917\\
1.5	2.5	-0.0469868632006756	0.138500854495677\\
1.5	3	-0.0579091426585951	0.103666143760755\\
1.5	3.5	-0.0687732645359302	0.0688895906064182\\
1.5	4	-0.0797537124837533	0.0339967141089229\\
1.5	4.5	-0.0906759919416727	-0.000838015717307111\\
1.5	5	-0.101598282308911	-0.0356727264522286\\
1.5	5.5	-0.11252057267615	-0.07050743718715\\
1.5	6	-0.123443408509348	-0.105341575182814\\
1.5	6.5	-0.134366637078037	-0.140175369534924\\
1.5	7	-0.145289854737406	-0.175009163887034\\
2	-2	0.103300343379105	0.354269667642496\\
2	-1.5	0.082101714342993	0.329711306485767\\
2	-1	0.0609030962162001	0.305152901691762\\
2	-0.5	0.039704467180088	0.280594540535033\\
2	0	0.0185058490532951	0.256036179378304\\
2	0.5	-0.00269277998281705	0.231477774584298\\
2	1	-0.0238914035642696	0.206919391608931\\
2	1.5	-0.0450900271457221	0.182361019542883\\
2	2	-0.0662886507271746	0.157802647476835\\
2	2.5	-0.0809233333217197	0.12668033442388\\
2	3	-0.0917327340540141	0.0919585024145834\\
2	3.5	-0.102767903146878	0.0570109129540369\\
2	4	-0.113690182604797	0.0221761858551366\\
2	4.5	-0.124612462062717	-0.0126585248797848\\
2	5	-0.135534741520636	-0.0474932465240255\\
2	5.5	-0.146457610081792	-0.0823273899743492\\
2	6	-0.157380838650481	-0.117161189781119\\
2	6.5	-0.168304034491212	-0.151994973223909\\
2	7	-0.179227241241262	-0.186828745757381\\
2.5	-2	0.0785175993453437	0.333295410574097\\
2.5	-1.5	0.0573189593999124	0.30873702759873\\
2.5	-1	0.0361203521824387	0.284178666442001\\
2.5	-0.5	0.0149217340556458	0.259620283466634\\
2.5	0	-0.00627689498046634	0.235061900491267\\
2.5	0.5	-0.0274755185619189	0.2105035175159\\
2.5	1	-0.0486741421433714	0.185945156359171\\
2.5	1.5	-0.0698727711794835	0.161386773383804\\
2.5	2	-0.0910713893062764	0.136828390408437\\
2.5	2.5	-0.112110524095879	0.112429512588898\\
2.5	3	-0.125782093810002	0.0800250927078422\\
2.5	3.5	-0.136704373267922	0.0451903874275804\\
2.5	4	-0.14762666363516	0.0103556657833397\\
2.5	4.5	-0.158548954002399	-0.0244790613155605\\
2.5	5	-0.169471800744916	-0.0593131993112247\\
2.5	5.5	-0.180395029313605	-0.0941469991179941\\
2.5	6	-0.191318236063655	-0.128980793470104\\
2.5	6.5	-0.202241431904387	-0.163814555094256\\
2.5	7	-0.213164649563756	-0.198648338537047\\
3	-2	0.0537348444022632	0.312321153505699\\
3	-1.5	0.0325362262754702	0.28776279234897\\
3	-1	0.0113376081486773	0.263204387554964\\
3	-0.5	-0.00986103179675399	0.238646026398235\\
3	0	-0.0310596499235469	0.214087643422868\\
3	0.5	-0.052258278959659	0.189529260447501\\
3	1	-0.0734568861771328	0.164970899290772\\
3	1.5	-0.0946555152132449	0.140412516315405\\
3	2	-0.115692413592414	0.116015863996981\\
3	2.5	-0.13705276237615	0.091295772183309\\
3	3	-0.158251391412262	0.0667373837532822\\
3	3.5	-0.170640865207604	0.0333698673557835\\
3	4	-0.181563133756205	-0.00146485428845718\\
3	4.5	-0.192485991408041	-0.0362990031934405\\
3	5	-0.20340920906741	-0.0711327975455503\\
3	5.5	-0.21433242672678	-0.10596659189766\\
3	6	-0.225255622567511	-0.140800364431131\\
3	6.5	-0.236180127526545	-0.175632926030173\\
3	7	-0.247160270013431	-0.210409762826809\\
3.5	-2	0.0289521003685018	0.2913468964373\\
3.5	-1.5	0.00775349315102803	0.266788535280571\\
3.5	-1	-0.0134451467944033	0.242230130486566\\
3.5	-0.5	-0.0346437703758558	0.217671769329837\\
3.5	0	-0.0558423885026487	0.19311338635447\\
3.5	0.5	-0.0770410175387608	0.168555025197741\\
3.5	1	-0.0982396411202134	0.143996642222373\\
3.5	1.5	-0.119274303088949	0.119602215405064\\
3.5	2	-0.140636888283118	0.0948798871809583\\
3.5	2.5	-0.161835495500592	0.0703215096602507\\
3.5	3	-0.183034135446023	0.0457631321395431\\
3.5	3.5	-0.203428060370965	0.0204000668716438\\
3.5	4	-0.215500182071165	-0.0132848234396352\\
3.5	4.5	-0.226423399730535	-0.0481186068824257\\
3.5	5	-0.237346617389905	-0.0829523903252163\\
3.5	5.5	-0.248269835049274	-0.117786173768007\\
3.5	6	-0.259216027734851	-0.152596971275271\\
3.5	6.5	-0.270196213859013	-0.187373840799865\\
3.5	7	-0.281176334527261	-0.222150655777862\\
4	-2	0.00416936724405955	0.27037266118754\\
4	-1.5	-0.0170292727013717	0.245814278212172\\
4	-1	-0.0382278962828243	0.221255895236805\\
4	-0.5	-0.0594265198642768	0.196697512261438\\
4	0	-0.0806251379910697	0.17213914019539\\
4	0.5	-0.101823756117863	0.147580768129342\\
4	1	-0.122856203494803	0.123188566813146\\
4	1.5	-0.144221003280768	0.0984640130879268\\
4	2	-0.165419643226199	0.0739056355672191\\
4	2.5	-0.186618261352992	0.0493472471371924\\
4	3	-0.207816879479785	0.0247888859804635\\
4	3.5	-0.229016064891175	0.000231070289694036\\
4	4	-0.248339872878453	-0.0262021446438271\\
4	4.5	-0.260360808053029	-0.0599381996620918\\
4	5	-0.27128400389376	-0.094771961286244\\
4	5.5	-0.282251971580434	-0.129561016520369\\
4	6	-0.293232135885958	-0.164337896954282\\
4	6.5	-0.304212256554205	-0.199114733750918\\
4	7	-0.315193686340755	-0.233890305066528\\
4.5	-2	-0.0206133986083402	0.249398382300502\\
4.5	-1.5	-0.0418120112804736	0.224840021143774\\
4.5	-1	-0.0630106403165857	0.200281638168406\\
4.5	-0.5	-0.084209252988719	0.175723255193039\\
4.5	0	-0.105407871115512	0.15116489403631\\
4.5	0.5	-0.126438103900658	0.126774918221229\\
4.5	1	-0.147805129187736	0.102048128085576\\
4.5	1.5	-0.169003747314529	0.0774897614741876\\
4.5	2	-0.190202365441322	0.0529313730441608\\
4.5	2.5	-0.211401005386753	0.0283729900687936\\
4.5	3	-0.232600190798144	0.00381519619666251\\
4.5	3.5	-0.253799747126387	-0.020742270395893\\
4.5	4	-0.27499930345463	-0.045299704260491\\
4.5	4.5	-0.293261569229124	-0.0727944504974764\\
4.5	5	-0.305287915426017	-0.106525116312063\\
4.5	5.5	-0.316268057912903	-0.141301953108699\\
4.5	6	-0.327248200399788	-0.176078789905335\\
4.5	6.5	-0.338228342886674	-0.210855648520609\\
4.5	7	-0.34922042016875	-0.24562061597797\\
5	-2	-0.0453961426421016	0.228424125232104\\
5	-1.5	-0.0665947716782137	0.203865764075375\\
5	-1	-0.0877933788956875	0.179307381100008\\
5	-0.5	-0.10899199702248	0.15474900903396\\
5	0	-0.130019993397193	0.130361258719992\\
5	0.5	-0.151389255094705	0.105632253992545\\
5	1	-0.172587873221498	0.0810738764718369\\
5	1.5	-0.19378649134829	0.0565154934964697\\
5	2	-0.214985109475083	0.0319571159757621\\
5	2.5	-0.236184316705112	0.00739932210363098\\
5	3	-0.257383873033355	-0.0171581335796053\\
5	3.5	-0.278583429361598	-0.0417155892628417\\
5	4	-0.299814273617269	-0.0662417570186505\\
5	4.5	-0.321070776591667	-0.0907422769650513\\
5	5	-0.338309355490955	-0.119260666439966\\
5	5.5	-0.350284144245371	-0.153042878787709\\
5	6	-0.361264286732257	-0.187819704675026\\
5	6.5	-0.372251542095253	-0.222589472232829\\
5	7	-0.383247132178107	-0.257350861433496\\
5.5	-2	-0.070178886675863	0.207449889982343\\
5.5	-1.5	-0.0913774993479963	0.182891507006976\\
5.5	-1	-0.112576133838768	0.158333124031609\\
5.5	-0.5	-0.133601904712366	0.133947610128075\\
5.5	0	-0.154973381001673	0.109216379899513\\
5.5	0.5	-0.176171999128466	0.0846580023788054\\
5.5	1	-0.197370617255259	0.0600996139487786\\
5.5	1.5	-0.21856925720069	0.0355412418827306\\
5.5	2	-0.239768442612081	0.0109834371012803\\
5.5	2.5	-0.260968020758962	-0.0135740076726369\\
5.5	3	-0.282167577087205	-0.0381314633558732\\
5.5	3.5	-0.303420109069419	-0.0626358779292237\\
5.5	4	-0.324676612043817	-0.0871364087849437\\
5.5	4.5	-0.345933115018214	-0.111636917822025\\
5.5	5	-0.367189574355335	-0.136137405040469\\
5.5	5.5	-0.383357098115509	-0.165726860563816\\
5.5	6	-0.395282685840393	-0.199558350306327\\
5.5	6.5	-0.406278275923248	-0.234319761325632\\
5.5	7	-0.417273866006102	-0.269081128707661\\
6	-2	-0.0949616307096244	0.186475632913945\\
6	-1.5	-0.116160248836417	0.161917271757216\\
6	-1	-0.137183794208902	0.137533961536157\\
6	-0.5	-0.158557495999322	0.112800505806481\\
6	0	-0.179756125035435	0.0882421282857738\\
6	0.5	-0.200954743162227	0.0636837507650662\\
6	1	-0.22215336128902	0.039125367789699\\
6	1.5	-0.243352568519049	0.0145675630082487\\
6	2	-0.264552124847292	-0.00998988176566841\\
6	2.5	-0.285769485184448	-0.0345295116213536\\
6	3	-0.307025988158846	-0.0590300206584352\\
6	3.5	-0.328282469314605	-0.0835305187861977\\
6	4	-0.349538994107641	-0.108031049641918\\
6	4.5	-0.370795453444762	-0.13253154776968\\
6	5	-0.392051912781883	-0.157032034988123\\
6	5.5	-0.41284551152455	-0.181995546440809\\
6	6	-0.428417779192619	-0.212180192600346\\
6	6.5	-0.440304966113966	-0.246050006781159\\
6	7	-0.451300599834097	-0.280811417800465\\
6.5	-2	-0.119744385652705	0.165501375845546\\
6.5	-1.5	-0.140765683705437	0.14112031294424\\
6.5	-1	-0.162141621906291	0.11638463171345\\
6.5	-0.5	-0.183340250942403	0.0918262487380827\\
6.5	0	-0.204538869069196	0.0672678657627155\\
6.5	0.5	-0.225737509014627	0.0427094936966675\\
6.5	1	-0.246936672607379	0.0181516889152172\\
6.5	1.5	-0.268136228935622	-0.00640576676801913\\
6.5	2	-0.289375364273875	-0.030923643441246\\
6.5	2.5	-0.310631867248273	-0.0554241524783276\\
6.5	3	-0.331888326585394	-0.0799246506060901\\
6.5	3.5	-0.35314485137843	-0.104425192371129\\
6.5	4	-0.374401332534189	-0.128925668680253\\
6.5	4.5	-0.39565779187131	-0.153426166808016\\
6.5	5	-0.416914294845708	-0.177926686754417\\
6.5	5.5	-0.43817891435358	-0.202419101076662\\
6.5	6	-0.457916796459909	-0.228438176435216\\
6.5	6.5	-0.473491638727306	-0.258620291632702\\
6.5	7	-0.485327333662092	-0.292541706893268\\
7	-2	-0.144347584111291	0.144706653443003\\
7	-1.5	-0.16572573690394	0.11996873580178\\
7	-1	-0.186924355030733	0.0954103746450512\\
7	-0.5	-0.208122994976164	0.0708519862150244\\
7	0	-0.229321613102957	0.046293619603636\\
7	0.5	-0.250520820332986	0.0217358039128665\\
7	1	-0.271724740388904	-0.00281726622405678\\
7	1.5	-0.292981243363302	-0.0273177970797768\\
7	2	-0.314237702700423	-0.0518182952075392\\
7	2.5	-0.33549420567482	-0.0763188260632592\\
7	3	-0.356750708649218	-0.100819335100341\\
7	3.5	-0.378007211623616	-0.125319822318784\\
7	4	-0.399263670960737	-0.149820309537227\\
7	4.5	-0.420520173935134	-0.174320829483628\\
7	5	-0.441779949705288	-0.198818065724954\\
7	5.5	-0.463051900275654	-0.223303072619471\\
7	6	-0.48432385084602	-0.247788166788542\\
7	6.5	-0.50299074326915	-0.274878319104849\\
7	7	-0.518565541899271	-0.305060456120973\\
};


\addplot [color=blue, line width=1.5pt, draw=none, mark size=3.0pt, mark=x, mark options={solid, blue}]
  table[row sep=crcr]{%
1.0279096364975	2.97098207473755\\
};

\addplot [color=blue, line width=2.0pt, mark size=3.0pt, mark=x, mark options={solid, blue}, forget plot]
  table[row sep=crcr]{%
1.0279096364975	2.97098207473755\\
0	0\\
};
\addplot [color=blue, line width=1.5pt, draw=none, mark size=3.0pt, mark=x, mark options={solid, blue}, forget plot]
  table[row sep=crcr]{%
0.806910276412964	3.99175930023193\\
};
\addplot [color=blue, line width=2.0pt, mark size=3.0pt, mark=x, mark options={solid, blue}, forget plot]
  table[row sep=crcr]{%
0.806910276412964	3.99175930023193\\
1.0279096364975	2.97098207473755\\
};
\addplot [color=blue, line width=1.5pt, draw=none, mark size=3.0pt, mark=x, mark options={solid, blue}, forget plot]
  table[row sep=crcr]{%
0.522108197212219	4.43651676177979\\
};
\addplot [color=blue, line width=2.0pt, mark size=3.0pt, mark=x, mark options={solid, blue}, forget plot]
  table[row sep=crcr]{%
0.522108197212219	4.43651676177979\\
0.806910276412964	3.99175930023193\\
};
\addplot [color=blue, line width=1.5pt, draw=none, mark size=3.0pt, mark=x, mark options={solid, blue}, forget plot]
  table[row sep=crcr]{%
0.321355998516083	4.66925811767578\\
};
\addplot [color=blue, line width=2.0pt, mark size=3.0pt, mark=x, mark options={solid, blue}, forget plot]
  table[row sep=crcr]{%
0.321355998516083	4.66925811767578\\
0.522108197212219	4.43651676177979\\
};
\addplot [color=blue, line width=1.5pt, draw=none, mark size=3.0pt, mark=x, mark options={solid, blue}, forget plot]
  table[row sep=crcr]{%
0.195228964090347	4.80169630050659\\
};
\addplot [color=blue, line width=2.0pt, mark size=3.0pt, mark=x, mark options={solid, blue}, forget plot]
  table[row sep=crcr]{%
0.195228964090347	4.80169630050659\\
0.321355998516083	4.66925811767578\\
};
\addplot [color=blue, line width=1.5pt, draw=none, mark size=3.0pt, mark=x, mark options={solid, blue}, forget plot]
  table[row sep=crcr]{%
0.118220888078213	4.87911462783813\\
};
\addplot [color=blue, line width=2.0pt, mark size=3.0pt, mark=x, mark options={solid, blue}, forget plot]
  table[row sep=crcr]{%
0.118220888078213	4.87911462783813\\
0.195228964090347	4.80169630050659\\
};
\addplot [color=blue, line width=1.5pt, draw=none, mark size=3.0pt, mark=x, mark options={solid, blue}, forget plot]
  table[row sep=crcr]{%
0.118130572140217	4.87920475006104\\
};
\addplot [color=blue, line width=2.0pt, mark size=3.0pt, mark=x, mark options={solid, blue}, forget plot]
  table[row sep=crcr]{%
0.118130572140217	4.87920475006104\\
0.118220888078213	4.87911462783813\\
};
\addplot [color=blue, line width=1.5pt, draw=none, mark size=3.0pt, mark=x, mark options={solid, blue}, forget plot]
  table[row sep=crcr]{%
0.118119291961193	4.87921619415283\\
};
\addplot [color=blue, line width=2.0pt, mark size=3.0pt, mark=x, mark options={solid, blue}, forget plot]
  table[row sep=crcr]{%
0.118119291961193	4.87921619415283\\
0.118130572140217	4.87920475006104\\
};
\addplot [color=blue, line width=1.5pt, draw=none, mark size=3.0pt, mark=x, mark options={solid, blue}, forget plot]
  table[row sep=crcr]{%
0.118096731603146	4.87923860549927\\
};
\addplot [color=blue, line width=2.0pt, mark size=3.0pt, mark=x, mark options={solid, blue}, forget plot]
  table[row sep=crcr]{%
0.118096731603146	4.87923860549927\\
0.118119291961193	4.87921619415283\\
};
\addplot [color=blue, line width=1.5pt, draw=none, mark size=3.0pt, mark=x, mark options={solid, blue}, forget plot]
  table[row sep=crcr]{%
0.118074178695679	4.8792610168457\\
};
\addplot [color=blue, line width=2.0pt, mark size=3.0pt, mark=x, mark options={solid, blue}, forget plot]
  table[row sep=crcr]{%
0.118074178695679	4.8792610168457\\
0.118096731603146	4.87923860549927\\
};
\addplot [color=blue, line width=1.5pt, draw=none, mark size=3.0pt, mark=x, mark options={solid, blue}, forget plot]
  table[row sep=crcr]{%
0.118062905967236	4.8792724609375\\
};
\addplot [color=blue, line width=2.0pt, mark size=3.0pt, mark=x, mark options={solid, blue}, forget plot]
  table[row sep=crcr]{%
0.118062905967236	4.8792724609375\\
0.118074178695679	4.8792610168457\\
};
\addplot [color=blue, line width=1.5pt, draw=none, mark size=3.0pt, mark=x, mark options={solid, blue}, forget plot]
  table[row sep=crcr]{%
0.118040360510349	4.87929487228394\\
};
\addplot [color=blue, line width=2.0pt, mark size=3.0pt, mark=x, mark options={solid, blue}, forget plot]
  table[row sep=crcr]{%
0.118040360510349	4.87929487228394\\
0.118062905967236	4.8792724609375\\
};
\addplot [color=blue, line width=1.5pt, draw=none, mark size=3.0pt, mark=x, mark options={solid, blue}, forget plot]
  table[row sep=crcr]{%
0.118029087781906	4.87930631637573\\
};
\addplot [color=blue, line width=2.0pt, mark size=3.0pt, mark=x, mark options={solid, blue}, forget plot]
  table[row sep=crcr]{%
0.118029087781906	4.87930631637573\\
0.118040360510349	4.87929487228394\\
};
\addplot [color=blue, line width=1.5pt, draw=none, mark size=3.0pt, mark=x, mark options={solid, blue}, forget plot]
  table[row sep=crcr]{%
0.1180065497756	4.87932872772217\\
};
\addplot [color=blue, line width=2.0pt, mark size=3.0pt, mark=x, mark options={solid, blue}, forget plot]
  table[row sep=crcr]{%
0.1180065497756	4.87932872772217\\
0.118029087781906	4.87930631637573\\
};
\addplot [color=blue, line width=1.5pt, draw=none, mark size=3.0pt, mark=x, mark options={solid, blue}, forget plot]
  table[row sep=crcr]{%
0.117995284497738	4.87934017181396\\
};
\addplot [color=blue, line width=1.5pt, mark size=3.0pt, mark=x, mark options={solid, blue}, forget plot]
  table[row sep=crcr]{%
0.117995284497738	4.87934017181396\\
0.1180065497756	4.87932872772217\\
};

\addplot [color=mycolor2, line width=1.5pt, draw=none, mark size=3.0pt, mark=x, mark options={solid, mycolor2}]
  table[row sep=crcr]{%
0.0564122048664746	4.7717430512871\\
0.526380372501768	4.7896522953997\\
-0.884209330862958	5.4367437241026\\
0.089995154469105	5.14141897969074\\
0.610385882491538	5.4908836008354\\
-0.693364070255412	5.30021720482627\\
-0.693776926268841	4.52996803371776\\
0.0651002292560559	5.69779026083464\\
0.330846431068311	4.47001001843488\\
0.44109520484135	5.00479066343026\\
-0.814509535557422	5.28625612011379\\
0.883248203822483	4.94332041567406\\
0.917186792741197	4.89788472159849\\
-0.758406129956736	4.641460056008\\
-0.387842285894774	6.0322800312358\\
-0.348794152393402	5.41503160351195\\
-0.121038124937394	5.46913408235362\\
-0.0844850967268441	4.00092599377353\\
-0.429381064585652	4.96229005696987\\
0.19475090437651	4.61853477341041\\
0.843954502042891	5.01149276027048\\
-0.142591857804827	6.03534692952404\\
-0.258785247909006	5.20486447765904\\
0.743591140369287	4.55209537197808\\
0.20561379726423	4.31929928019662\\
-0.367037787465725	4.84549875602467\\
0.506504042789148	6.14736859963147\\
-0.394133520307896	5.15435369724667\\
0.0175270807338539	3.97471424765517\\
-0.376939637719114	4.26011714002417\\
0.597866891228481	4.37830483248397\\
-0.179367604824169	4.95714955639856\\
-0.0620136943089292	4.91361081056674\\
0.0830829217019989	5.53158130068721\\
0.267087059091798	5.49179506451985\\
-0.78150157615695	4.29252956776867\\
0.742335271893561	4.39888424791663\\
0.329491271036343	5.21301056383554\\
0.992408019378515	3.74549029229844\\
-0.441152396757815	4.73838054114084\\
};

\end{axis}
\end{tikzpicture}%

%% file: f4.tex
%
%
\definecolor{mycolor1}{rgb}{0.00000,0.44700,0.74100}%
\definecolor{mycolor2}{rgb}{0.85000,0.32500,0.09800}%
\begin{tikzpicture}

\begin{axis}[%
width=2in, 
height=1.7in,
at={(1.011in,0.642in)},
scale only axis,
xmin=-2,
xmax=7,
ymin=-2,
ymax=7,
axis background/.style={fill=white},
xmajorgrids,
ymajorgrids,
tick label style={font=\scriptsize},
legend style={font=\scriptsize, legend cell align=left, align=left, draw=white!15!black}
]

\addplot [color=red, line width=1.5pt]
  table[row sep=crcr]{%
7	-2\\
-2	7\\
};
\addlegendentry{$Au=f$}

\addplot[-Straight Barb, color=mycolor1] 
table[row sep=crcr] {
x	y	u	v\\
-2	-2	0.424750526787253	0.424750526787253\\
};
\addlegendentry{Network flow field}

\addplot [color=blue, line width=1.5pt, mark size=3.0pt, mark=x, mark options={solid, blue}]
  table[row sep=crcr]{%
0.118062905967236	4.8792724609375\\
0.118074178695679	4.8792610168457\\
};
\addlegendentry{Path of network}

\addplot [color=mycolor2, only marks, line width=1.5pt, draw=none, mark size=3.0pt, mark=x, mark options={solid, mycolor2}]
table[row sep=crcr] {
x y \\
-0.441152396757815	4.73838054114084\\
};
\addlegendentry{Training examples}

\addplot[-Straight Barb, color=mycolor1, line width=0.5pt, point meta={sqrt((\thisrow{u})^2+(\thisrow{v})^2)}, point meta min=0, quiver={scale arrows=1,u=\thisrow{u}, v=\thisrow{v}, every arrow/.append style={-{Straight Barb[angle'=18.263, scale={(2/2500)*\pgfplotspointmetatransformed}]}}}]
 table[row sep=crcr] {%
x	y	u	v\\
-2	-2	-0.293966345923749	0.526378278591907\\
-2	-1.5	-0.293857832133464	0.526269823574211\\
-2	-1	-0.293749337934041	0.526161290193062\\
-2	-0.5	-0.293640824143755	0.526052795993639\\
-2	0	-0.293532329944332	0.52594426261249\\
-2	0.5	-0.293423816154047	0.525835768413068\\
-2	1	-0.293315341545487	0.525727274213645\\
-2	1.5	-0.293206827755201	0.525618780014222\\
-2	2	-0.293098333555778	0.525510246633073\\
-2	2.5	-0.292989819765493	0.525401752433651\\
-2	3	-0.29288132556607	0.525293258234228\\
-2	3.5	-0.292772831366647	0.525184724853079\\
-2	4	-0.292664337167224	0.525076230653656\\
-2	4.5	-0.292555823376939	0.524967658090781\\
-2	5	-0.292447368359242	0.524859203073085\\
-2	5.5	-0.292338913341545	0.524750591328484\\
-2	6	-0.292230575869027	0.524641979583883\\
-2	6.5	-0.292122924076714	0.524532584204761\\
-2	7	-0.406841140681944	0.409597379200685\\
-1.5	-2	-0.293272163283384	0.525684135133268\\
-1.5	-1.5	-0.293163669083961	0.525575601752119\\
-1.5	-1	-0.293055155293675	0.525467146734422\\
-1.5	-0.5	-0.292946680685116	0.525358613353273\\
-1.5	0	-0.292838147303967	0.525250119153851\\
-1.5	0.5	-0.292729653104544	0.525141585772702\\
-1.5	1	-0.292621158905121	0.525033091573279\\
-1.5	1.5	-0.292512664705698	0.524924597373856\\
-1.5	2	-0.292404170506276	0.524816103174434\\
-1.5	2.5	-0.29229565671599	0.524707569793285\\
-1.5	3	-0.292187142925704	0.524599075593862\\
-1.5	3.5	-0.292078687908008	0.524490542212713\\
-1.5	4	-0.291970193708585	0.524382048013291\\
-1.5	4.5	-0.291861719100025	0.524273514632142\\
-1.5	5	-0.291753283673191	0.524164942069267\\
-1.5	5.5	-0.291644907018947	0.52405629114294\\
-1.5	6	-0.291537274817497	0.523946934945544\\
-1.5	6.5	-0.406258430052181	0.409008752130289\\
-1.5	7	-0.523729868183247	0.291320325600378\\
-1	-2	-0.292578019824744	0.524989952492902\\
-1	-1.5	-0.292469525625321	0.524881497475205\\
-1	-1	-0.292361011835036	0.524773003275782\\
-1	-0.5	-0.29225249804475	0.524664469894634\\
-1	0	-0.29214402343619	0.524555936513485\\
-1	0.5	-0.292035529236767	0.524447481495788\\
-1	1	-0.291927015446482	0.524338948114639\\
-1	1.5	-0.291818501656196	0.524230414733491\\
-1	2	-0.291710027047636	0.524121920534068\\
-1	2.5	-0.291601532848213	0.524013426334645\\
-1	3	-0.291492999467064	0.523904892953496\\
-1	3.5	-0.291384524858505	0.523796359572347\\
-1	4	-0.291276050249945	0.523687865372925\\
-1	4.5	-0.291167595232248	0.52357929281005\\
-1	5	-0.29105925775973	0.523470641883723\\
-1	5.5	-0.29095162555828	0.523361285686327\\
-1	6	-0.405675758604144	0.408420125059892\\
-1	6.5	-0.52314421892403	0.290734637159435\\
-1	7	-0.523036625904306	0.290625261371176\\
-0.5	-2	-0.291883837184378	0.524295809034262\\
-0.5	-1.5	-0.291775382166682	0.524187314834839\\
-0.5	-1	-0.291666848785533	0.524078820635417\\
-0.5	-0.5	-0.291558334995247	0.523970248072542\\
-0.5	0	-0.291449860386687	0.523861793054845\\
-0.5	0.5	-0.291341346596402	0.523753259673696\\
-0.5	1	-0.291232832806116	0.523644765474274\\
-0.5	1.5	-0.291124377788419	0.523536271274851\\
-0.5	2	-0.29101584440727	0.523427737893702\\
-0.5	2.5	-0.290907389389574	0.523319243694279\\
-0.5	3	-0.290798875599288	0.523210749494857\\
-0.5	3.5	-0.290690400990728	0.523102216113708\\
-0.5	4	-0.290581985154757	0.522993643550833\\
-0.5	4.5	-0.290473608500513	0.522884992624506\\
-0.5	5	-0.2903659567082	0.522775597245384\\
-0.5	5.5	-0.40509304797438	0.407831497989496\\
-0.5	6	-0.522558569664813	0.290149027081944\\
-0.5	6.5	-0.522450937463363	0.290039592521096\\
-0.5	7	-0.522342599990844	0.289930941594769\\
0	-2	-0.291189654544013	0.523601626393897\\
0	-1.5	-0.291081179935453	0.5234931713762\\
0	-1	-0.290972705326893	0.523384637995051\\
0	-0.5	-0.290864171945744	0.523276104613902\\
0	0	-0.290755677746322	0.523167610414479\\
0	0.5	-0.290647183546899	0.523059116215057\\
0	1	-0.290538689347476	0.522950622015634\\
0	1.5	-0.290430195148053	0.522842088634485\\
0	2	-0.290321700948631	0.522733594435062\\
0	2.5	-0.290213206749208	0.522625061053914\\
0	3	-0.290104751731511	0.522516566854491\\
0	3.5	-0.289996277122951	0.522407994291616\\
0	4	-0.289887939650433	0.522299304183563\\
0	4.5	-0.28978028785812	0.522189987167893\\
0	5	-0.404510376526343	0.407242910100825\\
0	5.5	-0.521972920405596	0.289563377822727\\
0	6	-0.521865327385872	0.289453943261879\\
0	6.5	-0.521756950731628	0.289345292335552\\
0	7	-0.521648495713931	0.289236700181814\\
0.5	-2	-0.290495550267099	0.522907482935257\\
0.5	-1.5	-0.29038701688595	0.522798949554108\\
0.5	-1	-0.29027854227739	0.522690494536411\\
0.5	-0.5	-0.290170008896242	0.522581961155262\\
0.5	0	-0.290061534287682	0.522473506137566\\
0.5	0.5	-0.289953040088259	0.522364972756417\\
0.5	1	-0.289844526297973	0.522256439375268\\
0.5	1.5	-0.289736051689414	0.522147945175846\\
0.5	2	-0.289627557489991	0.522039411794697\\
0.5	2.5	-0.289519082881431	0.521930917595274\\
0.5	3	-0.289410647454597	0.521822345032399\\
0.5	3.5	-0.289302270800353	0.521713654924346\\
0.5	4	-0.289194638598903	0.52160429872695\\
0.5	4.5	-0.403927665896579	0.406654283030429\\
0.5	5	-0.521387231964653	0.288977689381784\\
0.5	5.5	-0.521279638944929	0.288868294002662\\
0.5	6	-0.521171301472411	0.288759603894609\\
0.5	6.5	-0.52106288563644	0.28865107051346\\
0.5	7	-0.520954430618743	0.288542537132311\\
1	-2	-0.289801387217596	0.522213339476617\\
1	-1.5	-0.28969287342731	0.522104806095468\\
1	-1	-0.289584379227888	0.521996311896046\\
1	-0.5	-0.289475885028465	0.521887817696623\\
1	0	-0.289367390829042	0.521779284315474\\
1	0.5	-0.289258896629619	0.521670790116051\\
1	1	-0.289150363248471	0.521562256734902\\
1	1.5	-0.289041888639911	0.52145376253548\\
1	2	-0.288933414031351	0.521345229154331\\
1	2.5	-0.288824978604518	0.521236695773182\\
1	3	-0.288716621541136	0.521128005665129\\
1	3.5	-0.28860895015796	0.521018649467733\\
1	4	-0.403344994448542	0.406065655960032\\
1	4.5	-0.520801621887162	0.28839205971343\\
1	5	-0.520693989685712	0.288282644743445\\
1	5.5	-0.520585652213194	0.288173954635392\\
1	6	-0.520477197195497	0.28806540166338\\
1	6.5	-0.520368702996074	0.287956887873094\\
1	7	-0.520260247978377	0.287848354491946\\
1.5	-2	-0.289107224168093	0.521519156836251\\
1.5	-1.5	-0.288998690786945	0.521410623455102\\
1.5	-1	-0.288890235769248	0.521302168437406\\
1.5	-0.5	-0.288781721978962	0.521193595874531\\
1.5	0	-0.288673208188676	0.521085140856834\\
1.5	0.5	-0.28856475317098	0.520976607475686\\
1.5	1	-0.28845627856242	0.520868152457989\\
1.5	1.5	-0.288347764772134	0.520759579895114\\
1.5	2	-0.288239329345301	0.520651007332239\\
1.5	2.5	-0.288130952691056	0.520542317224186\\
1.5	3	-0.288023320489606	0.520433000208516\\
1.5	3.5	-0.402762283818778	0.405477068071362\\
1.5	4	-0.520215933446219	0.28780639086335\\
1.5	4.5	-0.520108340426495	0.287696995484228\\
1.5	5	-0.520000002953977	0.287588305376175\\
1.5	5.5	-0.51989154793628	0.2874797328133\\
1.5	6	-0.519783053736857	0.287371219023014\\
1.5	6.5	-0.51967459871916	0.287262685641866\\
1.5	7	-0.519566143701464	0.287154191442443\\
2	-2	-0.288413080709454	0.520825013377612\\
2	-1.5	-0.288304586510031	0.520716519178189\\
2	-1	-0.288196053128882	0.52060798579704\\
2	-0.5	-0.288087578520322	0.520499452415891\\
2	0	-0.287979064730037	0.520390997398195\\
2	0.5	-0.287870590121477	0.520282464017046\\
2	1	-0.287762115512917	0.520173891454171\\
2	1.5	-0.287653660495221	0.520065358073022\\
2	2	-0.287545303431839	0.519956707146695\\
2	2.5	-0.287437651639526	0.519847350949299\\
2	3	-0.402179573189015	0.404888401819239\\
2	3.5	-0.519630245005276	0.28722072201327\\
2	4	-0.519522691167278	0.287111307043285\\
2	4.5	-0.51941435369476	0.287002616935232\\
2	5	-0.519305898677063	0.286894122735809\\
2	5.5	-0.51919740447764	0.28678558935466\\
2	6	-0.519088910278217	0.286677016791786\\
2	6.5	-0.518980455260521	0.286568522592363\\
2	7	-0.518877407321019	0.286465494243724\\
2.5	-2	-0.287718917659951	0.520130830737246\\
2.5	-1.5	-0.287610403869665	0.520022336537823\\
2.5	-1	-0.287501890079379	0.519913763974948\\
2.5	-0.5	-0.28739341547082	0.519805308957252\\
2.5	0	-0.28728494086226	0.519696775576103\\
2.5	0.5	-0.287176427071974	0.51958828137668\\
2.5	1	-0.287068011236004	0.519479708813805\\
2.5	1.5	-0.286959634581759	0.519371018705752\\
2.5	2	-0.286852002380309	0.519261701690082\\
2.5	2.5	-0.401596901740977	0.404299813930569\\
2.5	3	-0.519044634927785	0.28663505316319\\
2.5	3.5	-0.518936963544609	0.286525638193205\\
2.5	4	-0.518828665253817	0.286416967676015\\
2.5	4.5	-0.51872021023612	0.286308414704003\\
2.5	5	-0.518611755218423	0.286199900913717\\
2.5	5.5	-0.518503300200727	0.286091387123432\\
2.5	6	-0.518394766819578	0.285982873333146\\
2.5	6.5	-0.518291718880075	0.285879805802781\\
2.5	7	-0.518188788485751	0.285776855817593\\
3	-2	-0.287024754610448	0.51943664809688\\
3	-1.5	-0.286916260411026	0.519328153897457\\
3	-1	-0.286807766211603	0.519219620516309\\
3	-0.5	-0.28669927201218	0.519111126316886\\
3	0	-0.286590777812757	0.519002632117463\\
3	0.5	-0.286482342385924	0.518894020372862\\
3	1	-0.286374004913405	0.518785408628261\\
3	1.5	-0.286266353121092	0.518676013249139\\
3	2	-0.40101423029294	0.403711147678446\\
3	2.5	-0.518458985668568	0.286049403903973\\
3	3	-0.518351353467118	0.285940008524851\\
3	3.5	-0.5182430159946	0.285831318416798\\
3	4	-0.518134600158629	0.285722765444786\\
3	4.5	-0.518026105959206	0.2856142516545\\
3	5	-0.517917611759784	0.285505737864215\\
3	5.5	-0.517809117560361	0.285397224073929\\
3	6	-0.517706069620859	0.285294176134427\\
3	6.5	-0.517603100044808	0.28519116737665\\
3	7	-0.517500169650484	0.285088217391463\\
3.5	-2	-0.286330611151809	0.51874250463824\\
3.5	-1.5	-0.286222116952386	0.518633971257092\\
3.5	-1	-0.286113622752963	0.518525477057669\\
3.5	-0.5	-0.28600512855354	0.51841694367652\\
3.5	0	-0.285896673535844	0.518308371113645\\
3.5	0.5	-0.285788336063325	0.518199720187318\\
3.5	1	-0.285680703861875	0.518090363989922\\
3.5	1.5	-0.400431519663177	0.40312252060805\\
3.5	2	-0.517873336409351	0.285463735053893\\
3.5	2.5	-0.517765704207901	0.285354359265634\\
3.5	3	-0.517657366735383	0.285245708339307\\
3.5	3.5	-0.517548911717686	0.285137096594706\\
3.5	4	-0.517440456699989	0.285028582804421\\
3.5	4.5	-0.51733192331884	0.284920049423272\\
3.5	5	-0.517223468301144	0.284811574814712\\
3.5	5.5	-0.517120420361642	0.284708487693484\\
3.5	6	-0.517017450785591	0.284605537708296\\
3.5	6.5	-0.516914481209541	0.284502568132246\\
3.5	7	-0.516811511633491	0.284399598556196\\
4	-2	-0.285636428511443	0.518048321997875\\
4	-1.5	-0.285527953902883	0.517939788616726\\
4	-1	-0.285419479294323	0.517831294417303\\
4	-0.5	-0.285311024276627	0.517722721854428\\
4	0	-0.285202686804108	0.517614070928101\\
4	0.5	-0.285095015420932	0.517504714730706\\
4	1	-0.399848848215139	0.402533893537653\\
4	1.5	-0.517287647968408	0.284878105385539\\
4	2	-0.517180054948684	0.284768710006417\\
4	2.5	-0.517071717476166	0.284659980716638\\
4	3	-0.516963262458469	0.284551447335489\\
4	3.5	-0.516854807440772	0.284442913954341\\
4	4	-0.51674631324135	0.284334419754918\\
4	4.5	-0.516637819041927	0.284225886373769\\
4	5	-0.516534731920699	0.284122838434267\\
4	5.5	-0.5164318407081	0.28401988844908\\
4	6	-0.516328831950324	0.283916899282166\\
4	6.5	-0.516225901556	0.283813949296979\\
4	7	-0.516122892798224	0.283710960130066\\
4.5	-2	-0.284942304643666	0.517354178539235\\
4.5	-1.5	-0.28483379085338	0.517245645158086\\
4.5	-1	-0.284725394608273	0.517137111776937\\
4.5	-0.5	-0.284617037544891	0.517028382487158\\
4.5	0	-0.284509346570852	0.516919026289762\\
4.5	0.5	-0.399266176767102	0.401945305648983\\
4.5	1	-0.516701998709191	0.284292416944596\\
4.5	1.5	-0.516594366507741	0.284183041156337\\
4.5	2	-0.516486068216949	0.284074351048284\\
4.5	2.5	-0.516377613199252	0.283965837257998\\
4.5	3	-0.516269118999829	0.28385730387685\\
4.5	3.5	-0.516160624800407	0.283748750904838\\
4.5	4	-0.51605216978271	0.283640237114552\\
4.5	4.5	-0.515949082661482	0.283537149993324\\
4.5	5	-0.515846152267157	0.283434239189863\\
4.5	5.5	-0.515743182691107	0.283331250022949\\
4.5	6	-0.515640213115057	0.283228280446899\\
4.5	6.5	-0.515537243539007	0.283125350052575\\
4.5	7	-0.515434313144683	0.283022380476525\\
5	-2	-0.284248141594163	0.516659995898869\\
5	-1.5	-0.28413970616733	0.516551423335994\\
5	-1	-0.284031388285674	0.516442733227941\\
5	-0.5	-0.283923716902498	0.516333416212272\\
5	0	-0.398683466137338	0.401356678578586\\
5	0.5	-0.516116310268248	0.283706767685379\\
5	1	-0.516008717248524	0.283597372306257\\
5	1.5	-0.515900379776006	0.283488682198204\\
5	2	-0.515791924758309	0.283380129226192\\
5	2.5	-0.515683469740612	0.283271615435907\\
5	3	-0.51557497554119	0.283163082054758\\
5	3.5	-0.515466520523493	0.283054587855335\\
5	4	-0.515363433402265	0.282951500734107\\
5	4.5	-0.51526050300794	0.282848570339783\\
5	5	-0.515157494250164	0.282745581172869\\
5	5.5	-0.51505456385584	0.282642631187682\\
5	6	-0.51495159427979	0.282539661611632\\
5	6.5	-0.51484862470374	0.282436692035582\\
5	7	-0.514745694309416	0.282333761641258\\
5.5	-2	-0.283554076498976	0.515965774076777\\
5.5	-1.5	-0.283445699844731	0.515857083968724\\
5.5	-1	-0.283338067643281	0.515747727771329\\
5.5	-0.5	-0.398100755507575	0.40076805150819\\
5.5	0	-0.515530661009031	0.283121098835299\\
5.5	0.5	-0.515423067989307	0.28301172304704\\
5.5	1	-0.515314730516789	0.282903032938987\\
5.5	1.5	-0.515206314680818	0.282794479966975\\
5.5	2	-0.515097781299669	0.282685946585827\\
5.5	2.5	-0.514989326281973	0.282577432795541\\
5.5	3	-0.51488083208255	0.282468938596118\\
5.5	3.5	-0.514777823324774	0.282365890656616\\
5.5	4	-0.514674853748723	0.282262901489703\\
5.5	4.5	-0.514571884172673	0.282159931913652\\
5.5	5	-0.514468914596623	0.282056962337602\\
5.5	5.5	-0.514365945020573	0.281954012352415\\
5.5	6	-0.514262975444523	0.281851023185502\\
5.5	6.5	-0.514181085637086	0.281769094196339\\
5.5	7	-0.514208042664606	0.281796109996448\\
6	-2	-0.282860050585514	0.515271434709507\\
6	-1.5	-0.282752398793201	0.515162117693838\\
6	-1	-0.397518084059537	0.400179424437793\\
6	-0.5	-0.514945011749814	0.282535469166945\\
6	0	-0.51483741873009	0.282426073787823\\
6	0.5	-0.514729081257572	0.28231738367977\\
6	1	-0.514620626239875	0.282208811116895\\
6	1.5	-0.514512171222178	0.282100277735747\\
6	2	-0.514403677022756	0.281991783536324\\
6	2.5	-0.514295182823333	0.281883289336901\\
6	3	-0.514192134883831	0.281780202215673\\
6	3.5	-0.51408916530778	0.281677252230486\\
6	4	-0.51398619573173	0.281574263063572\\
6	4.5	-0.51388322615568	0.281471313078385\\
6	5	-0.51378025657963	0.281368323911472\\
6	5.5	-0.513677326185306	0.281265393517148\\
6	6	-0.513639163184136	0.281227230515979\\
6	6.5	-0.513666198575109	0.281254265906951\\
6	7	-0.513800944530982	0.281389051044551\\
6.5	-2	-0.282166729943121	0.514576429252895\\
6.5	-1.5	-0.396935334248048	0.399590797367397\\
6.5	-1	-0.514359362490597	0.281949780726002\\
6.5	-0.5	-0.514251769470873	0.281840404937743\\
6.5	0	-0.514143392816629	0.28173171482969\\
6.5	0.5	-0.514034976980658	0.281623201039404\\
6.5	1	-0.513926482781235	0.281514648067393\\
6.5	1.5	-0.513818027763539	0.281406134277107\\
6.5	2	-0.513709533564116	0.281297600895958\\
6.5	2.5	-0.513606446442888	0.281194552956456\\
6.5	3	-0.513503516048563	0.281091583380406\\
6.5	3.5	-0.513400546472513	0.280988613804355\\
6.5	4	-0.513297616078189	0.280885663819168\\
6.5	4.5	-0.513194646502139	0.280782713833981\\
6.5	5	-0.513091676926089	0.280679724667068\\
6.5	5.5	-0.513097279912913	0.280685347244756\\
6.5	6	-0.513137597909015	0.280725645649994\\
6.5	6.5	-0.51327669303648	0.280864740777459\\
6.5	7	-0.513415788163945	0.281003855495788\\
7	-2	-0.39635266280001	0.399002131115274\\
7	-1.5	-0.513773674049654	0.281364131466785\\
7	-1	-0.51366608102993	0.281254736087663\\
7	-0.5	-0.513557782739138	0.281146065570473\\
7	0	-0.513449327721441	0.281037512598461\\
7	0.5	-0.513340872703745	0.280928979217313\\
7	1	-0.513232339322596	0.280820465427027\\
7	1.5	-0.513123884304899	0.280711951636741\\
7	2	-0.513020836365397	0.280608903697239\\
7	2.5	-0.512917866789347	0.280505953712052\\
7	3	-0.512814897213296	0.280402964545139\\
7	3.5	-0.512711927637246	0.280299994969088\\
7	4	-0.512608958061196	0.280197025393038\\
7	4.5	-0.512528400432444	0.280116467764287\\
7	5	-0.512555435823416	0.280143522746122\\
7	5.5	-0.512613346414513	0.280201413746355\\
7	6	-0.512752441541978	0.280340489282957\\
7	6.5	-0.512891536669443	0.280479564819559\\
7	7	-0.513030631796908	0.280618659947024\\
};

\addplot [color=mycolor2, line width=1.5pt, draw=none, mark size=3.0pt, mark=x, mark options={solid, mycolor2}]
  table[row sep=crcr]{%
0.0564122048664746	4.7717430512871\\
0.526380372501768	4.7896522953997\\
-0.884209330862958	5.4367437241026\\
0.089995154469105	5.14141897969074\\
0.610385882491538	5.4908836008354\\
-0.693364070255412	5.30021720482627\\
-0.693776926268841	4.52996803371776\\
0.0651002292560559	5.69779026083464\\
0.330846431068311	4.47001001843488\\
0.44109520484135	5.00479066343026\\
-0.814509535557422	5.28625612011379\\
0.883248203822483	4.94332041567406\\
0.917186792741197	4.89788472159849\\
-0.758406129956736	4.641460056008\\
-0.387842285894774	6.0322800312358\\
-0.348794152393402	5.41503160351195\\
-0.121038124937394	5.46913408235362\\
-0.0844850967268441	4.00092599377353\\
-0.429381064585652	4.96229005696987\\
0.19475090437651	4.61853477341041\\
0.843954502042891	5.01149276027048\\
-0.142591857804827	6.03534692952404\\
-0.258785247909006	5.20486447765904\\
0.743591140369287	4.55209537197808\\
0.20561379726423	4.31929928019662\\
-0.367037787465725	4.84549875602467\\
0.506504042789148	6.14736859963147\\
-0.394133520307896	5.15435369724667\\
0.0175270807338539	3.97471424765517\\
-0.376939637719114	4.26011714002417\\
0.597866891228481	4.37830483248397\\
-0.179367604824169	4.95714955639856\\
-0.0620136943089292	4.91361081056674\\
0.0830829217019989	5.53158130068721\\
0.267087059091798	5.49179506451985\\
-0.78150157615695	4.29252956776867\\
0.742335271893561	4.39888424791663\\
0.329491271036343	5.21301056383554\\
0.992408019378515	3.74549029229844\\
-0.441152396757815	4.73838054114084\\
};

\addplot [color=blue, line width=1.5pt, draw=none, mark size=3.0pt, mark=x, mark options={solid, blue}]
  table[row sep=crcr]{%
-0.708884239196777	1.27456963062286\\
};

\addplot [color=blue, line width=1.5pt, mark size=3.0pt, mark=x, mark options={solid, blue}, forget plot]
  table[row sep=crcr]{%
-0.708884239196777	1.27456963062286\\
0	0\\
};
\addplot [color=blue, line width=1.5pt, draw=none, mark size=3.0pt, mark=x, mark options={solid, blue}, forget plot]
  table[row sep=crcr]{%
-1.41949069499969	2.55086135864258\\
};
\addplot [color=blue, line width=1.5pt, mark size=3.0pt, mark=x, mark options={solid, blue}, forget plot]
  table[row sep=crcr]{%
-1.41949069499969	2.55086135864258\\
-0.708884239196777	1.27456963062286\\
};
\addplot [color=blue, line width=1.5pt, draw=none, mark size=3.0pt, mark=x, mark options={solid, blue}, forget plot]
  table[row sep=crcr]{%
-2.1318244934082	3.82888031005859\\
};
\addplot [color=blue, line width=1.5pt, mark size=3.0pt, mark=x, mark options={solid, blue}, forget plot]
  table[row sep=crcr]{%
-2.1318244934082	3.82888031005859\\
-1.41949069499969	2.55086135864258\\
};
\addplot [color=blue, line width=1.5pt, draw=none, mark size=3.0pt, mark=x, mark options={solid, blue}, forget plot]
  table[row sep=crcr]{%
-2.84589004516602	5.1086311340332\\
};
\addplot [color=blue, line width=1.5pt, mark size=3.0pt, mark=x, mark options={solid, blue}, forget plot]
  table[row sep=crcr]{%
-2.84589004516602	5.1086311340332\\
-2.1318244934082	3.82888031005859\\
};
\addplot [color=blue, line width=1.5pt, draw=none, mark size=3.0pt, mark=x, mark options={solid, blue}, forget plot]
  table[row sep=crcr]{%
-3.56169271469116	6.39011907577515\\
};
\addplot [color=blue, line width=1.5pt, mark size=3.0pt, mark=x, mark options={solid, blue}, forget plot]
  table[row sep=crcr]{%
-3.56169271469116	6.39011907577515\\
-2.84589004516602	5.1086311340332\\
};
\addplot [color=blue, line width=1.5pt, draw=none, mark size=3.0pt, mark=x, mark options={solid, blue}, forget plot]
  table[row sep=crcr]{%
-4.27923727035522	7.67334890365601\\
};
\addplot [color=blue, line width=1.5pt, mark size=3.0pt, mark=x, mark options={solid, blue}, forget plot]
  table[row sep=crcr]{%
-4.27923727035522	7.67334890365601\\
-3.56169271469116	6.39011907577515\\
};
\addplot [color=blue, line width=1.5pt, draw=none, mark size=3.0pt, mark=x, mark options={solid, blue}, forget plot]
  table[row sep=crcr]{%
-4.99852895736694	8.95832538604736\\
};
\addplot [color=blue, line width=1.5pt, mark size=3.0pt, mark=x, mark options={solid, blue}, forget plot]
  table[row sep=crcr]{%
-4.99852895736694	8.95832538604736\\
-4.27923727035522	7.67334890365601\\
};
\addplot [color=blue, line width=1.5pt, draw=none, mark size=3.0pt, mark=x, mark options={solid, blue}, forget plot]
  table[row sep=crcr]{%
-5.71957302093506	10.2450532913208\\
};
\addplot [color=blue, line width=1.5pt, mark size=3.0pt, mark=x, mark options={solid, blue}, forget plot]
  table[row sep=crcr]{%
-5.71957302093506	10.2450532913208\\
-4.99852895736694	8.95832538604736\\
};
\addplot [color=blue, line width=1.5pt, draw=none, mark size=3.0pt, mark=x, mark options={solid, blue}, forget plot]
  table[row sep=crcr]{%
-6.44237613677979	11.5335359573364\\
};
\addplot [color=blue, line width=1.5pt, mark size=3.0pt, mark=x, mark options={solid, blue}, forget plot]
  table[row sep=crcr]{%
-6.44237613677979	11.5335359573364\\
-5.71957302093506	10.2450532913208\\
};
\addplot [color=blue, line width=1.5pt, draw=none, mark size=3.0pt, mark=x, mark options={solid, blue}, forget plot]
  table[row sep=crcr]{%
-6.76490545272827	11.7147092819214\\
};
\addplot [color=blue, line width=1.5pt, mark size=3.0pt, mark=x, mark options={solid, blue}, forget plot]
  table[row sep=crcr]{%
-6.76490545272827	11.7147092819214\\
-6.44237613677979	11.5335359573364\\
};
\addplot [color=blue, line width=1.5pt, draw=none, mark size=3.0pt, mark=x, mark options={solid, blue}, forget plot]
  table[row sep=crcr]{%
-6.85561943054199	11.8760948181152\\
};
\addplot [color=blue, line width=1.5pt, mark size=3.0pt, mark=x, mark options={solid, blue}, forget plot]
  table[row sep=crcr]{%
-6.85561943054199	11.8760948181152\\
-6.76490545272827	11.7147092819214\\
};
\addplot [color=blue, line width=1.5pt, draw=none, mark size=3.0pt, mark=x, mark options={solid, blue}, forget plot]
  table[row sep=crcr]{%
-6.87523889541626	11.8879947662354\\
};
\addplot [color=blue, line width=1.5pt, mark size=3.0pt, mark=x, mark options={solid, blue}, forget plot]
  table[row sep=crcr]{%
-6.87523889541626	11.8879947662354\\
-6.85561943054199	11.8760948181152\\
};
\addplot [color=blue, line width=1.5pt, draw=none, mark size=3.0pt, mark=x, mark options={solid, blue}, forget plot]
  table[row sep=crcr]{%
-6.88122701644897	11.8977670669556\\
};
\addplot [color=blue, line width=1.5pt, mark size=3.0pt, mark=x, mark options={solid, blue}, forget plot]
  table[row sep=crcr]{%
-6.88122701644897	11.8977670669556\\
-6.87523889541626	11.8879947662354\\
};
\addplot [color=blue, line width=1.5pt, draw=none, mark size=3.0pt, mark=x, mark options={solid, blue}, forget plot]
  table[row sep=crcr]{%
-6.88920879364014	11.9055461883545\\
};
\addplot [color=blue, line width=1.5pt, mark size=3.0pt, mark=x, mark options={solid, blue}, forget plot]
  table[row sep=crcr]{%
-6.88920879364014	11.9055461883545\\
-6.88122701644897	11.8977670669556\\
};
\addplot [color=blue, line width=1.5pt, draw=none, mark size=3.0pt, mark=x, mark options={solid, blue}, forget plot]
  table[row sep=crcr]{%
-6.89309978485107	11.9095363616943\\
};
\addplot [color=blue, line width=1.5pt, mark size=3.0pt, mark=x, mark options={solid, blue}, forget plot]
  table[row sep=crcr]{%
-6.89309978485107	11.9095363616943\\
-6.88920879364014	11.9055461883545\\
};
\end{axis}
\end{tikzpicture}%

%% file: f2_1.tex
%
%
\definecolor{mycolor1}{rgb}{0.00000,0.44700,0.74100}%
\definecolor{mycolor2}{rgb}{0.85000,0.32500,0.09800}%
\definecolor{mycolor3}{rgb}{0.92900,0.69400,0.12500}%
\definecolor{mycolor4}{rgb}{0.49400,0.18400,0.55600}%
\begin{tikzpicture}

\begin{axis}[%
width=1.45in,
height=1.1in,
at={(1.011in,0.642in)},
scale only axis,
xmin=0,
xmax=200,
xlabel style={font=\color{white!15!black}},
xlabel={Iterations},
ymin=20,
ymax=33,
yminorticks=true,
xlabel style={font=\scriptsize}, 
ylabel style={font=\scriptsize}, 
ylabel={PSNR},
tick label style={font=\scriptsize},
axis background/.style={fill=white},
title={noise $\sigma= 0.3$},
xmajorgrids,
ymajorgrids,
yminorgrids,
xlabel shift=-5pt,
ylabel shift=-2pt,
title style={yshift=-4pt},
]
\addplot [color=mycolor1, line width=1.5pt]
  table[row sep=crcr]{%
1	7.83756461647623\\
2	9.96539248955178\\
3	11.8827533666104\\
4	14.3116060169628\\
5	17.6581700095118\\
6	22.9321692987809\\
7	30.7078068228505\\
8	30.7707626401837\\
9	30.8138211074267\\
10	30.8398611156398\\
11	30.8593840711013\\
12	30.8608058358232\\
13	30.8667641238949\\
14	30.8581221416109\\
15	30.8566471271102\\
16	30.8426468485343\\
17	30.8372312109363\\
18	30.8206466529563\\
19	30.8134394742879\\
20	30.795889347353\\
21	30.7880886678701\\
22	30.7704345063282\\
23	30.7627156614785\\
24	30.7454382319328\\
25	30.7381709926966\\
26	30.7348501322268\\
27	30.0682447315513\\
28	30.1337033215392\\
29	30.1639013152211\\
30	30.1668084318476\\
31	30.1699739154778\\
32	30.1517973482657\\
33	30.1463195829187\\
34	30.1382974487338\\
35	30.1314439102229\\
36	30.1226354901593\\
37	30.1151645426901\\
38	30.1060238054222\\
39	30.0983544922353\\
40	30.0891496949712\\
41	30.0815171328695\\
42	30.0724073246492\\
43	30.0649490696309\\
44	30.0560309827836\\
45	30.0488285962937\\
46	30.0401593647194\\
47	30.033259081876\\
48	30.0248691585216\\
49	30.0182940804734\\
50	30.0101960321934\\
51	30.0039539084706\\
52	29.9961483350664\\
53	29.9902372623694\\
54	29.9827174508302\\
55	29.9771299668823\\
56	29.9698851430744\\
57	29.9646109271569\\
58	29.9576284583002\\
59	29.9526560579627\\
60	29.9459227704891\\
61	29.9412405288722\\
62	29.9347433335257\\
63	29.9303397534796\\
64	29.9240658272966\\
65	29.9199296326853\\
66	29.9171713190475\\
67	29.6843570405634\\
68	29.7001084821117\\
69	29.7050053810009\\
70	29.694171549322\\
71	29.6890593332615\\
72	29.6871015322823\\
73	29.6823520531614\\
74	29.6790698316303\\
75	29.6744354390169\\
76	29.6704214314004\\
77	29.6659068357848\\
78	29.6615032036015\\
79	29.6571152106769\\
80	29.6525252801856\\
81	29.6482678114469\\
82	29.6436129163928\\
83	29.6394886508407\\
84	29.6348412629635\\
85	29.6308520833638\\
86	29.6262554551436\\
87	29.6224024184183\\
88	29.6178821814143\\
89	29.6141654847276\\
90	29.6097365097685\\
91	29.606155528914\\
92	29.6018259864358\\
93	29.5983793799734\\
94	29.5941532113954\\
95	29.5908390363712\\
96	29.5867174683334\\
97	29.5835333111755\\
98	29.579515802293\\
99	29.5764588917071\\
100	29.5725437609926\\
101	29.5696110719032\\
102	29.5657958864937\\
103	29.562984231612\\
104	29.5592660804141\\
105	29.556572159529\\
106	29.5529478389582\\
107	29.550368309571\\
108	29.5468344382159\\
109	29.5443659500153\\
110	29.5409190603458\\
111	29.5385582829064\\
112	29.5351948787052\\
113	29.5329385311002\\
114	29.5296551286054\\
115	29.5274999926827\\
116	29.5259604109895\\
117	29.270587211849\\
118	29.3011123209935\\
119	29.3140640732069\\
120	29.3182094276496\\
121	29.3202886266933\\
122	29.3215402208457\\
123	29.3220172189013\\
124	29.3219410491932\\
125	29.3215356471011\\
126	29.3207056260882\\
127	29.3198029014615\\
128	29.3185164509646\\
129	29.3173193953851\\
130	29.3157468320957\\
131	29.3143702426696\\
132	29.3126152759444\\
133	29.3111283463259\\
134	29.3092569626881\\
135	29.3077037792588\\
136	29.3057594467236\\
137	29.3041694360835\\
138	29.302181617935\\
139	29.3005751555743\\
140	29.2985642941079\\
141	29.2969558986195\\
142	29.2949364623049\\
143	29.2933366940875\\
144	29.2913190797993\\
145	29.2897357396688\\
146	29.2877275040904\\
147	29.2861664057079\\
148	29.2841730842877\\
149	29.2826385850334\\
150	29.2806642429599\\
151	29.2791596075018\\
152	29.277207220871\\
153	29.2757348862951\\
154	29.2738066116834\\
155	29.2723683767629\\
156	29.2704657516646\\
157	29.2690629163905\\
158	29.2671870000743\\
159	29.2658204738806\\
160	29.2639719556029\\
161	29.2626423395401\\
162	29.2608216154672\\
163	29.2595292650512\\
164	29.2577365051473\\
165	29.2564815813183\\
166	29.25471677047\\
167	29.2534992765289\\
168	29.2517622577633\\
169	29.2505820748221\\
170	29.248872574753\\
171	29.2477294793906\\
172	29.2460471347962\\
173	29.2449408225123\\
174	29.2432851973022\\
175	29.2422152996199\\
176	29.2405859011\\
177	29.2395519944894\\
178	29.2387653207039\\
179	29.1068936236558\\
180	29.122265011863\\
181	29.1286856441716\\
182	29.1302433486059\\
183	29.1308855959312\\
184	29.1311820493681\\
185	29.1311304975753\\
186	29.1308065145057\\
187	29.1303445663482\\
188	29.129662019572\\
189	29.1289615326804\\
190	29.1280617099461\\
191	29.12721717812\\
192	29.126182498012\\
193	29.1252492588865\\
194	29.1241302254565\\
195	29.1231426593829\\
196	29.1219714213569\\
197	29.1209519570196\\
198	29.1197496504073\\
199	29.118713493178\\
200	29.1174944082438\\
};

\addplot [color=mycolor2, line width=1.5pt]
  table[row sep=crcr]{%
1	7.83756461647623\\
2	14.2713233317354\\
3	15.6771399346392\\
4	16.5387585726121\\
5	17.0926468482173\\
6	17.8251710897736\\
7	18.5055695382513\\
8	18.8383351218034\\
9	20.0219425038336\\
10	20.2923182426406\\
11	20.7790161928829\\
12	21.1983777760132\\
13	21.4300072585427\\
14	22.2082415193261\\
15	22.4047399613063\\
16	23.0527488147795\\
17	23.2210452903205\\
18	23.7697520310511\\
19	23.9148412709972\\
20	24.3853223339328\\
21	24.510951433342\\
22	24.9180483887386\\
23	25.0271481978273\\
24	25.3817339806457\\
25	25.4766790088751\\
26	25.787032814033\\
27	25.8698011955688\\
28	26.1424594047904\\
29	26.2147288875684\\
30	26.4549973888356\\
31	26.5182005083686\\
32	26.7304564671392\\
33	26.7858195091666\\
34	26.9737205737709\\
35	27.0222962327184\\
36	27.1889358262169\\
37	27.2316283723118\\
38	27.3796447616615\\
39	27.4172322051885\\
40	27.5488945571264\\
41	27.5820463150735\\
42	27.6993147045305\\
43	27.7286044363876\\
44	27.8331685322539\\
45	27.8590852947971\\
46	27.9523977430153\\
47	27.9753588071897\\
48	28.0586738654026\\
49	28.0790373522081\\
50	28.153449145311\\
51	28.1715257673494\\
52	28.3021421369654\\
53	28.3180034429528\\
54	28.3448886828653\\
55	28.3714448679261\\
56	28.3838159791301\\
57	28.432569154917\\
58	28.443637144405\\
59	28.4872189723314\\
60	28.5065925710245\\
61	28.5167919778487\\
62	28.5535011931387\\
63	28.5625621571164\\
64	28.5954635688866\\
65	28.6035213228634\\
66	28.6330052033867\\
67	28.6401759020527\\
68	28.6665879087255\\
69	28.672971839044\\
70	28.7193348411641\\
71	28.7248960627895\\
72	28.7345250366904\\
73	28.7439626135262\\
74	28.7483995942585\\
75	28.7657342625903\\
76	28.7734639491712\\
77	28.777506868311\\
78	28.7920680767042\\
79	28.7956465045064\\
80	28.8086359327425\\
81	28.8118026865959\\
82	28.8233735550048\\
83	28.8261747104491\\
84	28.8364659028866\\
85	28.8412900715026\\
86	28.843565566645\\
87	28.8523510050016\\
88	28.8543621209054\\
89	28.8621381815265\\
90	28.8639129322505\\
91	28.8707807427189\\
92	28.8723440896244\\
93	28.8783955241086\\
94	28.8810612546102\\
95	28.882467806874\\
96	28.8874329276487\\
97	28.888657610769\\
98	28.8930093343368\\
99	28.8940733465441\\
100	28.8978735492105\\
101	28.8987953496678\\
102	28.9051645813449\\
103	28.9059539632157\\
104	28.9072239799393\\
105	28.9084616337777\\
106	28.909032778189\\
107	28.9132021137684\\
108	28.9136473257474\\
109	28.9144726660024\\
110	28.9151971181984\\
111	28.9155803816828\\
112	28.9168716837335\\
113	28.9174612204528\\
114	28.9177263240054\\
115	28.918704798988\\
116	28.9189128989405\\
117	28.9196783942482\\
118	28.919835751849\\
119	28.9204093989028\\
120	28.9206027568498\\
121	28.9207362347726\\
122	28.9210417332714\\
123	28.9211298652042\\
124	28.9212926583642\\
125	28.9213409930217\\
126	28.9213756185956\\
127	28.9213890254454\\
128	28.9211387640716\\
129	28.921133470899\\
130	28.9210077318522\\
131	28.9208838780956\\
132	28.9208075191492\\
133	28.9200783998408\\
134	28.9199766664224\\
135	28.919760993682\\
136	28.9195135520465\\
137	28.9194047120719\\
138	28.918293709899\\
139	28.9181825932371\\
140	28.9178811945886\\
141	28.917586710782\\
142	28.9174302336542\\
143	28.9161130560918\\
144	28.9159466156558\\
145	28.9156114267727\\
146	28.9152561623134\\
147	28.9150913122264\\
148	28.9136079663207\\
149	28.9134486591307\\
150	28.913074448732\\
151	28.9127089953618\\
152	28.9123212794054\\
153	28.9121466114654\\
154	28.9113729964778\\
155	28.9111932023266\\
156	28.9104117333012\\
157	28.9102278919952\\
158	28.9094404395983\\
159	28.9092534517275\\
160	28.9076412645742\\
161	28.907463236532\\
162	28.907065863777\\
163	28.9066777462897\\
164	28.9064790313368\\
165	28.9048757849777\\
166	28.9046780290116\\
167	28.9042884025631\\
168	28.9038895231385\\
169	28.9035033009145\\
170	28.9033052209156\\
171	28.9025206021942\\
172	28.9023232831207\\
173	28.9015420243124\\
174	28.901345782396\\
175	28.9005687318154\\
176	28.9003738295519\\
177	28.8988129828289\\
178	28.8986208843409\\
179	28.8982439456525\\
180	28.8978604971701\\
181	28.8974887079213\\
182	28.8972986845959\\
183	28.8965485541629\\
184	28.896360518456\\
185	28.8956180684767\\
186	28.8954321964025\\
187	28.8946978992408\\
188	28.8945143312797\\
189	28.8930511434991\\
190	28.8928713696592\\
191	28.8925195912506\\
192	28.8921632238095\\
193	28.8918174115734\\
194	28.8916410585743\\
195	28.8909464999351\\
196	28.8907727819746\\
197	28.8900881500361\\
198	28.8899171485646\\
199	28.8892426596624\\
200	28.889074433024\\
};

\addplot [color=mycolor3, line width=1.5pt]
  table[row sep=crcr]{%
1	7.83756461647623\\
2	9.96539248955178\\
3	11.8833944618388\\
4	14.3126625164992\\
5	17.6596776396811\\
6	22.9337647043041\\
7	30.7057958971115\\
8	30.7401377402808\\
9	30.7492838055124\\
10	30.7287205764232\\
11	30.7104394239976\\
12	30.6742114025756\\
13	30.6416408324169\\
14	30.5954478336047\\
15	30.554798714092\\
16	30.5031590049071\\
17	30.4584838435607\\
18	30.4045241571897\\
19	30.3585293827348\\
20	30.3043481882525\\
21	30.2588250093994\\
22	30.2057842635758\\
23	30.1618866377601\\
24	30.110830779729\\
25	30.0692706070927\\
26	30.0206848121477\\
27	29.9818691726919\\
28	29.9359935264919\\
29	29.9001220330646\\
30	29.8570318606428\\
31	29.8241643017827\\
32	29.783826523101\\
33	29.7539311332992\\
34	29.7162433418208\\
35	29.6892296814869\\
36	29.6540473452162\\
37	29.6297894734742\\
38	29.6173797765526\\
39	27.393380466708\\
40	27.5206110873047\\
41	27.103897476694\\
42	27.11799900952\\
43	27.0971685193893\\
44	27.0645037000941\\
45	27.0430658234465\\
46	27.0234321596463\\
47	27.0028323202809\\
48	26.9839303426794\\
49	26.9639433619521\\
50	26.9457091276706\\
51	26.9263222824427\\
52	26.9087316416741\\
53	26.8899321270813\\
54	26.87296510885\\
55	26.854736872563\\
56	26.8383742452507\\
57	26.8206998398471\\
58	26.8049225680441\\
59	26.7877839762158\\
60	26.7725732287262\\
61	26.7559523319054\\
62	26.7412895186458\\
63	26.7251683575136\\
64	26.7110352704493\\
65	26.6953961371776\\
66	26.6817749855085\\
67	26.666600575598\\
68	26.653473997087\\
69	26.6387474412857\\
70	26.626098576284\\
71	26.6118034703885\\
72	26.5996159499543\\
73	26.5857363955245\\
74	26.5739943531859\\
75	26.560514938314\\
76	26.5492030316081\\
77	26.5361088658726\\
78	26.5252122406973\\
79	26.5124889412006\\
80	26.5019932725874\\
81	26.4896269432501\\
82	26.4795183766541\\
83	26.4674956284514\\
84	26.4577608265334\\
85	26.4460687544878\\
86	26.4366948438642\\
87	26.4253210172338\\
88	26.4162955857733\\
89	26.4052280401605\\
90	26.3965391371971\\
91	26.3857663664183\\
92	26.3774024722586\\
93	26.3669134010687\\
94	26.3588634351915\\
95	26.3486474138171\\
96	26.340900687128\\
97	26.3309475066626\\
98	26.323493726852\\
99	26.3137935621784\\
100	26.3066228406051\\
101	26.2971662369782\\
102	26.2902690455197\\
103	26.2810469405236\\
104	26.2744141052095\\
105	26.2654178006559\\
106	26.2590404854223\\
107	26.2502616179915\\
108	26.244131336908\\
109	26.2355618603445\\
110	26.2296704320407\\
111	26.2213026497797\\
112	26.215642201736\\
113	26.2074687063171\\
114	26.2020316645033\\
115	26.1940453484113\\
116	26.1888244039749\\
117	26.1810184577489\\
118	26.17600657927\\
119	26.1683744598379\\
120	26.1635648901202\\
121	26.1561002883514\\
122	26.1514865115126\\
123	26.1441834072488\\
124	26.1397591432481\\
125	26.1326117473904\\
126	26.1283709343854\\
127	26.1213737014053\\
128	26.1173105163638\\
129	26.1104580928514\\
130	26.1065669229095\\
131	26.0998541771167\\
132	26.0961296249193\\
133	26.0895516222163\\
134	26.0859884647954\\
135	26.0795404895456\\
136	26.0761336935903\\
137	26.0698111996045\\
138	26.0665559199549\\
139	26.0603545456097\\
140	26.057246103534\\
141	26.0511616562782\\
142	26.048195548744\\
143	26.0422239934089\\
144	26.0393958704389\\
145	26.033533338527\\
146	26.0308389976167\\
147	26.0250817789189\\
148	26.0225171725746\\
149	26.0168616768377\\
150	26.0144229044298\\
151	26.0117793399543\\
152	24.0074640307541\\
153	24.2144852079233\\
154	24.2524377740281\\
155	24.2641906020476\\
156	24.2648528321067\\
157	24.2712245499979\\
158	24.2764081506786\\
159	24.2816363487322\\
160	24.2866524612198\\
161	24.2911041173298\\
162	24.2957697425182\\
163	24.2995967127038\\
164	24.3038759621097\\
165	24.3071646689894\\
166	24.3110611901074\\
167	24.3138719735803\\
168	24.3174023869066\\
169	24.3197827190401\\
170	24.322967881187\\
171	24.3249576222243\\
172	24.3278193349203\\
173	24.3294531695647\\
174	24.3320128368405\\
175	24.333321545868\\
176	24.3355996546292\\
177	24.3366108151574\\
178	24.3386267599366\\
179	24.3393652422165\\
180	24.3411372512372\\
181	24.3416255967495\\
182	24.3431707706313\\
183	24.3434294316829\\
184	24.3447638079687\\
185	24.3448113897845\\
186	24.3459499656777\\
187	24.3458034471431\\
188	24.3467602596186\\
189	24.3464351417663\\
190	24.3472233101331\\
191	24.3467337847938\\
192	24.3473655817734\\
193	24.3467246518985\\
194	24.3472115410894\\
195	24.3464311461307\\
196	24.3467838408776\\
197	24.3458749534233\\
198	24.3461034568559\\
199	24.345076190846\\
200	24.3451898420841\\
};

\addplot [color=mycolor4, line width=1.5pt]
  table[row sep=crcr]{%
1	7.83756461647623\\
2	14.2713233317354\\
3	15.6769110562073\\
4	16.5441840094138\\
5	17.1042485552621\\
6	17.8502040029142\\
7	18.5448895936185\\
8	18.8841013647676\\
9	20.0891111947919\\
10	20.3620217030651\\
11	20.8514722454769\\
12	21.2678402879184\\
13	21.4971265176735\\
14	22.2486213977715\\
15	22.4358053249067\\
16	23.0274981705248\\
17	23.1790783874471\\
18	23.6424042631527\\
19	23.7636219000818\\
20	24.1222174498683\\
21	24.2175855031352\\
22	24.4902289988117\\
23	24.5637761556622\\
24	24.7659977155996\\
25	24.8213465220449\\
26	24.966239793738\\
27	25.0066125028285\\
28	25.1052488637011\\
29	25.1334572891255\\
30	25.1951328373112\\
31	25.2135795100731\\
32	25.2460116167766\\
33	25.256709415732\\
34	25.2662268516228\\
35	25.2708339838722\\
36	25.2625682022389\\
37	25.2624320024219\\
38	25.2405033755436\\
39	25.2367067332403\\
40	25.1630949790885\\
41	25.1589079520271\\
42	25.136074016302\\
43	25.1155430396019\\
44	25.1031804153092\\
45	25.0557312953101\\
46	25.0425267614234\\
47	24.9362739235927\\
48	24.9221538935557\\
49	24.8958698908549\\
50	24.8669256690354\\
51	24.8545295745853\\
52	24.7976265032267\\
53	24.7711839802713\\
54	24.7566297599159\\
55	24.7011292233775\\
56	24.6866437675623\\
57	24.6312561274287\\
58	24.6169157426655\\
59	24.5619092504931\\
60	24.5477685959002\\
61	24.4933471842211\\
62	24.4650443640518\\
63	24.4527694739527\\
64	24.3981110413925\\
65	24.3858450332151\\
66	24.3322930335105\\
67	24.3200976005051\\
68	24.2676821389827\\
69	24.2431193539696\\
70	24.2300813219079\\
71	24.1801793507882\\
72	24.1674481000062\\
73	24.1185667767336\\
74	24.1061443051358\\
75	24.0582919838184\\
76	24.0461760406077\\
77	23.9518833355\\
78	23.9401376600716\\
79	23.9179163590913\\
80	23.8950385433854\\
81	23.8844060553585\\
82	23.7948763766667\\
83	23.7852100320413\\
84	23.7637626424446\\
85	23.7432037885454\\
86	23.7217719057459\\
87	23.7120459987732\\
88	23.6705483184768\\
89	23.6609520304117\\
90	23.620446705994\\
91	23.6109916911354\\
92	23.5714349311569\\
93	23.55265925881\\
94	23.5428414014708\\
95	23.5050525465189\\
96	23.4954695775667\\
97	23.4584635069033\\
98	23.4491066676309\\
99	23.4128579610198\\
100	23.3944457690332\\
101	23.3859912021498\\
102	23.350238785657\\
103	23.3418931617911\\
104	23.3069216918336\\
105	23.2986919306269\\
106	23.2644649316478\\
107	23.2481336296149\\
108	23.2396220437025\\
109	23.2067815960389\\
110	23.1984559394456\\
111	23.1662272910179\\
112	23.1580799938127\\
113	23.0946072709239\\
114	23.0866364914595\\
115	23.0714570290714\\
116	23.0559398688795\\
117	23.0485935212367\\
118	22.9877045441404\\
119	22.9808509068303\\
120	22.9660911536356\\
121	22.9518089686488\\
122	22.9370603297793\\
123	22.9301964097697\\
124	22.9014326734922\\
125	22.8946412876707\\
126	22.8663936975771\\
127	22.8596791514291\\
128	22.8040341851892\\
129	22.7977490512016\\
130	22.784216513702\\
131	22.7711126925935\\
132	22.7575743016278\\
133	22.7512671274979\\
134	22.7248395217906\\
135	22.7185895256183\\
136	22.6925958047376\\
137	22.6864072809012\\
138	22.6608280950598\\
139	22.6485200162861\\
140	22.6421256190465\\
141	22.6173611055841\\
142	22.6110743756621\\
143	22.5866477899298\\
144	22.5804642855142\\
145	22.5321616326269\\
146	22.5260560735355\\
147	22.5143741678574\\
148	22.5024611568782\\
149	22.4910176406947\\
150	22.485070705578\\
151	22.4620313872793\\
152	22.4561770632622\\
153	22.433422546752\\
154	22.4276570616643\\
155	22.3826087896785\\
156	22.3769029944668\\
157	22.3659809736417\\
158	22.3548408740549\\
159	22.3441305321689\\
160	22.3385641356844\\
161	22.3169901463\\
162	22.3115041418688\\
163	22.2901722853881\\
164	22.2847634319807\\
165	22.2636678615105\\
166	22.252957331861\\
167	22.247909605885\\
168	22.2269092307342\\
169	22.2218921543167\\
170	22.2011487111194\\
171	22.1961654664482\\
172	22.1551000937825\\
173	22.1503545142408\\
174	22.1402676542556\\
175	22.1304436225941\\
176	22.1203323872622\\
177	22.1155602059619\\
178	22.0957196424304\\
179	22.0909724553699\\
180	22.0713558929893\\
181	22.0666367845225\\
182	22.0277693855898\\
183	22.0232645983797\\
184	22.0137018317121\\
185	22.004381159325\\
186	21.9947909706978\\
187	21.9902578326726\\
188	21.97142585316\\
189	21.9669132352402\\
190	21.9482780940433\\
191	21.9437889633104\\
192	21.9068386396421\\
193	21.9025437741238\\
194	21.8934395378464\\
195	21.8845590159894\\
196	21.8754254750173\\
197	21.8711013967639\\
198	21.8531535158524\\
199	21.8488464048461\\
200	21.8310727773747\\
};

\end{axis}
\end{tikzpicture}%

%% file: f2_2.tex
%
%
\definecolor{mycolor1}{rgb}{0.00000,0.44700,0.74100}%
\definecolor{mycolor2}{rgb}{0.85000,0.32500,0.09800}%
\definecolor{mycolor3}{rgb}{0.92900,0.69400,0.12500}%
\definecolor{mycolor4}{rgb}{0.49400,0.18400,0.55600}%
\begin{tikzpicture}

\begin{axis}[%
width=1.45in,
height=1.1in,
at={(1.011in,0.642in)},
scale only axis,
xmin=0,
xmax=200,
xlabel style={font=\color{white!15!black}},
xlabel={Iterations},
ymin=20,
ymax=33,
yminorticks=true,
xlabel style={font=\scriptsize}, 
ylabel style={font=\scriptsize}, 
tick label style={font=\scriptsize},
axis background/.style={fill=white},
title={noise $\sigma= 0.15$},
xmajorgrids,
ymajorgrids,
yminorgrids,
xlabel shift=-5pt,
ylabel shift=-5pt,
title style={yshift=-4pt},
yticklabels={,,},
]
\addplot [color=mycolor1, line width=1.5pt]
  table[row sep=crcr]{%
1	7.83756461647623\\
2	9.97298915896972\\
3	11.8936782580442\\
4	14.3289430597976\\
5	17.6942148191902\\
6	23.0383413335204\\
7	31.3348621791225\\
8	31.40844503386\\
9	31.4718868782168\\
10	31.5158774740187\\
11	31.5557021256196\\
12	31.5694629460603\\
13	31.5941404659323\\
14	31.5955318734425\\
15	31.6113505760343\\
16	31.6059469404893\\
17	31.6166693000811\\
18	31.6075289478636\\
19	31.6152646404411\\
20	31.6156991557726\\
21	31.1882927397683\\
22	31.2411823755397\\
23	31.2670659516737\\
24	31.2736564838069\\
25	31.2760108566695\\
26	31.2733178668004\\
27	31.2719274354856\\
28	31.2665482876507\\
29	31.2629142341873\\
30	31.2559502559401\\
31	31.2512245532688\\
32	31.2434493379273\\
33	31.2382272321596\\
34	31.230083622779\\
35	31.2247055226056\\
36	31.2164587377036\\
37	31.2111246458379\\
38	31.2029331089762\\
39	31.1977615874675\\
40	31.189720487489\\
41	31.1847822015267\\
42	31.176947522464\\
43	31.1722824840955\\
44	31.1646833925321\\
45	31.1603103532158\\
46	31.1529575892851\\
47	31.1488815393783\\
48	31.1417747750265\\
49	31.1379932899866\\
50	31.1311266135929\\
51	31.1276337014406\\
52	31.1209985085\\
53	31.117786376508\\
54	31.1149337837879\\
55	30.920997176012\\
56	30.9334606693712\\
57	30.936462796332\\
58	30.9309139760416\\
59	30.9298781214166\\
60	30.9256838300149\\
61	30.9235092838217\\
62	30.9190780483676\\
63	30.9164109735865\\
64	30.9119188760224\\
65	30.9090588734903\\
66	30.9045898121239\\
67	30.9016877381661\\
68	30.8972835578585\\
69	30.8944187769688\\
70	30.8901013419539\\
71	30.8873166546387\\
72	30.8830973540096\\
73	30.8804159425988\\
74	30.8763001663238\\
75	30.8737342085601\\
76	30.8697236846006\\
77	30.8672789558347\\
78	30.8633730563588\\
79	30.86105149111\\
80	30.857248040781\\
81	30.8550491901354\\
82	30.8513449898016\\
83	30.8492669192308\\
84	30.8456580940778\\
85	30.8436979508959\\
86	30.8401802068271\\
87	30.8383345717036\\
88	30.8349033786527\\
89	30.8331685208572\\
90	30.8298192536792\\
91	30.8281913049335\\
92	30.8266399047862\\
93	30.6043666641238\\
94	30.6299749694557\\
95	30.6418255847073\\
96	30.6459114202071\\
97	30.6470733907516\\
98	30.6480275611486\\
99	30.6488286362409\\
100	30.6487290925849\\
101	30.6487927856911\\
102	30.6480864905069\\
103	30.647726012843\\
104	30.6466429519867\\
105	30.6460231185213\\
106	30.6446956185589\\
107	30.6439114247253\\
108	30.642421228595\\
109	30.6415314039419\\
110	30.6399316085107\\
111	30.6389745411005\\
112	30.6373009636737\\
113	30.6363028588606\\
114	30.6345803840341\\
115	30.6335597400285\\
116	30.631806058595\\
117	30.630776271959\\
118	30.6290041932906\\
119	30.6279751478146\\
120	30.6261940757738\\
121	30.6251731707942\\
122	30.6233900846689\\
123	30.622382914718\\
124	30.6206030275363\\
125	30.619613847235\\
126	30.6178410706926\\
127	30.6168731313441\\
128	30.6151103950373\\
129	30.6141661794335\\
130	30.6124156709683\\
131	30.6114970703429\\
132	30.6097604069733\\
133	30.6088688506925\\
134	30.6071472046124\\
135	30.6062837596952\\
136	30.6045779583083\\
137	30.6037434032378\\
138	30.602053997562\\
139	30.6012488865714\\
140	30.6004329339962\\
141	30.4837149912405\\
142	30.496573306914\\
143	30.502159879473\\
144	30.5038453671635\\
145	30.5036846223454\\
146	30.503708530945\\
147	30.5036488594208\\
148	30.5031781264458\\
149	30.5027827761891\\
150	30.502026475714\\
151	30.501437420594\\
152	30.5005077576991\\
153	30.4998018342253\\
154	30.4987640628833\\
155	30.4979864152559\\
156	30.4968806046924\\
157	30.4960595025523\\
158	30.4949113161399\\
159	30.4940653383347\\
160	30.4928917892783\\
161	30.4920335929327\\
162	30.4908462841233\\
163	30.4899847246402\\
164	30.488791749483\\
165	30.4879331443633\\
166	30.4867402464937\\
167	30.4858891648616\\
168	30.4847004716175\\
169	30.483860244557\\
170	30.4826787443698\\
171	30.4818518155936\\
172	30.4806796756919\\
173	30.4798678353458\\
174	30.4787066242775\\
175	30.4779111771891\\
176	30.4767620156636\\
177	30.4759839024059\\
178	30.4748475768013\\
179	30.4740874595552\\
180	30.4729645006448\\
181	30.4722228284545\\
182	30.4711135722819\\
183	30.4703906306548\\
184	30.4692952612849\\
185	30.4685912090359\\
186	30.4675097920609\\
187	30.4668246885649\\
188	30.4657571991935\\
189	30.4650910256313\\
190	30.4640373685188\\
191	30.4633900471011\\
192	30.462350069772\\
193	30.4617214737284\\
194	30.4606949819689\\
195	30.4600849496131\\
196	30.4590717157104\\
197	30.4584800563598\\
198	30.4574798247897\\
199	30.4569063272811\\
200	30.4559188237573\\
};

\addplot [color=mycolor2, line width=1.5pt]
  table[row sep=crcr]{%
1	7.83756461647623\\
2	14.270507001512\\
3	15.6772983636147\\
4	16.5392284477824\\
5	17.0942652511624\\
6	17.8284004113701\\
7	18.5115737490902\\
8	18.8454270985059\\
9	20.0370132126056\\
10	20.3092307091379\\
11	20.8011374699088\\
12	21.2256819884787\\
13	21.4608383999111\\
14	22.2535661107103\\
15	22.4547350302509\\
16	23.1212327781182\\
17	23.2953433884487\\
18	23.8664651534466\\
19	24.0184494867777\\
20	24.5150291383154\\
21	24.6485276725515\\
22	25.0850449811485\\
23	25.2028470707892\\
24	25.5897078102864\\
25	25.6940218821625\\
26	26.0389576323479\\
27	26.1315835699457\\
28	26.4405401234289\\
29	26.5229794937742\\
30	26.8006757180389\\
31	26.8741997115368\\
32	27.1244744551907\\
33	27.1901713218604\\
34	27.4162159595499\\
35	27.4750220519717\\
36	27.6795264368121\\
37	27.7322491540043\\
38	27.9175055724981\\
39	27.9648446956545\\
40	28.1328348156461\\
41	28.1754009433227\\
42	28.3278634747564\\
43	28.3661923677987\\
44	28.5046693645107\\
45	28.5392308924501\\
46	28.6650923253257\\
47	28.6962974657577\\
48	28.8107602815081\\
49	28.838969417325\\
50	28.9431205801948\\
51	28.9686506569509\\
52	29.1564057693253\\
53	29.1794820855153\\
54	29.2190298142044\\
55	29.2585075688935\\
56	29.2768586969249\\
57	29.350315102496\\
58	29.3670435237403\\
59	29.4339424187011\\
60	29.4639286818479\\
61	29.4798310712337\\
62	29.5375706038133\\
63	29.5519748867547\\
64	29.6046839861963\\
65	29.6177455556711\\
66	29.6658670539912\\
67	29.6777228163048\\
68	29.721660963742\\
69	29.7324316611435\\
70	29.7725544391159\\
71	29.7918498332271\\
72	29.8007722390205\\
73	29.8365816350339\\
74	29.8447253797701\\
75	29.8774083021455\\
76	29.884844394102\\
77	29.91467694498\\
78	29.9214691920517\\
79	29.9487026641264\\
80	29.9609189588135\\
81	29.9674176420387\\
82	29.9909604962966\\
83	29.9968612374857\\
84	30.0183865205654\\
85	30.0237481889716\\
86	30.0434248666684\\
87	30.0482995396256\\
88	30.0662826115883\\
89	30.075017117037\\
90	30.0789886044826\\
91	30.0951123302243\\
92	30.0987380258286\\
93	30.1134534415264\\
94	30.1167638574447\\
95	30.1301901859551\\
96	30.133212743276\\
97	30.1573342204274\\
98	30.1598357347555\\
99	30.1650980132387\\
100	30.1698961433379\\
101	30.1724406031353\\
102	30.1905932658229\\
103	30.1929376229017\\
104	30.1967954495229\\
105	30.2007440636389\\
106	30.2025067968963\\
107	30.2097511006674\\
108	30.2129015894531\\
109	30.2146784696192\\
110	30.2208180578072\\
111	30.2224217792057\\
112	30.2280082808478\\
113	30.2294555560736\\
114	30.234532806537\\
115	30.237092010566\\
116	30.2381660411573\\
117	30.2427603760406\\
118	30.2437336932167\\
119	30.2478901649867\\
120	30.2487713939767\\
121	30.2525262432072\\
122	30.2533231676404\\
123	30.2599440881116\\
124	30.2605525240324\\
125	30.2620011643451\\
126	30.26322957701\\
127	30.2639467190468\\
128	30.2686148654986\\
129	30.269306727272\\
130	30.2702626541089\\
131	30.2713153547003\\
132	30.2717288308747\\
133	30.2753083249642\\
134	30.275595249991\\
135	30.2763773682361\\
136	30.2769930831114\\
137	30.277388223693\\
138	30.2797404970654\\
139	30.2801360943956\\
140	30.2805959867095\\
141	30.2811442988037\\
142	30.2814944721496\\
143	30.281794653739\\
144	30.28254496406\\
145	30.2828027150011\\
146	30.2834479070403\\
147	30.2836677953123\\
148	30.284215967539\\
149	30.2844019780408\\
150	30.285258929808\\
151	30.2854613799466\\
152	30.2855993459544\\
153	30.2858186353516\\
154	30.2858559024765\\
155	30.2864166877205\\
156	30.2863931613552\\
157	30.2865079219832\\
158	30.2865210510105\\
159	30.2866483501448\\
160	30.2866230817696\\
161	30.2867347838739\\
162	30.2867018470408\\
163	30.2867556793782\\
164	30.2867156115321\\
165	30.2867166613581\\
166	30.2866699449281\\
167	30.2865402865738\\
168	30.2864494830561\\
169	30.2864094503125\\
170	30.2862875289119\\
171	30.2862647929401\\
172	30.2861791453228\\
173	30.286023755463\\
174	30.2859355408075\\
175	30.2857457369734\\
176	30.2856550919424\\
177	30.2854341662428\\
178	30.2853412193502\\
179	30.2848179916524\\
180	30.2846922818006\\
181	30.2845636473535\\
182	30.2843684662736\\
183	30.2842587145732\\
184	30.2841416473309\\
185	30.2838363374649\\
186	30.2837196389057\\
187	30.2833946381529\\
188	30.2832782349254\\
189	30.2829357479738\\
190	30.2828195811512\\
191	30.2820856376951\\
192	30.2819447558339\\
193	30.2817688048107\\
194	30.2815384656479\\
195	30.2813813058521\\
196	30.2812506044409\\
197	30.2808686016261\\
198	30.2807399219977\\
199	30.2803476193239\\
200	30.2802207789742\\
};

\addplot [color=mycolor3, line width=1.5pt]
  table[row sep=crcr]{%
1	7.83756461647623\\
2	9.97298915896972\\
3	11.8939815819671\\
4	14.3292435681081\\
5	17.6939848450179\\
6	23.0348327382622\\
7	31.324289625102\\
8	31.4082018747908\\
9	31.4881475279228\\
10	31.5587435213827\\
11	31.623613451848\\
12	31.6696435985044\\
13	31.7218310052193\\
14	31.7533601039276\\
15	31.7952415148519\\
16	31.8164362354805\\
17	31.8501023511523\\
18	31.8632314244196\\
19	31.857372665614\\
20	31.8797213574455\\
21	31.8754452614446\\
22	31.8284256218707\\
23	31.8243812049956\\
24	31.7647717526405\\
25	31.7590097473804\\
26	31.7301451669804\\
27	31.7133191466453\\
28	31.6960953659359\\
29	31.6803872075658\\
30	31.662572030783\\
31	31.6470989381643\\
32	31.6291587122345\\
33	31.6138308848652\\
34	31.5959290857452\\
35	31.5807761630866\\
36	31.5629975390645\\
37	31.5480659313223\\
38	31.5304676316904\\
39	31.5158004745579\\
40	31.4984253474237\\
41	31.4840587540052\\
42	31.466939841689\\
43	31.452902845773\\
44	31.4360655663449\\
45	31.4223808573839\\
46	31.4058443638311\\
47	31.3925291892867\\
48	31.3763074327065\\
49	31.3633743827566\\
50	31.3474769651401\\
51	31.3349347212841\\
52	31.3193675763505\\
53	31.3072215530229\\
54	31.2919875993679\\
55	31.2802404745853\\
56	31.2653400831649\\
57	31.2539922697897\\
58	31.2394237338719\\
59	31.228473771022\\
60	31.2142336832052\\
61	31.2036785560926\\
62	31.1897621159199\\
63	31.1795975760794\\
64	31.1659988381009\\
65	31.156219619697\\
66	31.1429317544509\\
67	31.1335317854986\\
68	31.1205472622178\\
69	31.1115198239201\\
70	31.0988305580946\\
71	31.0901684633478\\
72	31.0777659857743\\
73	31.0694616743348\\
74	31.0573372301831\\
75	31.0493828923067\\
76	31.037527516683\\
77	31.0299151686034\\
78	31.0183198102594\\
79	31.0110413790426\\
80	30.9996969291959\\
81	30.9927442904372\\
82	30.9875483330159\\
83	30.4178934008065\\
84	30.4582187024478\\
85	30.4693466078656\\
86	30.4618066224258\\
87	30.4455226085314\\
88	30.4343716301754\\
89	30.4220293303858\\
90	30.4099452721078\\
91	30.3986037744618\\
92	30.3863848076689\\
93	30.3755285859538\\
94	30.3634333934538\\
95	30.3529141102921\\
96	30.3410206896254\\
97	30.3307879520351\\
98	30.3191170865603\\
99	30.3091499684125\\
100	30.2977043476544\\
101	30.2879914533877\\
102	30.2767679478853\\
103	30.2673011934407\\
104	30.2562947941303\\
105	30.2470673756809\\
106	30.2362725090649\\
107	30.2272781800826\\
108	30.2166891561602\\
109	30.2079219787014\\
110	30.1975331393884\\
111	30.1889873824843\\
112	30.1787931723461\\
113	30.1704632730859\\
114	30.1604582407214\\
115	30.1523387888638\\
116	30.1425176209586\\
117	30.1346033512618\\
118	30.1249608449148\\
119	30.1172466324486\\
120	30.1077777260855\\
121	30.1002585787772\\
122	30.0909583265896\\
123	30.0836293911628\\
124	30.0744929852644\\
125	30.0673495361553\\
126	30.058372294702\\
127	30.0514097342933\\
128	30.0425870987556\\
129	30.0358009615489\\
130	30.0271285000711\\
131	30.0205144470629\\
132	30.0119878474985\\
133	30.005541657981\\
134	29.9971567380853\\
135	29.9908743162932\\
136	29.9826269970207\\
137	29.9765043656929\\
138	29.9683906976833\\
139	29.9624239905594\\
140	29.9544401434178\\
141	29.9486255999049\\
142	29.9407678564529\\
143	29.935101834054\\
144	29.9273665761191\\
145	29.9218455414021\\
146	29.9142292629966\\
147	29.908849784116\\
148	29.9013490840325\\
149	29.8961078287879\\
150	29.8887194094692\\
151	29.8836131438566\\
152	29.8763338072883\\
153	29.8713593888996\\
154	29.864186039883\\
155	29.8593404217144\\
156	29.8522700472697\\
157	29.8475502757685\\
158	29.8405799625247\\
159	29.8359831662682\\
160	29.8291100873467\\
161	29.8246334782136\\
162	29.8178548971464\\
163	29.8134957742603\\
164	29.806809033917\\
165	29.8025647684337\\
166	29.7959672957274\\
167	29.7918353432268\\
168	29.7853246335\\
169	29.7813025239976\\
170	29.7748761602457\\
171	29.7709614861066\\
172	29.7646171278712\\
173	29.7608075524748\\
174	29.7545429303272\\
175	29.7508361885646\\
176	29.7446490977574\\
177	29.7410429877751\\
178	29.7349312882106\\
179	29.731423677839\\
180	29.7253852971815\\
181	29.7219741001365\\
182	29.7190607194365\\
183	29.0747102275123\\
184	29.1472370707056\\
185	29.177737134961\\
186	29.1825677036983\\
187	29.18355555501\\
188	29.0349073435071\\
189	29.0472241554124\\
190	29.0183800788448\\
191	29.0209836755005\\
192	27.7941584862523\\
193	27.8404251343515\\
194	27.8541322501608\\
195	27.8623298617974\\
196	27.8545493344737\\
197	27.8620408149018\\
198	27.8567585838396\\
199	27.8582424597114\\
200	27.8121445448502\\
};

\addplot [color=mycolor4, line width=1.5pt]
  table[row sep=crcr]{%
1	7.83756461647623\\
2	14.270507001512\\
3	15.6770960406849\\
4	16.5449793492587\\
5	17.1066174834385\\
6	17.8556285301105\\
7	18.5555945395292\\
8	18.8976068573923\\
9	20.1215543790202\\
10	20.4001220151288\\
11	20.9039951968592\\
12	21.3366203019029\\
13	21.5762105016154\\
14	22.3750562765865\\
15	22.5771281246506\\
16	23.2333514967864\\
17	23.40454027506\\
18	23.9491574906826\\
19	24.0943970545062\\
20	24.5492360035897\\
21	24.6723810228077\\
22	25.053441966206\\
23	25.1576473197815\\
24	25.4771946531885\\
25	25.5651329816651\\
26	25.8329420787031\\
27	25.9069243956477\\
28	26.1310201445244\\
29	26.1930619312217\\
30	26.3801584688078\\
31	26.4320183514105\\
32	26.5877914974896\\
33	26.6309978151057\\
34	26.7602663688271\\
35	26.7961393521944\\
36	26.9029972563434\\
37	26.932670038735\\
38	27.0205912162066\\
39	27.0450305448892\\
40	27.1169570747823\\
41	27.1369835509312\\
42	27.1954021150702\\
43	27.2117094735553\\
44	27.2587182991776\\
45	27.2718912677163\\
46	27.3406914882566\\
47	27.3530756119858\\
48	27.3641407011896\\
49	27.3777786366579\\
50	27.3817131953281\\
51	27.401475044884\\
52	27.4041628714161\\
53	27.4185128683937\\
54	27.4209447242851\\
55	27.4254240543371\\
56	27.4307683580693\\
57	27.4339279126037\\
58	27.4360330480641\\
59	27.4380815775607\\
60	27.4374272826428\\
61	27.4385391978405\\
62	27.4355338236487\\
63	27.4356286930951\\
64	27.4335899946828\\
65	27.4297301640394\\
66	27.4273513038741\\
67	27.4216337494125\\
68	27.4189741151663\\
69	27.4116718299981\\
70	27.4087800736721\\
71	27.4001276139681\\
72	27.3936168780363\\
73	27.3922984848237\\
74	27.3804923669949\\
75	27.3787487717959\\
76	27.3662701953957\\
77	27.3641708457757\\
78	27.3511233639645\\
79	27.3487263720204\\
80	27.3207568336086\\
81	27.3189407212067\\
82	27.3115824138592\\
83	27.3054790702346\\
84	27.3015166501477\\
85	27.2735233336648\\
86	27.2690708300823\\
87	27.2623246085166\\
88	27.2543329285467\\
89	27.2512509794232\\
90	27.2201818391682\\
91	27.2176315475139\\
92	27.2097418147147\\
93	27.2028677902797\\
94	27.1944436198417\\
95	27.1914268142586\\
96	27.1756818924666\\
97	27.1725358195805\\
98	27.1568379598868\\
99	27.1535874918534\\
100	27.1379439381836\\
101	27.1311497930671\\
102	27.1268670277727\\
103	27.1121289485399\\
104	27.1079129010591\\
105	27.0931309684671\\
106	27.0889780750936\\
107	27.0741732795054\\
108	27.0658825543416\\
109	27.0627160810622\\
110	27.0471160599224\\
111	27.0438713923878\\
112	27.0284048051822\\
113	27.0250988740507\\
114	27.0097614970153\\
115	27.0029655061902\\
116	26.9988358478696\\
117	26.984350018163\\
118	26.9802939928391\\
119	26.9658364125388\\
120	26.9618488898985\\
121	26.9474302587068\\
122	26.9395080659572\\
123	26.9363499850286\\
124	26.9213493521025\\
125	26.918140494323\\
126	26.9032937901025\\
127	26.9000464500992\\
128	26.8853460863418\\
129	26.878730220529\\
130	26.8748141751343\\
131	26.8608688790046\\
132	26.8570232077436\\
133	26.8431323396788\\
134	26.8393517254927\\
135	26.8115043183231\\
136	26.8075536035094\\
137	26.8008679350234\\
138	26.793591817924\\
139	26.7871975284233\\
140	26.7834274831792\\
141	26.7699774521808\\
142	26.7662756352842\\
143	26.7528835938107\\
144	26.7492446620464\\
145	26.7359151677775\\
146	26.728704624422\\
147	26.725720467602\\
148	26.71197886298\\
149	26.7089632854696\\
150	26.69536552011\\
151	26.6923270041627\\
152	26.6652687504113\\
153	26.6625473129931\\
154	26.6558137095197\\
155	26.6496145224755\\
156	26.6426772955884\\
157	26.6397933379756\\
158	26.6265626166664\\
159	26.6236493702444\\
160	26.6105542235033\\
161	26.6076195231308\\
162	26.594651858277\\
163	26.5887138833463\\
164	26.5852916913085\\
165	26.5729065439231\\
166	26.5695412815715\\
167	26.557210619356\\
168	26.5538976301352\\
169	26.5292454853713\\
170	26.5258045224618\\
171	26.5198437630527\\
172	26.5134487353367\\
173	26.5077093083472\\
174	26.5044103513339\\
175	26.4924509570484\\
176	26.4892054381384\\
177	26.4772955234384\\
178	26.4740989973976\\
179	26.4622411481484\\
180	26.4559084471308\\
181	26.4531983146683\\
182	26.4410427725157\\
183	26.4383116272687\\
184	26.4262662888328\\
185	26.4235199718126\\
186	26.3995560102409\\
187	26.3970415203842\\
188	26.391069108417\\
189	26.3854953527098\\
190	26.3793737253014\\
191	26.3767437093697\\
192	26.3649824808549\\
193	26.3623321719092\\
194	26.3506728428234\\
195	26.3480077335298\\
196	26.3248087280797\\
197	26.3223614827992\\
198	26.3165755732442\\
199	26.311166752147\\
200	26.3052388354765\\
};

\end{axis}
\end{tikzpicture}%

%% file: f2_3.tex
%
%
\definecolor{mycolor1}{rgb}{0.00000,0.44700,0.74100}%
\definecolor{mycolor2}{rgb}{0.85000,0.32500,0.09800}%
\definecolor{mycolor3}{rgb}{0.92900,0.69400,0.12500}%
\definecolor{mycolor4}{rgb}{0.49400,0.18400,0.55600}%
\begin{tikzpicture}

\begin{axis}[%
width=1.45in,
height=1.1in,
at={(1.011in,0.642in)},
scale only axis,
xmin=0,
xmax=200,
xlabel style={font=\color{white!15!black}},
xlabel={Iterations},
ymode=log,
ymin=5e1,
ymax=10000,
yminorticks=true,
xlabel style={font=\scriptsize}, 
ylabel style={font=\scriptsize}, 
ylabel={$\norm{Au-f}^2$},
tick label style={font=\scriptsize},
axis background/.style={fill=white},
title={noise $\sigma= 0.3$},
xmajorgrids,
ymajorgrids,
yminorgrids,
xlabel shift=-5pt,
ylabel shift=-5pt,
title style={yshift=-4pt},
]
\addplot [color=mycolor1, line width=1.5pt]
  table[row sep=crcr]{%
5	263444.138072195\\
6	70280.8523960492\\
7	2247.39043456023\\
10	2083.3646318642\\
35	1945.83079922411\\
200	1923.75361951341\\
};

\addplot [color=mycolor2, line width=1.5pt]
  table[row sep=crcr]{%
2	135891.707976608\\
3	25694.9074863274\\
4	25016.7421830812\\
5	12686.8842938571\\
6	10780.2091716356\\
7	10324.5449581425\\
8	7391.69748317428\\
9	7074.10145522658\\
10	5061.88184488793\\
17	2824.31421145378\\
37	2010.38579714469\\
137	1924.02614169109\\
200	1923.20507687541\\
};

\addplot [color=mycolor3, line width=1.5pt]
  table[row sep=crcr]{%
5	262108.330562042\\
6	68950.0290259919\\
7	915.057592777853\\
9	770.298644532612\\
25	607.559576590921\\
38	545.204779323311\\
40	409.671856607354\\
45	383.112183855396\\
113	342.218549491194\\
152	315.996923462007\\
153	277.363495501455\\
179	273.578098671302\\
200	272.309824605561\\
};

\addplot [color=mycolor4, line width=1.5pt]
  table[row sep=crcr]{%
2	135219.831696153\\
3	25097.0029615926\\
4	24377.5270375873\\
5	12017.6876268371\\
6	10026.5725598877\\
7	9487.87818284571\\
8	6523.62796216511\\
9	6073.19362460001\\
10	4041.69609860287\\
21	1064.70867182725\\
24	890.330990869566\\
27	726.816614773374\\
40	478.84868989617\\
42	440.764442609196\\
53	378.434125021492\\
62	345.632469826421\\
76	311.446989515071\\
103	274.811579699912\\
200	220.113951987349\\
};

\end{axis}
\end{tikzpicture}%

%% file: f2_4.tex
%
%
\definecolor{mycolor1}{rgb}{0.00000,0.44700,0.74100}%
\definecolor{mycolor2}{rgb}{0.85000,0.32500,0.09800}%
\definecolor{mycolor3}{rgb}{0.92900,0.69400,0.12500}%
\definecolor{mycolor4}{rgb}{0.49400,0.18400,0.55600}%
\begin{tikzpicture}

\begin{axis}[%
width=1.45in,
height=1.1in,
at={(1.011in,0.642in)},
scale only axis,
xmin=0,
xmax=200,
xlabel style={font=\color{white!15!black}},
xlabel={Iterations},
ymode=log,
ymin=5e1,
ymax=10000,
yminorticks=true,
xlabel style={font=\scriptsize}, 
ylabel style={font=\scriptsize}, 
tick label style={font=\scriptsize},
axis background/.style={fill=white},
title={noise $\sigma= 0.15$},
xmajorgrids,
ymajorgrids,
yminorgrids,
xlabel shift=-5pt,
ylabel shift=-5pt,
title style={yshift=-4pt},
yticklabels={,,},
]
\addplot [color=mycolor1, line width=1.5pt]
  table[row sep=crcr]{%
1	2257293.38007631\\
2	1592566.0533375\\
3	1024582.0010244\\
4	581177.806481737\\
5	262865.255024406\\
6	69685.2228524873\\
7	1635.32638505567\\
8	1518.18746758333\\
9	1491.13775155597\\
10	1482.49379902297\\
11	1469.452854698\\
12	1464.46565203056\\
13	1455.0552549\\
14	1452.12457809286\\
15	1444.62039891614\\
16	1443.05822036285\\
17	1436.70841833213\\
18	1436.11880282501\\
19	1430.50288512711\\
20	1415.52702629203\\
21	1410.563307144\\
22	1388.07052651518\\
23	1380.34528712175\\
24	1377.97674739595\\
25	1376.36191888211\\
26	1375.52866636391\\
27	1374.58731786949\\
28	1374.02863931534\\
29	1373.26447671465\\
30	1372.82807515663\\
31	1372.14588474333\\
32	1371.78437409284\\
33	1371.15221891027\\
34	1370.84658207108\\
35	1370.25151167893\\
36	1369.99251908686\\
37	1369.42786773321\\
38	1369.20994750737\\
39	1368.67142851629\\
40	1368.49069028191\\
41	1367.97510262781\\
42	1367.82855914342\\
43	1367.33322879512\\
44	1367.21837992568\\
45	1366.74088948944\\
46	1366.65552715415\\
47	1366.19363789787\\
48	1366.13577653805\\
49	1365.68740438949\\
50	1365.65524416641\\
51	1365.21845016833\\
52	1365.21035652653\\
53	1364.78335003025\\
54	1363.8076616997\\
55	1361.16995580869\\
56	1359.86961866046\\
57	1359.35733385481\\
58	1359.29078493635\\
59	1359.06234552098\\
60	1359.00302163682\\
61	1358.8104781163\\
62	1358.74609127902\\
63	1358.57105277746\\
64	1358.50651526897\\
65	1358.34246664727\\
66	1358.28079142262\\
67	1358.12472356905\\
68	1358.0675008961\\
69	1357.91777745198\\
70	1357.86576982998\\
71	1357.72137751198\\
72	1357.67489047975\\
73	1357.53512651606\\
74	1357.49421759214\\
75	1357.35854721062\\
76	1357.32314116637\\
77	1357.19112763102\\
78	1357.16108062121\\
79	1357.03234926757\\
80	1357.00748408107\\
81	1356.88170342088\\
82	1356.86182907299\\
83	1356.73869951144\\
84	1356.72362292253\\
85	1356.60287030073\\
86	1356.59240249085\\
87	1356.47377393102\\
88	1356.46773347811\\
89	1356.35099500444\\
90	1356.34920986238\\
91	1356.23414437703\\
92	1355.98547507206\\
93	1355.97079487802\\
94	1354.34399895435\\
95	1354.02346085509\\
96	1353.87577091094\\
97	1353.82339618169\\
98	1353.79615100075\\
99	1353.75526249168\\
100	1353.73474236607\\
101	1353.69661702235\\
102	1353.67885543655\\
103	1353.64156484962\\
104	1353.62541427209\\
105	1353.58851268219\\
106	1353.57357361919\\
107	1353.53700263173\\
108	1353.5231184417\\
109	1353.48691516998\\
110	1353.4740192987\\
111	1353.43823253031\\
112	1353.42629438881\\
113	1353.39096203551\\
114	1353.37996520016\\
115	1353.34511058723\\
116	1353.3350438338\\
117	1353.30067720279\\
118	1353.29153030456\\
119	1353.25765178328\\
120	1353.2494139908\\
121	1353.21601608789\\
122	1353.2086755109\\
123	1353.17574528564\\
124	1353.16928881105\\
125	1353.13680953388\\
126	1353.13122300325\\
127	1353.09917529343\\
128	1353.09444376887\\
129	1353.06280648696\\
130	1353.05891447296\\
131	1353.02766542049\\
132	1353.02459705412\\
133	1352.99371340583\\
134	1352.99145275364\\
135	1352.96091145764\\
136	1352.95944250392\\
137	1352.92922050957\\
138	1352.92852746769\\
139	1352.89860189466\\
140	1352.83547477528\\
141	1352.73904693364\\
142	1352.36294203945\\
143	1352.2760998098\\
144	1352.24093824259\\
145	1352.22460192806\\
146	1352.2143728442\\
147	1352.20130899493\\
148	1352.1928541104\\
149	1352.18046279886\\
150	1352.17273478409\\
151	1352.16061062893\\
152	1352.15330958299\\
153	1352.14136935328\\
154	1352.1343979718\\
155	1352.12263534867\\
156	1352.11596130083\\
157	1352.10438558357\\
158	1352.0979987683\\
159	1352.08661920063\\
160	1352.08051677981\\
161	1352.06933919113\\
162	1352.06352007626\\
163	1352.05254678615\\
164	1352.04700963881\\
165	1352.03624015239\\
166	1352.03098277362\\
167	1352.02041446648\\
168	1352.01543370767\\
169	1352.00506245159\\
170	1352.00035434303\\
171	1351.99017491926\\
172	1351.98573483354\\
173	1351.97574130106\\
174	1351.97156413106\\
175	1351.96175004923\\
176	1351.95783035494\\
177	1351.94818898123\\
178	1351.94452111612\\
179	1351.93504553497\\
180	1351.93162373429\\
181	1351.92230698593\\
182	1351.9191254437\\
183	1351.90996054514\\
184	1351.90701349827\\
185	1351.89799354004\\
186	1351.89527529209\\
187	1351.88639344773\\
188	1351.88389840629\\
189	1351.87514799745\\
190	1351.87287068852\\
191	1351.8642451725\\
192	1351.86218027978\\
193	1351.85367329385\\
194	1351.85181564974\\
195	1351.84342099004\\
196	1351.84176559396\\
197	1351.83347724223\\
198	1351.83201927431\\
199	1351.82383138093\\
200	1351.82256620437\\
};

\addplot [color=mycolor2, line width=1.5pt]
  table[row sep=crcr]{%
1	2257293.38007631\\
2	135237.250242135\\
3	25037.2572253058\\
4	24380.4088353709\\
5	12054.3789986053\\
6	10156.2811825421\\
7	9706.84490472139\\
8	6775.94938576994\\
9	6466.95959318791\\
10	4455.84408674994\\
11	4029.0864717578\\
12	3818.14738477614\\
13	3283.74126039983\\
14	2995.0999530816\\
15	2625.39629003931\\
16	2497.96472000319\\
17	2234.96994633492\\
18	2179.04843890963\\
19	1988.70901132904\\
20	1964.83152269036\\
21	1825.52640044148\\
22	1815.75561418879\\
23	1713.01999914186\\
24	1709.10944531512\\
25	1632.92492608089\\
26	1631.12750640212\\
27	1574.39213110837\\
28	1573.07447120273\\
29	1530.67417069247\\
30	1529.19604163378\\
31	1497.40993168594\\
32	1495.58976019491\\
33	1471.69121081741\\
34	1469.54718687666\\
35	1451.52781151517\\
36	1449.15182381178\\
37	1435.52603476134\\
38	1433.02329869453\\
39	1422.68840673235\\
40	1420.1518878337\\
41	1412.28713392516\\
42	1409.78922024201\\
43	1403.78229227348\\
44	1401.37477254088\\
45	1396.76815879065\\
46	1394.485725219\\
47	1390.93711825797\\
48	1388.80088886726\\
49	1386.05371478044\\
50	1384.07415208764\\
51	1381.93573833722\\
52	1381.92475939676\\
53	1376.09121396771\\
54	1375.31316736435\\
55	1374.91926393671\\
56	1373.30537224317\\
57	1372.30217173562\\
58	1371.04069998809\\
59	1370.07006067849\\
60	1370.06845277556\\
61	1368.41097674423\\
62	1368.05987800892\\
63	1366.77898122854\\
64	1366.34468653646\\
65	1365.34900273237\\
66	1364.86928701559\\
67	1364.09015361336\\
68	1363.59121206771\\
69	1362.97696513431\\
70	1362.47660718885\\
71	1362.45472340354\\
72	1361.64532807594\\
73	1361.42723900545\\
74	1360.79703626189\\
75	1360.53161972752\\
76	1360.03735975183\\
77	1359.74548325114\\
78	1359.35467971355\\
79	1359.05102026454\\
80	1359.04348567052\\
81	1358.52368575431\\
82	1358.39115395648\\
83	1357.98543243467\\
84	1357.81654637826\\
85	1357.49734529456\\
86	1357.30687516856\\
87	1357.05349749715\\
88	1356.85191140518\\
89	1356.84325858301\\
90	1356.50562473093\\
91	1356.41087121881\\
92	1356.14635013923\\
93	1356.02621976017\\
94	1355.81714821353\\
95	1355.68164513836\\
96	1355.51476930887\\
97	1355.50581519708\\
98	1355.04070545828\\
99	1354.96665955419\\
100	1354.9209810289\\
101	1354.78582796012\\
102	1354.78479984698\\
103	1354.40599230726\\
104	1354.34649852383\\
105	1354.31000452445\\
106	1354.20011130875\\
107	1354.10836711875\\
108	1354.10140644614\\
109	1353.95470866981\\
110	1353.90652779592\\
111	1353.79089748603\\
112	1353.73068381012\\
113	1353.63863030383\\
114	1353.57101908941\\
115	1353.57048459285\\
116	1353.44617636762\\
117	1353.41472751329\\
118	1353.3170590849\\
119	1353.27424862508\\
120	1353.19676642101\\
121	1353.14671705925\\
122	1353.08458604836\\
123	1353.07902090809\\
124	1352.90565163285\\
125	1352.87692820036\\
126	1352.85862327513\\
127	1352.80769462435\\
128	1352.80105073477\\
129	1352.65939976126\\
130	1352.63544104198\\
131	1352.61997874229\\
132	1352.57817902673\\
133	1352.57062858685\\
134	1352.45484642204\\
135	1352.43477952658\\
136	1352.4216432319\\
137	1352.38729263922\\
138	1352.37904159002\\
139	1352.28436049974\\
140	1352.26747688955\\
141	1352.2562464072\\
142	1352.25621395866\\
143	1352.20868990316\\
144	1352.19678724529\\
145	1352.15939014095\\
146	1352.14287074979\\
147	1352.11313960972\\
148	1352.09361664782\\
149	1352.06970854303\\
150	1352.06636257512\\
151	1351.9997911587\\
152	1351.98836210265\\
153	1351.98089243833\\
154	1351.96115421303\\
155	1351.95651667675\\
156	1351.90199411658\\
157	1351.89221020039\\
158	1351.88565923866\\
159	1351.88553385765\\
160	1351.85811852208\\
161	1351.85099006849\\
162	1351.82938194427\\
163	1351.81951405034\\
164	1351.80229968527\\
165	1351.79063255559\\
166	1351.77675394062\\
167	1351.77382481721\\
168	1351.73536657866\\
169	1351.7285010752\\
170	1351.72390709365\\
171	1351.72383850684\\
172	1351.70449492635\\
173	1351.69947201211\\
174	1351.68421648755\\
175	1351.67721972974\\
176	1351.66505612443\\
177	1351.65675474605\\
178	1351.64693763717\\
179	1351.64466597638\\
180	1351.6174890887\\
181	1351.6125741684\\
182	1351.60925860882\\
183	1351.6091334389\\
184	1351.59544104723\\
185	1351.59173099481\\
186	1351.58092014911\\
187	1351.57580286925\\
188	1351.56717117619\\
189	1351.56112156036\\
190	1351.55414349203\\
191	1351.55213694784\\
192	1351.53290301914\\
193	1351.52932806707\\
194	1351.52688096461\\
195	1351.52668475978\\
196	1351.51696651295\\
197	1351.51412977384\\
198	1351.50644453647\\
199	1351.50261194375\\
200	1351.49646425488\\
};

\addplot [color=mycolor3, line width=1.5pt]
  table[row sep=crcr]{%
1	2256863.04407631\\
2	1591234.77503238\\
3	1023288.77249374\\
4	579906.599535889\\
5	261608.777333708\\
6	68430.8385653292\\
7	375.314253737907\\
8	267.929701800573\\
9	248.253939510826\\
10	248.211153647292\\
11	235.168980428545\\
12	234.708593442282\\
13	224.495395528004\\
14	224.112965652422\\
15	215.557852185044\\
16	215.449220300289\\
17	207.981491776262\\
18	193.345100732563\\
19	179.945868564451\\
20	142.06860804956\\
21	138.91695445002\\
22	136.707681820689\\
23	133.392943226683\\
24	130.822279554557\\
25	128.104048664538\\
26	126.69728353033\\
27	125.549903436517\\
28	124.574226835825\\
29	123.443440935906\\
30	122.570678944573\\
31	121.49152678074\\
32	120.695353550938\\
33	119.671824290926\\
34	118.941979378632\\
35	117.971142413771\\
36	117.301306456809\\
37	116.378981390389\\
38	115.764228076901\\
39	114.886248816823\\
40	114.322370460029\\
41	113.484864007439\\
42	112.968146091776\\
43	112.167574409992\\
44	111.69469452661\\
45	110.927839774031\\
46	110.495809070856\\
47	109.759743922216\\
48	109.365863472009\\
49	108.657923688812\\
50	108.299750579401\\
51	107.617506194131\\
52	107.292826838138\\
53	106.634057944755\\
54	106.340860830117\\
55	105.703536312257\\
56	105.439992978306\\
57	104.822250717609\\
58	104.586695030346\\
59	103.986826199152\\
60	103.777737705261\\
61	103.194171769566\\
62	103.010161269532\\
63	102.441451152191\\
64	102.281249360535\\
65	101.726060154915\\
66	101.588503237516\\
67	101.045602720427\\
68	100.929622796894\\
69	100.397871960069\\
70	100.302488682353\\
71	99.7808321684601\\
72	99.7051433896061\\
73	99.1926036662814\\
74	99.1357779978179\\
75	98.6314466476823\\
76	98.5927188514872\\
77	98.0957524056854\\
78	98.0744140493896\\
79	97.5840276342706\\
80	97.5794251486637\\
81	97.0948884192645\\
82	96.0732603946773\\
83	90.7621049881718\\
84	86.7220607924365\\
85	85.1099558027361\\
86	84.7630890831045\\
87	84.5381704855244\\
88	84.3142982990642\\
89	84.0710169534983\\
90	83.8776488820666\\
91	83.6303853084111\\
92	83.4513490838775\\
93	83.2070588045419\\
94	83.0377021930455\\
95	82.7984717261523\\
96	82.6371360513429\\
97	82.4034461177959\\
98	82.2493865419884\\
99	82.0212124398097\\
100	81.8739928449513\\
101	81.6511540010386\\
102	81.5104537991773\\
103	81.292725592522\\
104	81.1582741003612\\
105	80.9454243243552\\
106	80.8169769561058\\
107	80.6087781735638\\
108	80.4861073886467\\
109	80.2823407570383\\
110	80.1652322841109\\
111	79.9656887906735\\
112	79.8539397077454\\
113	79.6584202496073\\
114	79.5518381798708\\
115	79.3601533023872\\
116	79.2585553440259\\
117	79.0705248060782\\
118	78.9737374625426\\
119	78.7891897887208\\
120	78.6970478626014\\
121	78.5158201373755\\
122	78.4281668103822\\
123	78.250103865541\\
124	78.1667898252608\\
125	77.9917442439032\\
126	77.9126273329507\\
127	77.7404588535046\\
128	77.6654039185399\\
129	77.4959789343637\\
130	77.4248572184487\\
131	77.2580484557948\\
132	77.1907373609199\\
133	77.0264237718274\\
134	76.962806241745\\
135	76.8008724197582\\
136	76.7408372004866\\
137	76.5811730455405\\
138	76.5246137646126\\
139	76.3671144525197\\
140	76.3139294459399\\
141	76.1584953034796\\
142	76.1085874169509\\
143	75.9551230160806\\
144	75.9083997817036\\
145	75.7568140685918\\
146	75.7131869206241\\
147	75.5633930024598\\
148	75.5227772422716\\
149	75.3746921868365\\
150	75.3370067476869\\
151	75.1905513423784\\
152	75.1557185142243\\
153	75.0108172076173\\
154	74.978762482685\\
155	74.8353429841624\\
156	74.8059952202303\\
157	74.6639882596074\\
158	74.6372790557925\\
159	74.496618581465\\
160	74.4724823031653\\
161	74.3331051870398\\
162	74.3114787766368\\
163	74.1733244453845\\
164	74.1541474284614\\
165	74.0171581010317\\
166	74.0003722864938\\
167	73.8644924276302\\
168	73.8500421643934\\
169	73.7152185382074\\
170	73.7030499723031\\
171	73.5692319560882\\
172	73.5592931924446\\
173	73.4264322591246\\
174	73.4186733649257\\
175	73.2867228614153\\
176	73.2810958531776\\
177	73.1500111541659\\
178	73.1464698218055\\
179	73.0162078639821\\
180	73.0147075838522\\
181	72.8852276090954\\
182	72.6125307031577\\
183	71.9590090239871\\
184	68.2638617347236\\
185	66.6247970712153\\
186	66.5727989088041\\
187	66.3254832372452\\
188	65.6843419097712\\
189	64.8928849657176\\
190	64.8716859575549\\
191	64.5895514917856\\
192	58.621521967452\\
193	55.6808880298153\\
194	55.5535422286406\\
195	55.3201252099153\\
196	55.1760131177989\\
197	54.9938251239641\\
198	54.9518542167145\\
199	54.8909752283034\\
200	54.7426584222762\\
};

\addplot [color=mycolor4, line width=1.5pt]
  table[row sep=crcr]{%
1	2256863.04407631\\
2	134565.790248072\\
3	24442.753812476\\
4	23754.5185438289\\
5	11404.3577946633\\
6	9433.57711515856\\
7	8912.54187421246\\
8	5955.92575286184\\
9	5535.9718698197\\
10	3511.05283327518\\
11	3048.68928754723\\
12	2806.89115756037\\
13	2261.37211789085\\
14	1925.5733327931\\
15	1548.72822835237\\
16	1388.08882974632\\
17	1120.2805304517\\
18	1039.91114958993\\
19	846.177638573681\\
20	803.662802786908\\
21	661.862996872277\\
22	637.498801507592\\
23	532.85820756296\\
24	517.277785074808\\
25	439.585860312979\\
26	428.293255762228\\
27	370.326336543946\\
28	361.172636212531\\
29	317.738844537693\\
30	309.717801614284\\
31	277.044865801061\\
32	269.704831067876\\
33	245.030868760315\\
34	238.183910360412\\
35	219.475486076816\\
36	213.053265149444\\
37	198.806921618326\\
38	192.790171598278\\
39	181.890578377507\\
40	176.275737425695\\
41	167.893197653941\\
42	162.677709383808\\
43	156.193519368279\\
44	151.370331521706\\
45	146.32210773984\\
46	145.906908608687\\
47	132.233185072847\\
48	130.314251745164\\
49	129.288477421805\\
50	125.473641509046\\
51	122.911758574142\\
52	119.918918361091\\
53	117.420507889897\\
54	117.289812377419\\
55	113.379168734861\\
56	112.302281398278\\
57	109.267663449312\\
58	107.988377939695\\
59	105.616067079008\\
60	104.224794322574\\
61	102.354568476754\\
62	100.914413702845\\
63	100.885684554642\\
64	98.4025540914447\\
65	97.7819722944938\\
66	95.8470321959667\\
67	95.0491795591723\\
68	93.5285329072056\\
69	92.6223810731049\\
70	91.4157155113465\\
71	90.4502396145407\\
72	90.4025507463837\\
73	88.7954755719347\\
74	88.326780335714\\
75	87.0669035582983\\
76	86.470151343288\\
77	85.4726917910842\\
78	84.7957398655789\\
79	83.997215668895\\
80	83.861523311716\\
81	81.654610694671\\
82	81.2769882888976\\
83	81.0315141540756\\
84	80.3799412493757\\
85	80.2390305076396\\
86	78.4411812029234\\
87	78.1230872870674\\
88	77.9115670803443\\
89	77.375994723165\\
90	77.2069779928181\\
91	75.7391136733111\\
92	75.465108503343\\
93	75.2772604873459\\
94	75.2607859146153\\
95	74.5209511314105\\
96	74.3022098992397\\
97	73.7179030519601\\
98	73.4231907398634\\
99	72.9561112970502\\
100	72.6104168417426\\
101	72.6018574875195\\
102	71.9677574973692\\
103	71.7882953433816\\
104	71.2871841803496\\
105	71.039168931643\\
106	70.6381989064326\\
107	70.3435428188953\\
108	70.3367084022953\\
109	69.7911637602939\\
110	69.6362533151026\\
111	69.2046147566747\\
112	68.9887176087678\\
113	68.6427877220429\\
114	68.3850112364363\\
115	68.3764629210996\\
116	67.905589867522\\
117	67.7659803073721\\
118	67.3927863119432\\
119	67.1994357633151\\
120	66.8997001833812\\
121	66.669209276518\\
122	66.6571638891926\\
123	66.2495770675653\\
124	66.1194500467841\\
125	65.7957019768299\\
126	65.6185358416483\\
127	65.3578127668986\\
128	65.1479944043617\\
129	65.1316666777627\\
130	64.7779276483046\\
131	64.6536466305497\\
132	64.3719165108065\\
133	64.2066657172715\\
134	63.9790529012659\\
135	63.9684558652443\\
136	63.3235848393789\\
137	63.2111452325643\\
138	63.1352128512868\\
139	63.1335915435329\\
140	62.8082476093261\\
141	62.7176608890087\\
142	62.4594421848404\\
143	62.3285968591527\\
144	62.1207454595676\\
145	61.9614612964255\\
146	61.9526889112495\\
147	61.6690463336057\\
148	61.5762542728336\\
149	61.3503395142666\\
150	61.2227958679443\\
151	61.0401958689132\\
152	61.0398491066279\\
153	60.5196237168172\\
154	60.4298448683504\\
155	60.3693065341326\\
156	60.3685486316131\\
157	60.1059661084799\\
158	60.0331221729107\\
159	59.8244814371822\\
160	59.7180722562176\\
161	59.5498613801539\\
162	59.41951947332\\
163	59.4112703361671\\
164	59.1815345378605\\
165	59.1036764513609\\
166	58.9203610217163\\
167	58.8136498351393\\
168	58.6651307376622\\
169	58.6557216421093\\
170	58.2342227007756\\
171	58.1589437462085\\
172	58.1071632453555\\
173	58.1035500759291\\
174	57.8899312149333\\
175	57.8247192861792\\
176	57.6545350312684\\
177	57.5617074661099\\
178	57.4240556248037\\
179	57.3114327024886\\
180	57.3009406795083\\
181	57.113447473582\\
182	57.0426223837948\\
183	56.8925215957791\\
184	56.7980123289572\\
185	56.6759239377395\\
186	56.6508372744146\\
187	56.3081347364392\\
188	56.2427283940743\\
189	56.196243490123\\
190	56.1886053237053\\
191	56.013696475671\\
192	55.9514062658774\\
193	55.8115136779615\\
194	55.7266758484964\\
195	55.6129998064357\\
196	55.5964779005055\\
197	55.2756012269903\\
198	55.2157566033116\\
199	55.1736199706431\\
200	55.1678947110574\\
};

\end{axis}
\end{tikzpicture}%